\documentclass[twoside,11pt]{article}

\usepackage{blindtext}

\usepackage{amsthm}
\usepackage[preprint]{jmlr2e}

\usepackage[utf8]{inputenc} 
\usepackage[T1]{fontenc}    
\usepackage{url}            
\usepackage{booktabs}       
\usepackage{amsfonts}       
\usepackage{nicefrac}       
\usepackage{microtype}      
\usepackage{xcolor}         
\usepackage{pifont}         
\newcommand{\cmark}{\ding{51}}%
\newcommand{\xmark}{\ding{55}}%

\usepackage{graphicx}
\usepackage{amsmath}
\usepackage{amssymb}
\usepackage{mathtools}

\usepackage{natbib}
\setcitestyle{numbers,square}
\usepackage[capitalize,noabbrev]{cleveref}

\newcommand{\ind}{\perp\!\!\!\!\perp} 
 
\newcommand{\simeqint}{\simeq_\mathcal{I}}
\newcommand{\notsimeqint}{\not\simeq_\mathcal{I}}
\newcommand{\simeqintstar}{\simeq_{\mathcal{I}^*}}
\newcommand{\notsimeqintstar}{\not\simeq_{\mathcal{I}^*}}

\newcommand{\Z}{\boldsymbol{Z}}
\newcommand{\X}{\boldsymbol{X}}
\newcommand{\U}{\mathbf{U}}
\newcommand{\V}{\mathbf{V}}
\newcommand{\W}{\mathbf{W}}
\newcommand{\Uj}{\mathbf{U_j}}
\newcommand{\Vi}{\mathbf{V_i}}

\newcommand{\bPhi}{\boldsymbol{\Phi}}
\newcommand{\bTheta}{\boldsymbol{\Theta}}
\newcommand{\bLambda}{\boldsymbol{\Lambda}}
\newcommand{\bB}{\boldsymbol{B}}

\newcommand{\bY}{\boldsymbol{Y}}

\newcommand{\calG}{\boldsymbol{\mathcal{G}}}

\newcommand{\calD}{\boldsymbol{\mathcal{D}}}
\newcommand{\calF}{\boldsymbol{\mathcal{F}}}
\newcommand{\calI}{\mathcal{I}}
\newcommand{\calL}{\boldsymbol{\mathcal{L}}}
\newcommand{\calM}{\boldsymbol{\mathcal{M}}}

\newcommand{\calIstar}{\mathcal{I}^*}
\newcommand{\truemarginal}{\calM_{\calIstar}(G)}
\newcommand{\truefdagmarginal}{\calM_{\calIstar}(D)}

\newcommand{\pk}{p^{(k)}}
\newcommand{\qG}{q(G;\bLambda)}
\newcommand{\qD}{q(D;\bLambda)}

\newcommand{\pG}{p(G)}
\newcommand{\qkG}{q^{(k)}(G;\bLambda)}

\newcommand{\pkXG}{p_\Phi^{(k)}(\X|G)}
\newcommand{\pkGX}{p_\Phi^{(k)}(G|\X)}
\newcommand{\fk}{f^{(k)}}

\newcommand{\Zpar}{\boldsymbol{Z_{\pi_i^D}}}
\newcommand{\Xpar}{\boldsymbol{X_{\pi_j^D}}}
\newcommand{\Gint}{G^\mathcal{I}}
\newcommand{\Dint}{D^\mathcal{I}}

\newcommand{\Doneintsq}{(D_1^\mathcal{I})^2[V]}
\newcommand{\Dtwointsq}{(D_2^\mathcal{I})^2[V]}

\newcommand{\scoreIstar}{S_\mathcal{I^*}(G)}
\newcommand{\scoreIGstar}{S_\mathcal{I^*}(G^*)}

\newcommand{\bayesianscorestar}{S_{\calIstar}(q)}

\usepackage{lastpage}

\ShortHeadings{Amortized Bayesian Causal Discovery of Extended Factor
Graphs}{Yichen Gu, Yuxuan Song, Weizhou Qian, Yixin Wang and Joshua Welch}
\firstpageno{1}

\begin{document}

\title{Amortized Bayesian Causal Discovery of Extended Factor Graphs}

\author{\name Yichen Gu \email        gyichen@umich.edu \\
       \addr Department of Electrical and Computer Engineering\\
       University of Michigan\\
       Ann Arbor, MI 48109, USA
       \AND
       \name Yuxuan Song \email yuxuans@umich.edu \\
       \addr Department of Computational Medicine and Bioinformatics\\
       University of Michigan, Ann Arbor\\
       Ann Arbor, MI 48109, USA
       \AND
       \name Weizhou Qian \email wzqian@umich.edu \\
       \addr Department of Computational Medicine and Bioinformatics\\
       University of Michigan, Ann Arbor\\
       Ann Arbor, MI 48109, USA
       \AND
       \name Yixin Wang \email yixinw@umich.edu \\
       \addr Department of Statistics\\
       University of Michigan, Ann Arbor\\
       Ann Arbor, MI 48109, USA
       \AND
       \name Joshua Welch \email welchjd@umich.edu \\
       \addr
       Department of Computer Science and Engineering\\
       Department of Computational Medicine and Bioinformatics\\
       University of Michigan, Ann Arbor\\
       Ann Arbor, MI 48109, USA}

\editor{My editor}

\maketitle

\begin{abstract}
Learning causal graphs from interventional data is a challenging problem with broad applications. In molecular biology, for example, a central goal is to uncover gene regulatory networks from large-scale perturbation data. An ideal algorithm for this task should scale to thousands of nodes, incorporate interventions even when their targets are unknown, quantify uncertainty, and provide identifiability guarantees.  However, existing approaches---e.g. approaches using score-based optimization or approximate Bayesian inference---often fail to meet all of these criteria. To address these limitations, we develop Amortized Bayesian Causal Discovery of Extended Factor Graphs (ABCDEFG). Our method guarantees exact acyclicity, scales to graphs with thousands of nodes, and naturally handles interventions even when their targets are unknown. Additionally, ABCDEFG estimates a posterior distribution whose maximum a posteriori estimate provably identifies the true causal graph up to an equivalence class. On simulated datasets, ABCDEFG achieves state-of-the-art accuracy, producing a well-calibrated posterior distribution while outperforming previous score-based and approximate Bayesian methods. Applied to large-scale single-cell perturbation data, ABCDEFG identifies both established and novel gene targets of growth factors~\cite{AMIN20241831}.
\end{abstract}

\begin{keywords}
  causal discovery, variational inference, graphical model
\end{keywords}

\section{Introduction}
Discovering causal relationships is a fundamental challenge across scientific domains. In many settings, both observational and interventional data are available to probe underlying causal mechanisms. Yet, inferring causal relationships remains difficult in large, complex systems.  For example, in computational biology, understanding how genes influence one another through gene regulatory networks is crucial for understanding cellular development and homeostasis.  Recent biotechnological advances now enable high-throughput perturbation experiments, providing measurements of gene expression across thousands to millions of cells under various interventions, providing exciting new data for inferring causal relationships in the cell. 

However, existing causal discovery methods fall short when applied to inferring a gene regulatory network from high-throughput perturbation data. Many approaches cannot scale to the large number of variables in the gene regulatory network (more than 20,000 genes) or the large number of samples ($10^4-10^6$ cells). Very noisy data, correlated causal edge probabilities, and interventions with unknown targets (such as drug treatments) pose additional challenges. While approximate Bayesian methods offer the advantage of uncertainty quantification (a crucial property for noisy biological data), they typically struggle to scale to problems of this size. Although prior work has addressed some of these issues in isolation, no existing method satisfies all the requirements simultaneously. There remains a need for new causal inference approaches that are scalable, uncertainty-aware, and capable of jointly learning causal gene relationships and intervention targets from large-scale single-cell drug or growth factor screens. 

To address these challenges, we develop \textbf{A}mortized \textbf{B}ayesian \textbf{C}ausal \textbf{D}iscovery of \textbf{E}xtended \textbf{F}actor \textbf{G}raphs (ABCDEFG). Our key idea is to represent causal structures using \emph{extended factor graphs}, where feature nodes and intervention nodes are connected through auxiliary factor nodes. This extended factor graph formulation enables accurate and scalable distributional estimation of causal DAGs, while incorporating interventions with unknown targets and guaranteeing acyclicity. Moreover, it supports joint modeling of edge probabilities as coupled random variables, capturing complex dependencies among edges. ABCDEFG also possesses strong theoretical guarantees: we prove that the argmax of the estimated posterior recovers the true causal graph up to an equivalence class. 

\textbf{Contributions.} Our core contributions include: (1) we introduce a new parametric model for sampling extended factor graphs that are acyclic by construction and have explicit intervention nodes; (2) we develop a variational Bayesian approach for discovering causal extended factor graphs from interventional data with known or unknown targets; (3) we integrate sum-product networks into the generative model to flexibly model complex joint distributions over causal edges; (4) we develop new theoretical results connecting our Bayesian framework to the identifiability guarantees of score-based methods; and (5) we demonstrate the effectiveness of ABCDEFG on a large-scale single-cell perturbation dataset, recovering both known and novel gene-to-gene and growth factor-to-gene interactions.

\begin{table}[t]
  \small
  \caption{Summary of the proposed and existing approaches. $n$ and $m$ denotes the number of vertices in a graph/factor graph and $m$ denotes the number of factors in a factor graph. Max nodes and samples indicate the size of the largest dataset evaluated in the original publication.}
  \label{tab:methods}
  \vskip 0.15in
  \centering
  \begin{tabular}{lccccccc}
    \toprule
    Method & DAG & Graph & Guaranteed & Intvn & Unknown & Max & Max\\
    & Uncertainty & Model Size & Acyclic & Data & Target & Nodes & Samples \\
    \midrule

    NO-TEARS& \xmark & $O(n^2)$ & \xmark & \xmark & \xmark & $100$ & $7,466$\\
    DCDI & \xmark & $O(n^2)$ & \xmark & \cmark & \cmark & $100$ & $10^{6}$\\
    DAGMA & \xmark & $O(n^2)$ & \xmark & \xmark & \xmark & $2,000$ & $1,000$ \\
    DCD-FG & \xmark & $O(mn)$ & \xmark & \cmark & \xmark & $1,000$ & $87,590$ \\
    ENCO & \xmark & $O(n^2)$ & \xmark & \cmark & \xmark & $1,000$ & $110,000$\\
    SDCD & \xmark & $O(n^2)$ & \xmark & \cmark & \xmark &$4,000$ &$10,500$\\
    DeepITE & \xmark & $O(n^2)$ & \xmark & \cmark & \cmark &$500$ &$10,000$\\
    LIT & \xmark & $O(n^2)$ & \xmark & \cmark & \cmark &$16$ &$32$\\
    iSCAN & \xmark & $O(n^2)$ & \xmark & \cmark & \cmark &$50$ & $1,000$\\
    BaCaDI & \cmark & $O(n^2)$ & \xmark & \cmark & \cmark & $20$ & $300$\\
    ProDAG & \cmark & $O(n^2)$ & \cmark & \xmark & \xmark & $100$ & $7,466$\\
    DECI & \cmark & $O(n^2)$ & \xmark & \xmark & \xmark & $64$ & $5,000$\\
    DP-DAG & \cmark & $O(n^2)$ & \cmark & \xmark & \xmark & $100$ & $1,000$\\
    VDESP & \cmark & $O(n^2)$ & \cmark & \xmark & \xmark & $20$ & $4,200$ \\
    ABCDEFG (ours)  & \cmark & $O(mn)$ & \cmark & \cmark & \cmark & $1,000$ & $31,425$\\
    \bottomrule
  \end{tabular}
\end{table}

\textbf{Related Work.}\label{related}
Classical causal discovery methods are typically divided into constraint-based and score-based methods. Constraint-based methods date back to the 90s when \citet{spirtes91pc} proposed the PC algorithm. In contrast, score-based differentiable causal discovery methods have gained popularity in recent years due to their better performance and computational efficiency. \citet{NEURIPS2018_e347c514} pioneered the formulation of causal DAG discovery as a continuous optimization problem under a linear causal model, using an augmented Lagrangian approach with a matrix exponential constraint to enforce acyclicity. \citet{lee2019scaling} built on this by designing a polynomial regression loss tailored to gene expression data and reducing computational cost. Subsequent works improved performance and expanded the modeling framework. \citet{NEURIPS2022_36e2967f} proposed an alternative log-det function for the acyclicity constraint, resulting in better performance, better-behaved gradient and faster convergence. \citet{lippe2022efficient} designed an optimization strategy alternating between distribution and graph fitting and proved convergence to the true graph under specific conditions.

A parallel line of work developed Bayesian methods for causal discovery. \citet{cundy2021bcd} applied variational inference~(VI) to linear Gaussian SEMs. \citet{annadani2023bayesdag} adopted the NoCurl DAG model \cite{pmlr-v139-yu21a} and derived a VI method for the parameters. \citet{charpentier2022differentiable} proposed a fully probabilistic and differentiable DAG model and performs VI by maximizing the ELBO. \citet{geffner2024deep} developed a Bayesian method based on a previous probabilistic DAG model~\cite{lippe2022efficient} and applied a flow-based generative model for distributional fitting. \citet{thompson2024prodag} proposed a Bayesian method for DAGs by first pruning a weighted matrix to be acyclic and projecting it onto an L1 ball. \citet{bonilla2024variational} designed a differentiable DAG distribution using a continuous relaxation of permutation~\cite{pmlr-v119-prillo20a}. These Bayesian methods tend to be significantly less scalable than the score-based methods, as reflected in the relatively small datasets used for evaluation.

The methods discussed above focus exclusively on observational data and are not designed to incorporate interventional data, which is critical for accurate causal discovery in applications such as computational biology. To address this, a separate line of work has explored causal discovery with interventions. \citet{brouillard2020differentiable} proposed a differentiable method that incorporates observational and interventional data; guarantees identifiability with known or unknown intervention targets; and models nonlinear effects using deep neural networks. \citet{Lopez2022largescale} used factor graphs to learn a low-rank approximation of DAGs, a key foundation for our approach. \citet{pmlr-v235-nazaret24a} proposed a robust acyclicity penalty loss. \citet{hagele2023bacadi} set up a Bayesian framework for causal discovery with interventional data. Our work is also distinct from intervention target estimation methods, which can infer the nodes targeted by interventions but cannot simultaneously estimate the causal graph (e.g., iSCAN \cite{chen2023iscan}, LIT \cite{pmlr-v238-yang24d}, and DeepITE \cite{tao2024deepite}). We summarize these and related methods, along with our own, in \Cref{tab:methods}.

\section{Methods}\label{methods}
\subsection{Definitions}
Our definitions and notation closely parallel previous differentiable causal discovery methods \cite{brouillard2020differentiable}, but we summarize the key points here to make the presentation of our approach more self-contained. Let $X=\{X_1,\ldots,X_n\}$ be a set of random variables. A causal graphical model (CGM) for these variables consists of a joint distribution and a graph $\{G=(V,E),p(X)\}$. $G\in\mathcal{G}$ (where $\mathcal{G}$ is the set of DAGs) and $G$ and $p$ are related as follows:
    $$ p(X) = \prod_{i\in V}p(X_i|X_{\pi_i})$$
    Here, $\pi_i$ is the set of parents of vertex $i$ in $G$.
Intuitively, an intervention on a variable modifies its conditional dependence on its parent. Interventions can be performed on multiple variables simultaneously; the \emph{interventional target} for each intervention is thus a set of vertices $I\subset V$.

Given a CGM with $\{G,p(X)\}$, intervening on targets $I$ modifies $p$ into $p^I$:
    $$ p^I(X) = \prod_{i\in I}p^I(X_i|X_{\pi_i})\prod_{i\not\in I}p(X_i|X_{\pi_i})$$
Note that the causal sufficiency assumption is implicit in this definition of intervention. The $I$-faithfulness assumption ensures that $p^I(X_i|X_{\pi_i})\neq p(X_i|X_{\pi_i})$. A \emph{hard intervention} removes all dependence on parents, so $p^{I_k}(X_i|X_{\pi_i})=p^{I_k}(X_i)$. 

To accommodate multiple interventions, we define an \emph{intervention set} as $\mathcal{I} := (I_1,\ldots, I_{n^\calI})$, where $n^\calI$ is the number of interventions. Note that the intervention set may include multiple interventions with the same targets, $I_j=I_k$. For convenience, we include the observational distribution in the intervention set and define it as $I_1:=\emptyset$. We also abbreviate $p^{I_k}(X)$ as $p^{(k)}(X)$. The set of joint distributions induced by a causal graph and intervention set is $\mathcal{M}_{\mathcal{I}^*}(G)$, which we can factorize according to the Markov property: $ \mathcal{M}_{\mathcal{I}^*}(G):=\{p^{I_k}(X)=\prod_{i=1}^{n}p^{I_k}(X_i|X_{\boldsymbol{\pi_i}})\}$

Our goal is to estimate $q(G;\Lambda)$, a probability mass function (PMF) over $\mathcal{G}$ parameterized by a set of real numbers $\Lambda$. In estimating $q(G;\Lambda)$, we will make use of $f(X;\Phi)$ and $f^I(X;\Phi)$, density models of $p(X)$ and $p^{(k)}(X)$, respectively, parameterized by a set of real numbers $\Phi$. 

\subsection{Factor Directed Acyclic Graphs (f-DAGs)} \label{subsec:fg}
Our goal is to build a generative model for DAGs and ultimately a Bayesian framework for inferring causal DAGs. To do this, we start with a type of graph called a factor DAG (f-DAG), following \citet{Lopez2022largescale}.
An f-DAG is formally defined as follows:
\begin{definition}[\citet{Lopez2022largescale}]
    Given a set of nodes, $V$, and factors, $F$, a factor directed acyclic graph (f-DAG), denoted as $(V,F,E)$, is a directed acyclic graph $(V\cup F,E)$ where edges $E\subset\{(i,j):i\in V, j\in F \mbox{ or } i\in F, j\in V\}$.
\end{definition}
Given an f-DAG, we can preserve the connection between any two nodes~(factors) by removing all intermediate factors~(nodes) along paths. This results in a node-only~(factor-only) graph:
\begin{definition}[\citet{Lopez2022largescale}]
    Given an f-DAG, $D=(V,F,E)$, its half-square node graph is defined as $D^2[V]=(V,\{(i,j):\exists f\in F, (i,f),(f,j)\in E\})$, and half-square factor graph is defined as $D^2[F]=(F,\{(f,g):\exists i\in V, (f,i), (i,g) \in D \})$.
\end{definition}

Let $\mathbf{A}$ be the adjacency matrix of a causal DAG. An f-DAG can be viewed as a Boolean factorization of $\mathbf{A}$, $\mathbf{A}=\mathbf{U}\mathbf{V}$. Here $\mathbf{U}\in\{0,1\}^{n\times m}$ and $\mathbf{V}\in\{0,1\}^{m\times n}$ are binary node-to-factor and factor-to-node connection matrices. Intuitively, if $m < n$, the node-only half-square graph of an f-DAG can be interpreted as a low-rank approximation of the full-rank DAG, and the factors represent groups of related nodes (modules, topics, etc.). \citet{Lopez2022largescale} proved that, with probability exponentially approaching one, adding incorrect edges to a random graph increases its Boolean rank. Viewing an f-DAG as a Boolean matrix factorization of the binary adjacency matrix (Fig. \ref{fig:spn_fg}), this result implies that the low-rank property of the f-DAG acts as a regularization for graph structure and increases robustness to noisy edges. This low-rank assumption is common in computational biology~\cite{ye2013low,8002659}. 

We further extend the f-DAG framework for identifying unknown intervention targets. We model the effect of each intervention on target nodes via factors. This is a natural abstraction for interventions whose exact targets are unknown, such as drugs that affect a biological pathway. Suppose $\calI=\{I_1, \ldots, I_{n^{\calI}}\}$ is a set of unknown intervention targets, and $\W$ is a $n^{\calI}$-by-$m$ binary matrix, where $W_{kj}$ represents whether the $k$-th intervention targets the $j$-th factor. We next define extended f-DAGs, a.k.a. extended factor graphs.
\begin{definition}[Extended f-DAG]
    Let $D=(V,F,E)$ be an f-DAG and $\mathcal{I}=\{I_1,\ldots,I_{n^{\calI}}\}$ be a set of interventions. Let $\Xi=\{\xi_k, k\in[n^{\calI}]\}$ be $n^{\calI}$ nodes corresponding to the $n^{\calI}$ interventions. An extended f-DAG is defined as an f-DAG $\Dint=(V\cup \Xi, F, E\cup E^{\mathcal{I}})$ where $E^{\mathcal{I}}\subseteq\{(\xi_k, l): l\in F\}$, i.e. set of edges from intervention nodes to factors.
\end{definition}

\begin{definition}[Extended Half-Square Graph]\label{def:ext_half_sq}
    Let $\Dint$ be an extended f-DAG obtained from an f-DAG $D=(V,F,E)$ and a set of interventions $\mathcal{I}=\{I_1,\ldots,I_{n^{\calI}}\}$ and intervention nodes $\Xi=\{\xi_k, k\in[n^{\calI}]\}$. An extended half-square node graph is defined as $(\Dint)^2[V]=(V\cup \Xi, \{(i,j): \exists f\in F, (i,f), (f,j)\in E\cup E^{\calI}\})$.
\end{definition}
We could also define an extended half-square factor graph, but that would be the same as a regular half-square factor graph because there is no edge from any factor to any intervention. Thus, we omit it in Def. \ref{def:ext_half_sq}.
\subsection{Probabilistic Modeling of f-DAGs}\label{subsec:spn-fg}
\textbf{Generative Model for f-DAGs.}
A key innovation of our approach is a generative process for efficiently sampling large-scale f-DAGs that guarantees acyclicity by construction. This eliminates the need for computationally expensive acyclicity penalties used in differentiable causal discovery methods, ensures that all sampled graphs are acyclic, and forms the foundation for probabilistic causal f-DAG inference. 

Given a set of $n$ nodes, $\{v_i:i\in[n]\}$, and $m$ factors, $\{f_j:j\in[m]\}$, we construct an f-DAG by forming a partial order of nodes and factors together and determining the node-to-factor or factor-to-node edge connection~(Fig. \ref{fig:spn_fg}). Since node-to-node edges are disallowed in f-DAGs (nodes are only connected via factors), we do not need to explicitly model the relative order between nodes. Instead, we form a total order of factors, $\tau: [m] \rightarrow [m]$, such that $f_{\tau(1)}<\ldots<f_{\tau(m)}$. They partition all nodes into $m+1$ subsets and each node $v_i$ is randomly inserted into one partition, i.e. $\exists k\in[m], f_{\tau(k-1)} < v_i <f_{\tau(k)}$ or $v_i<f_{\tau(1)}$ or $v_i>f_{\tau(m)}$. We model this assignment using $n$ categorical distributions with $m+1$ categories, denoted as $\bY=\{Y_i:i\in[n]\}$. The second step determines edge \emph{existence}, regardless of direction. These edge connection probabilities are related to a joint distribution of all edge connections. We use a binary matrix $\boldsymbol{B}\in \{0,1\}^{n\times m}$ to represent edge connections. Thus, $\bY$ contains all the direction information and $\boldsymbol{B}$ contains all the connection information. Hence, $\bY$ and $\bB$ uniquely determine an f-DAG, and we can generate an f-DAG by sampling $\bY$ and $\bB$~(Fig. \ref{fig:spn_fg}).

\begin{figure*}[t!]
    \centering
    \includegraphics[width=1\linewidth]{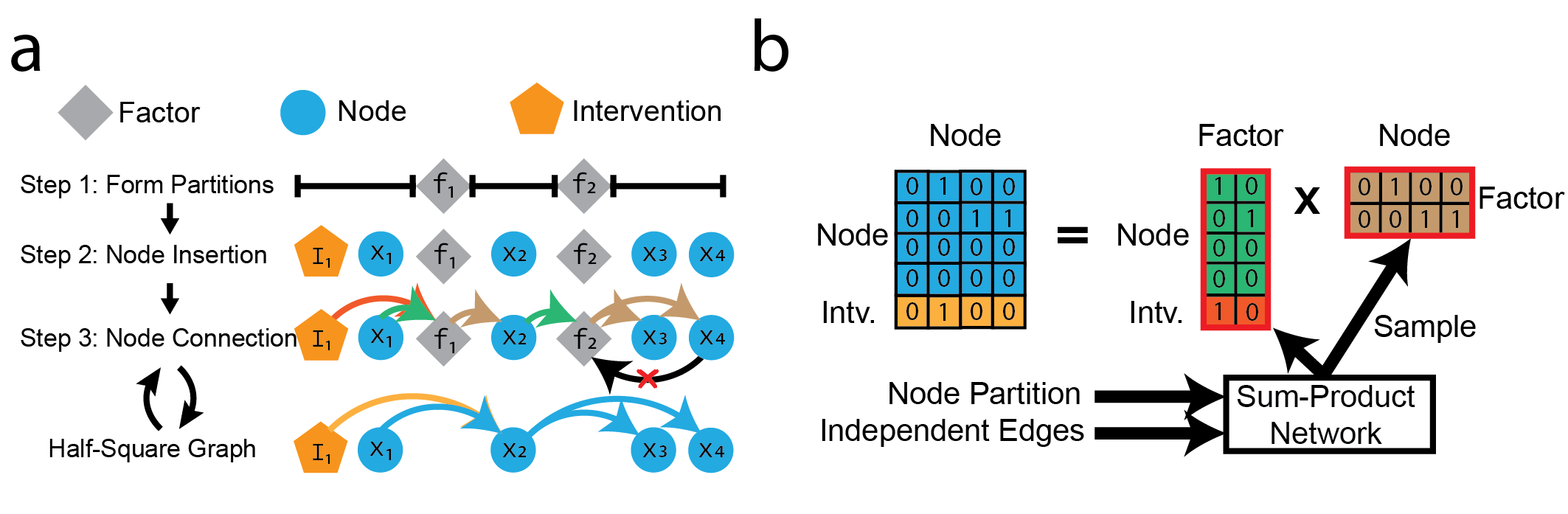}
    \caption{\textbf{Causal inference using extended factor graphs.} \textbf{(a)} Generative process for sampling extended factor graphs that are guaranteed to be acyclic. Factors are ordered to form partitions, then nodes and interventions are inserted into partitions. Finally, edges are added from earlier nodes, factors or interventions to later. Removing factors gives a ``half-square'' graph with direct node-to-node and intervention-to-node connnections. \textbf{(b)} An extended factor graph  factorizes a node/intervention-to-node adjacency matrix as a Boolean product of a node/intervention-to-factor and factor-to-node matrix. ABCDEFG samples edges in these matrices using either independent Bernoulli random variables or a joint PMF parametrized by a sum-product network.}
    \label{fig:spn_fg}
    \vskip -0.1in
\end{figure*}

\textbf{Sampling Independently or Jointly Distributed Causal Edges.}
Using the above generative process, we can infer a causal DAG by optimizing a score function with respect to $\bY$ and $\bB$. But what is the best way to sample $\bY$ and $\bB$? One possibility is to model the edges as independent Bernoulli random variables sampled using the Gumbel softmax trick~\cite{jang2017categorical}. However, such a naive approach neglects possible correlation between edges. A more general approach is to model the joint distribution of edges using a sum-product network (SPN) \cite{poon2011sum,shih2020probabilistic}. SPNs combine sum and product operations over latent variables, enabling flexible sampling from a categorical joint distribution (see Appendix~\ref{supp:complexity} for further details). We implemented and evaluated both strategies on real and simulated data.

\subsection{{Bayesian Causal Discovery of DAGs}}\label{subsec:bayes}

\textbf{A Differentiable Bayesian Framework for Causal Discovery.}
Let $\calG$ be the set of all DAGs. Consider a generative process where a DAG is first sampled from a prior, $p(G)$ with support on $\calG$, and a generative model $p(\X|G,I)$ under the intervention $I$. Given empirical observations, we can obtain a MAP estimate of the causal graph as $G^* = \arg\max_{G\in\calG}p(G|\X,I)$.

Because $|\calG|$ is super-exponential in $n$~\cite{10.1007/BFb0069178}, searching through the discrete space is computationally inefficient for large $n$. Instead, we resort to continuous optimization. As the true posterior is often intractable, we apply variational Bayes using a variational distribution $\qG$. In this way, we are able to find $G^*$ by optimizing a KL divergence: $G^* = \arg\min_{G\in\calG}KL(\qG||p(G|\X,I))$. In real experimental scenarios, the random intervention is replaced with Monte Carlo sampling, $I_1,\ldots,I_{n^\calI}$. From our derivation (Appendix \ref{appx_sub:id_bayesian}), minimizing the KL divergence is equivalent to maximizing the evidence lower bound (ELBO):
\begin{align}
    q^*(G) = & \arg\max_{\qG}\sum_{k=1}^{n^{\calI}}\mathbb{E}_{\pk(\X|G^*)}\left[\mathbb{E}_{\qG}\left[\log\pkXG \right]\right] 
    - KL\left(\qG || \pG\right). \label{eq:max_elbo}
\end{align}
This ELBO objective is directly connected to autoencoding variational Bayes~\cite{KingmaW13}. A slight difference compared to the traditional autoencoding variational Bayes setting is that we treat the causal graph as a constant during the likelihood calculation, so the expectation is over $\pk(\X|G^*)$ instead of $\pk(\X)$. (We provide a detailed derivation of the ELBO in the Appendix.) The posterior can be estimated by optimizing the ELBO to yield $q^*(G)=\pkGX$, assuming enough capacity of the variational family. 

As mentioned in \cref{subsec:fg}, we can narrow down the search space by considering extended f-DAGs as a reasonable low-rank approximation of the true causal DAG. In this work, we use either independent Bernoullis or SPNs as a parametric model for f-DAGs, but the Bayesian framework is general to parametric DAG models.

\subsection{Amortized Bayesian Causal Discovery of Extended Factor Graphs}
\label{abcdefg}
With the problem setup in \cref{subsec:bayes}, we now formally introduce our method, Amortized Bayesian Causal Discovery of Extended Factor Graphs~(ABCDEFG). (Note that ``amortized'' here refers to using a common inference function in contrast to traditional mean-field variational inference. Variational autoencoders (VAEs) are a type of amortized variational inference \cite{Agrawal:2020}.) Given a set of random variables $\X=\{X_i:i\in[n]\}$ generated via a causal graph $G^*$, we apply a Bayesian method by estimating $p(G|\X,I^*)$ via optimization as described in section \ref{subsec:bayes}:
\begin{align*}
    q^*(G) = & \arg\max_{\qG}\sum_{k=1}^{n^{\calI}}\mathbb{E}_{\pk(\X|G^*)}\left[\mathbb{E}_{\qG}\left[\log\pkXG \right]\right] 
    - KL\left(\qG || \pG\right).
\end{align*}
The key to convert discrete search into continuous optimization is thus to create a differentiable parametric model for DAGs and estimate the ELBO using Monte Carlo sampling. We assume the true causal graph is or can be approximated by an f-DAG. Thus, we use either independent Bernoullis sampled by Gumbel softmax or joint PMF sampled from an SPN to parameterize $\qG$. 

The model architecture~(bottom panel of Fig. \ref{fig:abcdefg}) consists of an f-DAG parametric model (Gumbel softmax or SPN) and a VAE for data distribution fitting. The output is a node-to-factor matrix $\U\in\mathbb{R}^{n\times m}$ and a factor-to-node matrix $\V\in\mathbb{R}^{m\times n}$. Next, we model the data distribution under the f-DAG as $p(\X)=\int_{\Z}\prod_{j=1}^{n}g(X_j|\Z_{\boldsymbol{\pi_j}})\prod_{i=1}^{m}f(Z_i|\X_{\boldsymbol{\pi_i}})d\Z$. Here, $\boldsymbol{\pi_i}$ and $\boldsymbol{\pi_j}$ are the parent nodes and factors in the f-DAG. Instead of using separate encoding and decoding functions to obtain the posterior of each $Z_j$ and conditional likelihood of each $X_i$, we follow \citet{Lopez2022largescale} and amortize all conditional distributions into a single encoding and decoding feed-forward neural network. Causal relations are injected into the VAE via masking operations $\Uj\odot \X$ and $\Vi\odot \Z$, where $\Uj$ is the $j$-th column of $\U$, $\Vi$ is the $i$-th column of $\V$ and $\odot$ denotes the Hadamard product.

When the intervention targets are unknown, the causal discovery problem can be treated as recovering an extended f-DAG with intervention nodes. Equivalently, our Gumbel softmax or SPN sampling procedure can be extended to generate an intervention-to-factor matrix $\W\in \{0,1\}^{k\times m}$. The causal mask operation becomes $ [\U_{\boldsymbol{j}}\odot \X; \W_{\boldsymbol{j}}\odot \boldsymbol{I}]$ where $\boldsymbol{I}$ is a one-hot encoding of the intervention. We can apply the same optimization approach to jointly infer the causal graph and intervention targets. Extended f-DAGs could also include intervention information such as the dosage of a chemical treatment, though we did not explore this in detail here. 

\subsection{Identifiability} 
We next provide identifiability guarantees for our approach. Our main theorem proves that the DAG with highest posterior probability (MAP estimate) belongs to the same equivalence class as the true causal DAG. We use the notion of $\mathcal{I}$-Markov equivalence from \cite{brouillard2020differentiable}: 
two DAGs $G_1$ and $G_2$ are $\mathcal{I}$-Markov equivalent if and only if $\mathcal{M}_{\mathcal{I}}(G_1)=\mathcal{M}_{\mathcal{I}}(G_2)$. Our theorem relies on the same four assumptions as previous identifiability results for differentiable causal inference methods \cite{brouillard2020differentiable}: sufficient model capacity, $\mathcal{I}$-faithfulness, positivity, and finite differential entropy. This result applies to any DAG, including half-square graphs obtained from f-DAGs.

\begin{theorem}[Identifiability via ELBO maximization]\label{thm:elbo_identify}
    Let $\X$ be a set of causally related random variables with a causal DAG $G^*$ and $\calIstar$ be a set of interventions with $I^*_1=\emptyset$. Let $\calG$ be a subset of all causal DAGs and $q^*(G)$ be an optimal graph distribution from the optimization problem:
    $$ \sup_{\qG:supp(q)\subseteq\calG}\calL(\qG), $$
    where
    \begin{align*}
        & \calL(\qG) = \mathbb{E}_{\qG}\left[\scoreIstar\right] - \beta KL(\qG||p(G)),\; \beta > 0,\\
        & \scoreIstar = \sup_{\bPhi}\sum_{k=1}^{n^{\calIstar}}\mathbb{E}_{\pk(\X)}\left[\log \fk(\X|G;\bPhi)\right]-\lambda |G|. 
    \end{align*}
    In addition, assume the following:
    \begin{enumerate}
        \item Sufficient capacity: The set of distributions from our parametric models contains the ground truth interventional distributions: $\{\pk(\X):k\in[n^{\calIstar}]\}\in\calF_{\calIstar}(G^*)$ where $\calF_{\calIstar}(G^*)=\{\{\fk(\X|G^*;\bPhi)\}: \bPhi\in\Omega(\bPhi)\}$. 
        \item $\mathcal{I}$-faithfulness as defined in \cite{brouillard2020differentiable} (See appendix \ref{appx:identifiability}, Thm. \ref{thm:supp:dcdi} for details). \\
        \item Positivity: $\forall G, I, \bPhi, \fk(\X|G,I;\bPhi)>0$. \\
        \item Finite differential entropy: $\forall k\in[n^{\calIstar}]$, $\left|\mathbb{E}_{\pk(\X)}\left[\log \pk(\X)\right]\right| < +\infty$.
    \end{enumerate}
    If $G^*\in\calG$, then, under the assumptions 1-4~\cite{brouillard2020differentiable} and with a proper $\beta > 0$, $\hat{G}=\arg\max_{G}q^*(G)$ is $\calIstar$-Markov equivalent to $G^*$.
\end{theorem}

The key idea of the proof is that any posterior distribution whose MAP is not $\calIstar$-Markov equivalent to the true causal DAG must have a lower ELBO. Here, we present a sketch proof. See Appendix \ref{appx_sub:id_bayesian} for details. 

\emph{Proof.}
    The proof is by contradiction. Suppose $\exists \hat{G}=\arg\max_{G}q^*(G)$ that is not $\calIstar$-Markov equivalent to $G^*$. We can create another distribution $q'$ such that $q^*(\hat{G}) - q'(\hat{G}) = q'(G^*) - q^*(G^*) = \epsilon > 0$ and for any other graph $G$, $q'(G)=q^*(G)$. From algebraic calculation, we have
    \begin{align*}
        & \calL(q')-\calL(q^*) = \epsilon \left(\scoreIGstar - S_{\calIstar}(\hat{G})\right) + \beta \Delta.
    \end{align*}
    Because $\scoreIGstar - S_{\mathcal{I}}(\hat{G}) > 0$, $\exists \beta > 0$ such that $\calL(q')-\calL(q^*) > 0$. Then, we have a contradiction about $q^*$ being an optimal solution to the optimization problem.
\hfill $\square$\\

Furthermore, our method can be extended to the unknown-target setting by replacing the causal DAG with an interventional DAG ($\calI$-DAG)\cite{pmlr-v80-yang18a}, following a derivation analogous to the known-target case above.

So far, we have shown that our Bayesian framework is able to recover the true causal graph up to an Interventional Markov equivalent class. This result is general to any causal DAG, including f-DAGs. We motivated using f-DAG not only because of its robustness to noise\cite{Lopez2022largescale}, but also because of the underlying physical meaning - the factors inform us about the organization of causal relations. Thus, we are motivated to answer the following question: does identifying the causal DAG guarantee identifying the underlying f-DAG? It turns out a subset of f-DAGs can be identified given a fixed number of factors. We call this subset "identifiable f-DAGs" and denote it as $\calD_m$ where $m$ is the number of factors. Below we give a formal definition of $\calD_m$ based on relevant concepts.

\begin{definition}
    Let $D=(V, F, E)$ be any f-DAG, $\forall f\in V\cup F$. Denote $par(\cdot;D)$ and $chd(\cdot;D)$ as the set of parents and children of a vertex in $D$. The set of unique parents and children of $f$ are defined as $P_{f}(G):=\{i: i\in par(f;D), chd(i;D)=\{f\}\}$  and $C_{f}(D):=\{j: j\in chd(f;D), par(j;D)=\{f\}\}$.
\end{definition}

\begin{definition}\label{def:dm}
    Let $\mathcal{I}$ be a set of interventions. $\calD_m \subseteq \{(V,F,E):|F|=m\}$ is defined as the set of f-DAGs with $m$ factors and the following properties:
    \begin{enumerate}
        \item $\forall f\in F$, $P_f(\Dint)\neq\emptyset$ and $C_f(\Dint)\neq\emptyset$.
        \item $\forall f_1, f_2 \in F$ such that $f_1\neq f_2, f_1\rightarrow f_2 $ in $(\Dint)^2[F]$, $|P_{f_1}(\Dint)|>1$ (inclusively) or $|P_{f_2}(\Dint)|>1$.
        \item $\forall f\in F$, $|par(f;\Dint)| > 1$, there is at most one factor $g\in par(f;(\Dint)^2[F])$ such that $|par(g;\Dint)| = 1$.
    \end{enumerate}
\end{definition}
Intuitively, the three additional conditions for f-DAGs mean
\begin{enumerate}
    \item Any factor should have a unique parent and unique child that distinguish it from other factors.
    \item There cannot be adjacent ``chain" or ``tree" structures in the f-DAG.
    \item There should be enough v-structures in the f-DAG.
\end{enumerate}
Based on the definition, we further define the resulting half-square graphs.
\begin{definition}\label{def:gm}
    Given a set of f-DAGs $\calD_m$ defined as in Def. \ref{def:dm}, $\calG_m$ is defined as the set of all DAGs having a rank-$m$ f-DAG factorization:
    $$ \calG_m:=\{G:\exists D \in\calD_m, G=D^2[V]\}.$$
\end{definition}

Def. \ref{def:dm} and \ref{def:gm} describe exactly the subset of f-DAGs that are identifiable given our proposed Bayesian framework. Because we already proved identifiability of any causal DAG in Thm. \cref{thm:elbo_identify}, the identifiability of f-DAG is straightforward given the following lemma:
\begin{lemma}\label{lemma:1}
    Let $D_1=(V,F,E_1)$ and $D_2=(V,F,E_2)$ be two f-DAGs on the same set of nodes, $V$, and factors, $F$, and $\mathcal{I}$ be a set of interventions. Denote $\simeq_\mathcal{I}$ as the $\mathcal{I}$-Markov equivalence relation. Let $\Xi=\{\xi_k:k\in[n^{\calI}]\}$ be intervention nodes. In addition, suppose $D_1, D_2\in \calD_m$ defined as in Def. \ref{def:dm}. Then, under a permutation of factors, we have $D_1\simeq_\mathcal{I}D_2\iff D_1^2[V]\simeq_\mathcal{I}D_2^2[V]$.
\end{lemma}

The exact proof is lengthy. For conciseness, we present a proof sketch here and include all details in Appendix \ref{appx:sub_iden_fdag}.\\

\noindent\emph{Proof.}
    By \cref{thm:yang}~\cite{pmlr-v80-yang18a}, we convert the proof of $\calI$-Markov equivalence to proof of equivalent graph structure. The general strategy is proof by contradiction under a discussion of different graph structures.
    
    We first show the forward direction. Suppose $D_1\simeq_\mathcal{I}D_2$ but $D_1^2[V]\not\simeq_\mathcal{I}D_2^2[V]$. There must be a mismatch in skeleton or v-structure. The former case implies a mismatch in skeleton between f-DAGs, while the latter implies a v-structure mismatch between f-DAGs. Hence, we have contradiction in both cases.

    For the reverse direction, the main idea is to show that there is a bijection $\phi: F\rightarrow F$ such that $par(f;D_1) = par(\phi(f);D_2)$ and $chd(f;D_1)=chd(\phi(f);D_2)$. In fact, we partition $F$ into three subsets in $D_1$: (1) $F_1:=\{f:f\in F, |par(f;D)|=|chd(f;D)|=1\}$, (2) $F_2:=\{f:f\in F, |par(f;D)|=1, |chd(f;D)|>1\}$ and (3) $F_3:\{f:f\in F, |par(f;D)|>1$. Similarly, we also partition factors into $F_1',F_2',F_3'$ in the same way in $D_2$. Next, we introduce three propositions:
    \begin{proposition}\label{prop:1}
        Given $\Doneintsq \simeq \Dtwointsq$ under the assumptions of lemma \ref{lemma:1}. Then, $\forall j\in F$ such that $par(j;\Dint_1) = \{i\}, chd(j;\Dint_1) = \{k\}$, $\exists j'\in F$ such that $par(j';\Dint_2)=\{i\}, chd(j';\Dint_2)=\{k\}$ or $par(j';\Dint_2)=\{k\}, chd(j';\Dint_2)=\{i\}$.
    \end{proposition}
    \begin{proposition}\label{prop:2}
        Given $\Doneintsq \simeq \Dtwointsq$ under the assumptions of lemma \ref{lemma:1}. Then, $\forall j\in F$ such that $par(j;\Dint_1)=P_j(\Dint_1) = \{i\}, |chd(j;\Dint_1)| > 1$, $\exists j' \in F$, $par(j';\Dint_2)=\{i\}, chd(j';\Dint_2)=chd(j;\Dint_1)$.
    \end{proposition}
    \begin{proposition}\label{prop:3}
        If $\Doneintsq \simeqint \Dtwointsq$ under the assumptions of lemma \ref{supp:lemma:1}, then $\forall j\in F$ such that $|par(j;\Dint_1)| > 1$, $\exists j'\in F$ such that $par(j;\Dint_1)=par(j';\Dint_2)$ and $chd(j;\Dint_1)=chd(j';\Dint_2)$.
    \end{proposition}
    
    Because $\Doneintsq\simeqint\Dtwointsq$, by proposition \ref{prop:1}, $\forall j\in F_1$, $\exists j'\in F_1'$ with the same parents and children and $\forall j'\in F_1'$, $\exists j\in F_1$ with the same parents and children. Hence, there is a bijection between $F_1$ and $F_1'$. Similarly, we have bijections from $F_2$ to $F_2'$ by proposition \ref{prop:2} and $F_3$ to $F_3'$ by proposition \ref{prop:3}. Finally, we have a bijection $\phi:F\rightarrow F$ from all factors in $\Dint_1$ to factors in $\Dint_2$. $\forall f\in F_2\cup F_3$, $f$ and $\phi(f)$ share the same parent(s) and children. $\forall f\in F_1$, $f$ and $\phi(f)$, either $par(f;\Dint_1)=par(\phi(f);\Dint_2), chd(f;\Dint_1)=chd(\phi(f);\Dint_2)$ or $par(f;\Dint_1)=chd(\phi(f);\Dint_2), chd(f;\Dint_1)=par(\phi(f);\Dint_2)$. Since $|par(f;\Dint_1)|=|chd(f;\Dint_1)|=1$, flipping the parent and child maintains the same skeleton and does not introduce v-structure.
    
    Finally, we conclude that $\Dint_1$ and $\Dint_2$ share the same skeleton and v-structures and $D_1^2[V] \simeqint D_2^2[V] \implies D_1 \simeqint D_2$
\hfill $\square$\\

Finally, we present the main theorem on the identifiability of f-DAGs. Here, we consider a variational distribution over f-DAGs, $\qD$. With a slight abuse of notation, we use $q(D^2[V])$ to represent the induced distribution over half-squared graphs given a f-DAG distribution, $\qD$.

\begin{theorem}[Identifiability of the f-DAG via ELBO maximization]\label{thm:iden_fdag}
    Let $\X$ be a set of causally related random variables with a causal DAG $G^*$ and $\calIstar$ be a set of interventions with $I^*_1=\emptyset$. Suppose $(G^*)^{\calIstar} \in \calG^{\calIstar}_m$. Let $q^*(D)$ be an optimal graph distribution from the optimization problem:
    $$ \sup_{q(D):supp(q)\subseteq\calD_m}\calL(\qD), $$
    where
    \begin{align*}
        & \calL(\qD) = \mathbb{E}_{q(D^2[V])}\left[S_{\calIstar}(D^2[V])\right] - \beta KL(q(D^2[V])||p(G)),\; \beta > 0,\\
        & S_{\calIstar}(D^2[V]) = \sup_{\bPhi}\sum_{k=1}^{n^{\calIstar}}\mathbb{E}_{\pk(\X)}\left[\log \fk(\X|D^2[V];\bPhi)\right]-\lambda |D^2[V]|. 
    \end{align*}
    If $G^*=(D^*)^2[V]\in\calG_m$, then, under the assumptions 1-4 as in \cref{thm:elbo_identify}~\cite{brouillard2020differentiable} and with a proper $\beta > 0$, $\hat{D}=\arg\max_{D}q^*(D)$ is $\calIstar$-Markov equivalent to $D^*$ up to a permutation of factors.
\end{theorem}

The theorem is a direct result of Theorem \ref{thm:elbo_identify} and lemma \ref{supp:lemma:1}. By Theorem \ref{thm:elbo_identify},  $(\hat{D})^2[V] \simeqintstar (D^*)^2[V]$. Since $\hat{G},G^*\in\calG_m$, $\exists \hat{D},D^*$ such that $\hat{G}=\hat{D}^2[V]$ and $G^*=(D^*)^2[V]$. By lemma \ref{supp:lemma:1}, $\hat{D}\simeqintstar D^*$ under a permutation of factors.

\section{Experiments}\label{results}
\subsection{Simulation Results}
We simulated data based on the approach of ~\cite{Lopez2022largescale}. We further explored the effects of correlations between edge probabilities, which our approach explicitly models but previous approaches do not, by constructing an SPN and then sampling from the joint distribution of edges. We also simulated interventions with unknown targets. To evaluate our method, we benchmarked ABCDEFG on 24 datasets and compared with four SOTA score-based methods: DCDI~\cite{brouillard2020differentiable}, DCDFG~\cite{Lopez2022largescale}, ENCO~\cite{lippe2022efficient} and SDCD~\cite{pmlr-v235-nazaret24a}.  The 24 datasets include eight types of SEMs -- a combination of (1) linear vs. non-linear causal effects, (2) independent vs. jointly distributed edge probabilities, and (3) hard vs. soft interventions. Each simulated graph includes 100 nodes and 10 factors. We simulated three separate graphs for each type of SEM. Similar to previous studies, we report Structural Hamming Distance (SHD) and F1 score for edge prediction. We used consistent hyperparameter settings for ABCDEFG across all simulations (Appendix \ref{supp:hparam_setting}). ABCDEFG significantly outperformed all other approaches on graphs with nonlinear causal effects and edge probabilities that are jointly distributed and sampled from an SPN (Table \ref{table:Scored based methods F1 and SHD targeted}). ABCDEFG performed similarly or better than SOTA methods on nonlinear SEMs, though SDCD showed strong performance in the nonlinear, non-SPN setting (Fig. \ref{supp:fig:linear_bechmark}). We also found that the other methods frequently produced cyclic graphs that required heuristic pruning to obtain a final DAG (Fig.\ref{supp:fig:num_acyc_bar}, Fig.\ref{supp:fig:dag_edge}). 

\begin{table}[t!]
\caption{F1 score and SHD of Scored methods on Nonlinear Targeted Simulated Datasets }
\label{table:Scored based methods F1 and SHD targeted}
\vskip 0.15in
\begin{center}
\begin{small}
\begin{sc}
\begin{tabular}{llcccc}
\toprule
Metric & Method & Hard & Soft & SPN & SPN\\
&& Intvn & Intvn & Hard & Soft\\
\midrule
F1 & DCDI  &   $0.19\pm0.05$ &   $\underline{0.25\pm0.07}$   &  $0.34\pm0.01$     &   $0.35\pm0.04$  \\
&DCDFG   &   $0.05\pm0.08$ &   $0.20\pm0.14$   &  $0.23\pm0.18$     &   $0.57\pm0.14$   \\
&ENCO   &   $0.10\pm0.01$ &   $0.10\pm0.03$   &  $0.25\pm0.01$     &   $0.23\pm0.03$     \\
&SDCD    &   \boldmath$0.31\pm0.01$ &   \boldmath$0.30\pm0.06$   &  $0.25\pm0.02$  &   $0.30\pm0.06$    \\
&ABCDEFG &   $\underline{0.29\pm0.03}$ &   $\underline{0.25\pm0.01}$   &  \boldmath$0.64\pm0.01$     &   \boldmath$0.61\pm0.03$ \\
&ABCDEFG &\underline{$0.29\pm0.04$} &   $0.21\pm0.01$   &  $\underline{0.61\pm0.02}$     &   $\underline{0.60\pm0.02}$ \\
&(SPN)\\
\midrule
SHD & DCDI   &   \underline{$740\pm291$} &   \underline{$559\pm106$}   &  $4293\pm301$     &   $3337\pm120$  \\
&DCDFG   &   $2513\pm0$ &   $900\pm272$   &  $2500\pm198$     &   \boldmath$2030\pm125$   \\
&ENCO    &   $1952\pm126$ &   $1992\pm141$   &  $2855\pm177$     &   $2896\pm100$     \\
&SDCD    &   \boldmath$421\pm77$ &   \boldmath$421\pm78$   &  $2973\pm72$  &   $2793\pm83$    \\
&ABCDEFG &   $1114\pm328$ &   $1406\pm361$   &  \boldmath$2046\pm49$     &   $2248\pm200$ \\
&ABCDEFG  &$1125\pm248$ &   $1791\pm249$   &  \underline{$2206\pm81$}     &   \underline{$2228\pm85$} \\
&(SPN)\\
\bottomrule
\end{tabular}
\end{sc}
\end{small}
\end{center}
\vskip -0.1in
\end{table}

\begin{table}[t!]
\caption{F1 score and SHD of ABCDEFG on Nonlinear Untargeted Simulated Datasets }
\label{table:F1 and SHD for ABCDEFG on untargeted}
\vskip 0.15in
\begin{center}
\begin{small}
\begin{sc}
\begin{tabular}{llcccc}
\toprule
Metric&Method & Hard & Soft & SPN & SPN\\
&& Intvn & Intvn & Hard & Soft\\
\midrule
F1 &ABCDEFG &   $0.23\pm0.01$ &   $0.23\pm0.05$   &  $0.22\pm0.06$     &   $0.46\pm0.03$ \\
&ABCDEFG &$0.20\pm0.02$ &   $0.17\pm0.02$   &  $0.28\pm0.04$     &   $0.55\pm0.05$ \\
&(SPN)\\
&ABCDEFG &   $0.36\pm0.01$ &   $0.38\pm0.01$   &  $0.46\pm0.10$     &   $0.85\pm0.02$ \\
&Intv.\\
&ABCDEFG &$0.35\pm0.01$ &   $0.35\pm0.02$   &  $0.51\pm0.01$     &   $0.84\pm0.01$ \\
&(SPN) Intv.\\
\midrule
SHD &ABCDEFG &   $857\pm112$ &   $1121\pm261$   &  $3067\pm56$     &   $2632\pm217$ \\
&ABCDEFG  &$1076\pm326$ &   $1399\pm342$   &  $3132\pm148$     &   $2307\pm239$ \\
&(SPN)\\
&ABCDEFG &   $1659\pm240$ &   $1426\pm346$   &  $2584\pm440$     &   $1021\pm56$ \\
&Intv.\\
&ABCDEFG &$1761\pm204$ &   $1516\pm280$   &  $2438\pm187$     &   $1071\pm71$ \\
&(SPN) Intv.\\
\bottomrule
\end{tabular}
\end{sc}
\end{small}
\end{center}
\vskip -0.1in
\end{table}

We next evaluated how ABCDEFG performs for interventions with unknown targets, a key advantage of our approach. To test target identification, we generated causal graphs but withheld the intervention target information during inference. SDCD, ENCO, and DCDFG cannot incorporate interventions with unknown targets. Although DCDI can in principle identify both causal relations and unknown intervention targets, we excluded it from this evaluation because it required extremely long runtimes and showed poor performance in the easier targeted case. In addition to SHD and F1 of the causal graph, we evaluated the accuracy of the intervention-to-node graph (Table \ref{table:F1 and SHD for ABCDEFG on untargeted}). The accuracy of inferred node-to-node relationships was lower compared to interventions with known targets, indicating that causal inference is more challenging under unknown interventions. Nevertheless, ABCDEFG inferred the intervention targets more accurately than the node-to-node causal relationships, achieving relatively high precision and recall, particularly for SPN-simulated graphs. 

We also benchmarked ABCDEFG against SOTA Bayesian causal inference methods: BaCaDi~\cite{hagele2023bacadi}, ProDAG~\cite{thompson2024prodag}, DECI~\cite{geffner2024deep} and VI-DP-DAG~\cite{charpentier2022differentiable}. These methods required significantly longer runtimes than the score-based approaches, so we used smaller datasets with 16 nodes and 260 samples. ABCDEFG and ProDAG were significantly faster than the other Bayesian approaches (see Table~\ref{supp:table:Time_Bayesian}). For each method, we sampled 100 graphs from the posterior after training. ABCDEFG outperformed the other methods by achieving the highest F1 score and the lowest SHD across four different linear and nonlinear settings (Table~\ref{table:Bayesian methods F1 and SHD}). We also evaluated the posterior calibration of each method by comparing the expected and predicted edge probabilities. The posterior estimated by ABCDEFG showed the best match between the predicted edge probability and empirical estimation~(Fig. \ref{fig:combined_results}a).

\begin{table}[t!]
\caption{F1 score and SHD of Bayesian methods on Simulated Datasets with 16 Nodes.}
\label{table:Bayesian methods F1 and SHD}
\vskip 0.15in
\begin{center}
\begin{small}
\begin{sc}
\begin{tabular}{llcccc}
\toprule
Metric&Method & LINEAR & LINEAR & NONLINEAR & NONLINEAR\\
&& & SPN &&  SPN\\
\midrule
F1 & BaCaDi   &   $0.18\pm0.02$ &   $0.22\pm0.03$   &  \underline{$0.16\pm0.03$}     &   $0.20\pm0.03$  \\
&DECI   &   $0.09\pm0.02$ &   $0.11\pm0.01$   &  $0.08\pm0.01$     &   $0.08\pm0.02$   \\
&VI-DP-DAG    &   $0.20\pm0.04$ &   $0.20\pm0.03$   &  $0.13\pm0.00$     &   $0.21\pm0.06$     \\
&ProDAG    &   $0.17\pm0.01$ &   $0.20\pm0.02$   &  \underline{$0.16\pm0.03$}  &   $0.23\pm0.05$    \\
&ABCDEFG &   \boldmath$0.74\pm0.13$ &   \boldmath$0.49\pm0.13$   &  \boldmath$0.23\pm0.31$     &   \boldmath$0.35\pm0.24$ \\
&ABCDEFG &\underline{$0.40\pm0.03$} &   \underline{$0.24\pm0.06$}   &  $0.13\pm0.13$     &   \underline{$0.30\pm0.24$} \\
&(SPN)\\
\midrule
SHD & BaCaDi   &   $108.28\pm0.95$ &   $106.50\pm1.49$   &  $109.34\pm0.82$     &   $107.78\pm0.85$  \\
&DECI   &   $37.27\pm0.76$ &   \underline{$41.89\pm5.36$}   &  $36.75\pm4.17$     &   $41.88\pm4.35$   \\
&VI-DP-DAG    &   $83.88\pm4.32$ &   $79.48\pm1.97$   &  $86.51\pm5.35$     &   $79.78\pm4.28$     \\
&ProDAG    &   $98.24\pm1.56$ &   $94.79\pm1.56$   &  $81.16\pm0.60$  &   $88.00\pm2.52$    \\
&ABCDEFG &   \boldmath$12.74\pm5.02$ &   \boldmath$29.25\pm3.57$   &  \boldmath$22.14\pm5.44$     &   \boldmath$27.68\pm0.88$ \\
&ABCDEFG  &\underline{$34.40\pm1.80$} &   $43.11\pm2.86$   &  \underline{$30.38\pm6.76$}     &   \underline{$34.31\pm3.87$} \\
&(SPN)\\
\bottomrule
\end{tabular}
\end{sc}
\end{small}
\end{center}
\vskip -0.1in
\end{table}



\begin{figure}[h!]
    \vskip -0.1in
    \centering
    \includegraphics[width=1\linewidth]{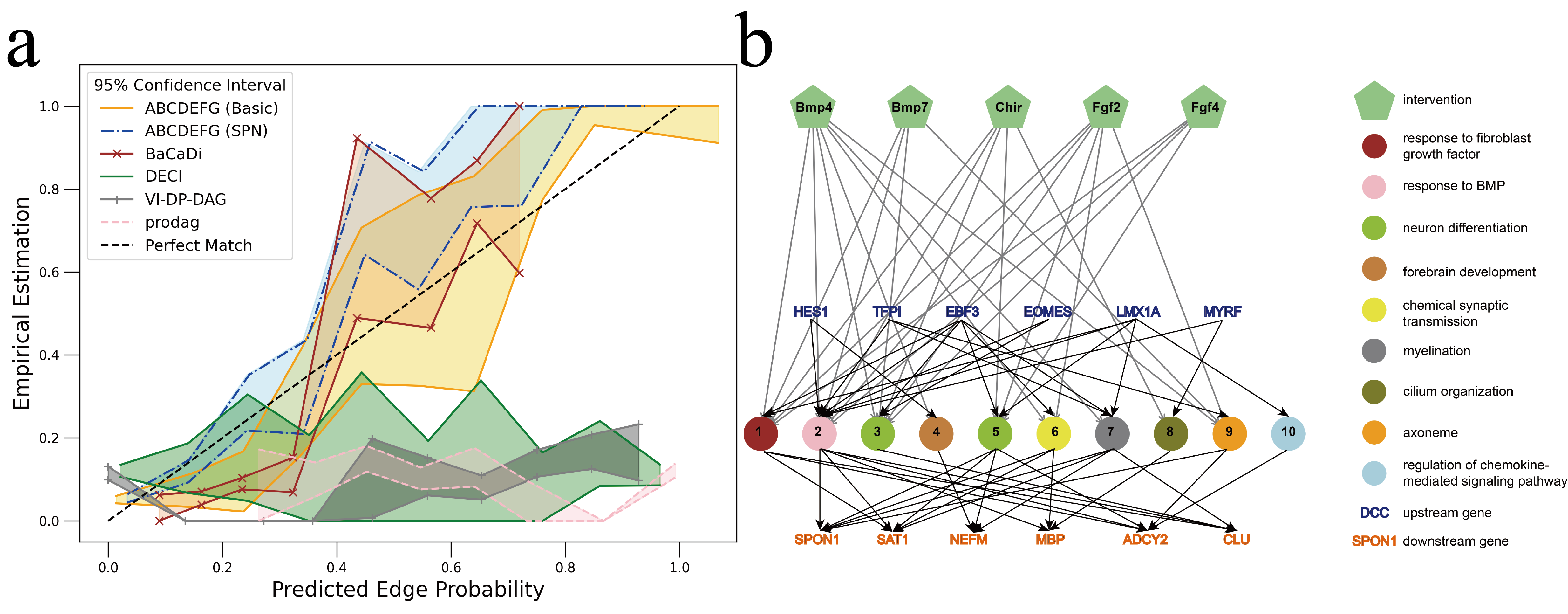}
    \caption{\textbf{Posterior calibration plot of Bayesian methods and extended factor graph inferred from growth factor screen.} \textbf{(a)} 95\% confidence intervals estimated empirically (colored regions) across the range of posterior edge probabilities for each method. The black dotted line indicates perfect calibration. \textbf{(b)} Inferred causal edges among interventions with unknown targets (growth factors; pentagons), factors (circles), and genes (text) are shown. Factor colors indicate gene ontology terms enriched in the upstream (blue) and downstream (orange) genes. Edges from interventions to factors are shown in gray arrows, and edges between genes and factors are shown in black arrows.}\label{fig:combined_results}
    \vskip -0.2in
\end{figure}
\subsection{Application to Real Cellular Perturbation Screen}
We applied our model to a large-scale single-cell perturbation screen in which cells were treated with 46 combinations of 14 growth factors \cite{AMIN20241831}. Growth factors are biomolecules that induce significant molecular changes through signaling pathways and are used to steer cells toward desired cell types in the dish. Though some downstream targets of growth factors are known, the targets are highly context-specific. The raw data contains gene expression counts for 34,469 genes in 31,475 cells. Following standard preprocessing steps for this type of data, we extracted the 1,000 most highly variable genes for causal graph inference. We used 10 factors in our model. To evaluate intervention target identification, we collected (growth factor,gene) pairs from the Gene Ontology and used these true positives to calculate recall. We cannot calculate precision because the full signaling network is unknown, so true negatives are not available. As a baseline model, we compared against random factor graphs with the same edge density as the graphs inferred by ABCDEFG. ABCDEFG achieved a recall of 0.325 (Basic) and 0.376 (SPN), significantly better than the baseline model (recall: 0.196). Second, we evaluated data reconstruction on held-out interventions. Both DCDI and ENCO failed to run on the real data. The remaining approaches DCDFG and SDCD cannot incorporate interventions with unknown targets, so we treated the data as observational when training them. We held out four intervention combinations during training, then calculated the MSE of reconstructed data on these held-out interventions. ABCDEFG achieved better MSE on the held-out samples (Basic: 0.917, SPN: 0.922) compared with DCDFG (0.957) and SDCD (1.029). Finally, we visualized the causal factor graph learned by ABCDEFG (Fig. \ref{fig:combined_results}b). 
\section{Conclusion}\label{conclusion}
ABCDEFG fills a key gap in the field by enabling scalable Bayesian causal discovery from interventional data with known or unknown intervention targets. However, we acknowledge several limitations. First, gene regulatory networks often contain cycles, violating the acyclicity assumption. Second, the f-DAG approach could poorly approximate a causal DAG when the true graph is high-rank (or when the number of factors in the f-DAG is too low). Also, our identifiability theorems do not describe the influence of sample size, though we think that our framework provides a promising foundation for future efforts to extend identifiability results into the limited data regime. ABCDEFG opens exciting new opportunities to infer gene regulatory networks and perturbation targets from large-scale cellular perturbation data.



\acks{This project was supported by NIH grant R01HG010883 to J.D.W. The authors declare that there are no competing interests.}


\newpage

\appendix


\section{Overview of Sum-Product Network} \label{supp:complexity}
Using the generative process we developed for constructing extended factor graphs, we can infer a causal DAG by optimizing a score function with respect to two binary matrices $\bY$ and $\bB$. But what is the best way to do this, given that $\bY$ and $\bB$ are discrete? One possibility is the Gumbel softmax trick~\cite{jang2017categorical}, often applied due to its simplicity. For $\bY$, we can parameterize each $Y_i$ with logits $\boldsymbol{\theta_i}$ and sample $Y_i$ using Gumbel softmax. Similarly we can treat each edge in $\bB$ as a Bernoulli random variable and sample from Gumbel softmax. However, such a naive approach treats all edges as independent and neglects possible correlation between edges. 

A more general approach is to model the joint distribution of edges in $\bB$ using a sum-product network \cite{poon2011sum}. Two naive ways to sample a binary vector $\mathbf{b}\in \{0,1\}^{d}$ are (1) sample from a single categorical distribution over all binary vectors or (2) sample each entry independently from a Bernoulli distribution. The former involves $2^d$ categories, which is impractical for large $d$, while the latter neglects dependency between any two entries and lacks expressiveness. In contrast, SPNs provide an appealing parametric model for $\bB$ due to their balance between model complexity and expressiveness. 

Let $\boldsymbol{B}=[B_1, \ldots, B_d]^T\in\{0,1\}^d$ be a random binary vector. We applied and extended the algorithm by \citet{shih2020probabilistic} to construct an SPN to model the joint distribution of $\boldsymbol{B}$. The construction of an SPN is analogous to building a neural network by sequentially adding layers. Each layer contains one type of computation nodes: (1) input node, (2) product node and (3) sum node and acts as a function of input as shown in Fig. \ref{fig:fg}a. The SPN starts with singletons $\{b_1\},\ldots,\{b_d\}$ as an initial partition. Each $b_i$ is passed to two input nodes outputting 0 and 1 respectively. Next, each product layer merges the partitions from the previous layer by creating all combinations of bit sequences for each merge. When the number of sequences from a merge exceeds a threshold, $w$, a sum layer is added to filter out sequences from the previous layer while keeping the same number of partitions. The merge filter process continues until a single partition remains. Thus, the SPN can also be interpreted as a deep mixture model whose trainable parameters are the mixture weights of all sum nodes.

The original algorithm by \citet{shih2020probabilistic} only works when $d$ is a power of two due to recursively halving the partitions, but we extended it to the general case. To do this, we divide $d$ into powers of two based on its binary representation: $d=\sum_{i=0}^{k}b_i\times 2^i$. Next, for each $b_i=1$, we build an SPN modeling joint PMF of $2^i$ bits. Finally, we apply a product and sum unit to merge the outputs from each SPN together. The number of parameters in an SPN with a maximum width of $w$ for an f-DAG with $m$ factors and $n$ nodes scales as $\bTheta\left(\frac{mnw^2}{\log w}\right) $, achieving a balance between model size and model expressiveness.

We further provide a theoretical bound on the space complexity of the SPN-FG model we used for ABCDEFG.

\textbf{Notation.}
As introduced in section \ref{subsec:spn-fg}, an SPN-FG model contains partition variables $\bY=\{Y_i:i\in[n]\}$ and connection matrix $\bB\in \{0,1\}^{n
\times m}$ parameterized by sum-product networks (SPN).  We use the following notation throughout the derivation.
\begin{enumerate}
    \item $n$: number of graph nodes.
    \item $m$: number of factors.
    \item $l$: SPN layer index
    \item $p_l$: number of partitions in the $l$-th layer of an SPN
    \item $u_l$: number of sum or product nodes in each partition in the $l$-th layer of an SPN.
    \item $w$: maximum number of bit sequences from a product node.
    \item $s$: Total number of trainable parameters of a single SPN.
    \item $S$: Total number of trainable parameters of an SPN-FG model.
\end{enumerate}
We define model complexity as the total number of trainable parameters of an SPN-FG. In our implementation, the joint PMF of either a row or a column of $\bB$ can be parameterized with a separate SPN. We consider the case of building an SPN for each row of $\bB$, i.e. each SPN models the joint distribution of connections between one node and all factors. This results in the following general formula for trainable parameters:
\begin{align}
    S = n(m+1) + ns
\end{align}
The first part $n(m+1)$ represents $n$ categorical distributions with $m+1$ categories for modeling $\bY$. The second part $ns$ represents $n$ SPNs, each having $s$ parameters and modeling a single row of $\bB$. Later, we will see that the space complexity stays the same when we choose to parameterize each column of $\bB$ with an SPN. Notice that $s$ is a function of $m$, $n$ and $w$. Next, we derive bounds of $s$.

\begin{figure}[t!]
    \centering
    \includegraphics[width=1\linewidth]{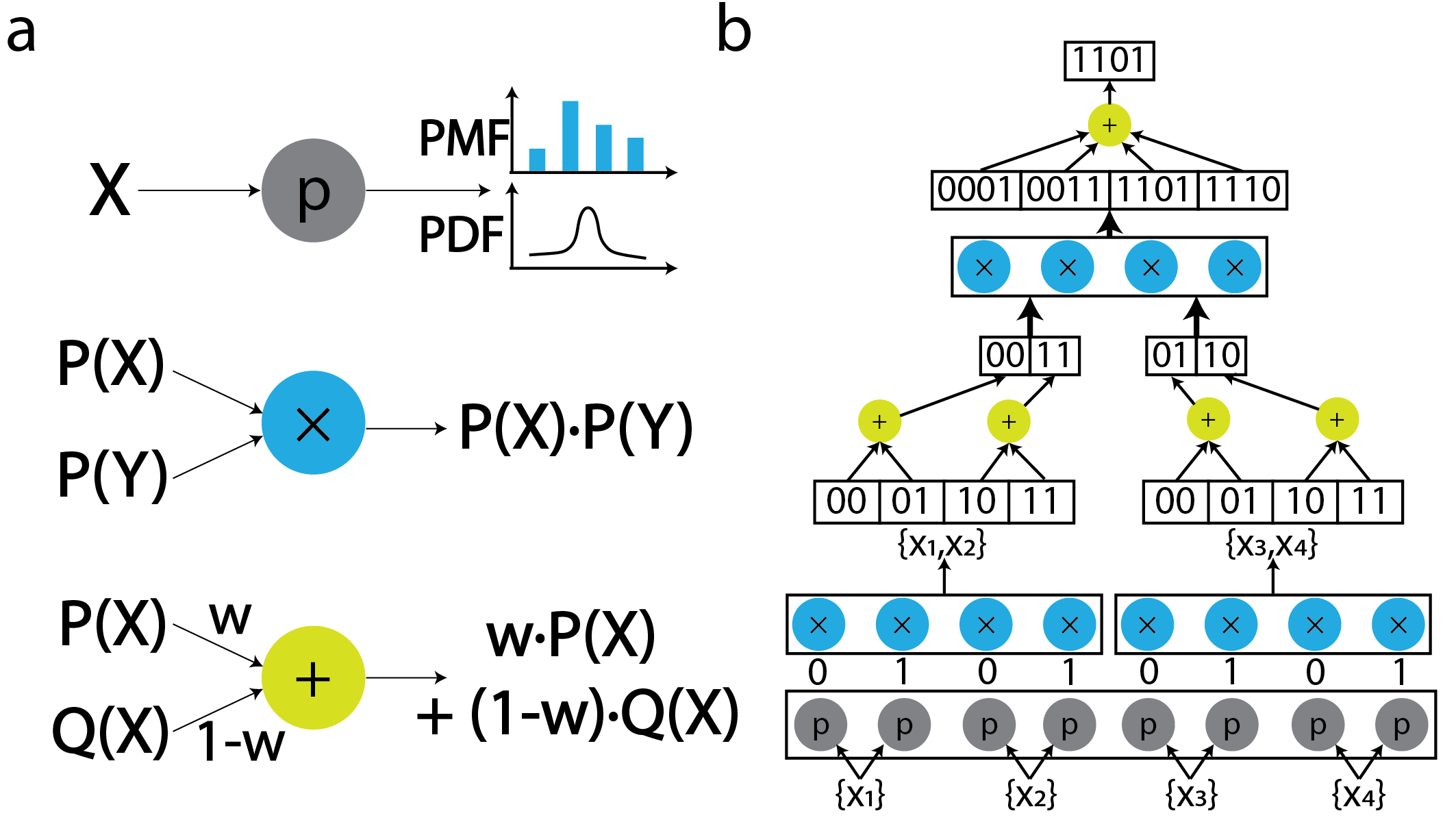}
    \caption{\textbf{Illustration of Sum-Product Network (SPN).} \textbf{(a)} Building blocks of an SPN. \textbf{Top}: An input node encodes a PMF or PDF given an input value $x$. \textbf{Middle}: A product node generates a product of input distributions as the output. \textbf{Bottom}: A sum node generates a mixture of input distributions as the output. \textbf{(b)} An Example of SPN Architecture. Assume inputs are random bits $x_1 \ldots x_4$. Input nodes generate both 0's and 1's for each bit. Next, a product layer merges $x_1, x_2$ and $x_3,x_4$ by generating all 2-bit sequences for $\{x_1, x_2\}$ and $\{x_3,x_4\}$ respectively. Then, a sum layer downsamples inputs. Finally, a product and a sum layer merge $x_1,\ldots,x_4$ together and output a 4-bit sequence.}
    \label{fig:fg}
\end{figure}

\textbf{Special Case.}
Here, we consider a special case of both $m$ and $w$ being a power of 2. Suppose $m=2^d$ and $w=2^k$. This is also the assumption in the original algorithm by \citet{shih2020probabilistic}.

We build an SPN by sequentially adding either a product or a sum layer to the network. The algorithm by Shih et al. keeps adding product layers until the Cartesian product of two partitions has a size exceeding the bound $w$. Here, we further assume $w < 2^m$ because if it's not the case, the width bound, $w$, has no effect and the SPN will be equivalent to a categorical distribution over all $2^m$ binary vectors. Once the width of an SPN exceeds $w$, we add sum and product layers alternatingly. Each sum node constraints the partition size $u_l$ to be $w$, while each product node always combines two sets of $w$ sequences into $w^2$ combinations. That is, we have
\begin{align}
    p_l &= \left\{
    \begin{array}{ll}
       \frac{1}{2}p_{l-1} &  l \leq L_0 \\
       p_{l-1} &  l > L_0, l-L_0 \; \mbox{odd (sum)}\\
       \frac{1}{2}p_{l-1} &  l > L_0, l-L_0 \; \mbox{even (product)}
    \end{array}
    \right. \label{eq:p_l_recursion}\\
    u_l &= \left\{
    \begin{array}{ll}
       u_{l-1}^2  &  l \leq L_0 \\
       w  &  l > L_0, l-L_0 \; \mbox{odd (sum)}\\
       w^2 &  l > L_0, l-L_0 \; \mbox{even (product)}
    \end{array}
    \right. \label{eq:u_l_recursion}
\end{align}
Here $L_0+1$ is the lowest index of the layer whose partition size exceeds the budget $w$, i.e. $L_0 := \max_{l}u_l \leq w $. Using Eq. \ref{eq:p_l_recursion}-\ref{eq:u_l_recursion}, we have $u_l = 2^{2^{l}}$ when $l < L_0$. This implies
\begin{align}
    L_0:= \max \{l: 2^{2^l} \leq 2^k \} \implies L_0 = \lfloor\log_2 k\rfloor.\label{eq:L0}
\end{align}

The trainable parameters of our SPN are the mixture weights of sum nodes and in each sum layer, the number of sum nodes equals the number of partitions times number of nodes for each partition. Therefore, the total number of trainable parameters of each SPN equals:
\begin{align}
    s &= p_{L_0+1}\cdot u_{L_0+1}+\sum_{l'=1}^{d-L_0-1}p_{L_0+2l'+1}\cdot u_{L_0+2l'+1}\\
      &= 2^{2^{L_0+1}}\frac{m}{2^{L_0+1}}+\sum_{l'=1}^{d-L_0-1}\frac{m}{2^{L_0+1+l'}}w^2\\
      &= 2^{2^{L_0+1}-L_0-1}m+mw^2\left(\frac{1}{2^{L_0+1}} - \frac{1}{2^{d}}\right) \label{eq:spn_size}
\end{align}
From Eq.~\ref{eq:L0}, we have
\begin{align}
    & \log_2 k - 1 < L_0 \leq \log_2 k \nonumber\\
    \iff & \frac{k}{2} < 2^{L_0} \leq k \nonumber\\
    \iff & \sqrt{w} = 2^{\frac{k}{2}} < 2^{2^{L_0}} \leq 2^k = w.
\end{align}
By plugging the upper and lower bound in the above inequality into Eq.~\ref{eq:spn_size}, we have
\begin{align}
    & \frac{mw}{2\log w} + mw^2\left(\frac{1}{2\log w} - \frac{1}{m}\right) < s < mw^2\left(\frac{2}{\log w} - \frac{1}{m}\right)\\ 
    \implies & s = \bTheta\left(\frac{mw^2}{\log w}\right)\\
    \implies & S = n(m+1)+ns = \bTheta\left(\frac{mnw^2}{\log w}\right).
\end{align}

When each SPN models a column of $\bB$ instead of a row, we have $ s=\bTheta(\frac{nw^2}{\log w}) $ and hence,
$$ S = n(m+1) + ms = \bTheta\left(\frac{mnw^2}{\log w}\right).$$

Finally, we conclude that 
$$ S = \bTheta\left(\frac{mnw^2}{\log w}\right). $$

Now we consider the alternative way of modeling each column of $\bB$ with an SPN. Then, the total number of parameters becomes
$$ S = n(m+1) + ms.$$
Following exactly the same derivation with $m$ replaced with $n$, we have each SPN $s=\bTheta(\frac{nw^2}{\log w})$ and the overall $m$ parallel SPNs have a space complexity of $S=\bTheta\left(\frac{mnw^2}{\log w}\right)$. Hence, we end up with the same space complexity.

\textbf{General Case.} The bound $s(m) = \bTheta\left(\frac{mw^2}{\log w}\right)$, and hence $S = \bTheta\left(\frac{mnw^2}{\log w}\right)$, continues to hold when $m$ (or $n$) is not a power of two, since the binary decomposition described above changes the complexity only by lower-order terms.

\section{Identifiability of Causal DAGs by ABCDEFG}\label{appx:identifiability}
In this section, we will introduce key concepts from existing literature\cite{pmlr-v80-yang18a,brouillard2020differentiable,Lopez2022largescale} and prove the identifiability of our method. Previously, Yang et al. introduced the concept of $\mathcal{I}$-Markov equivalence as an extension of Markov equivalence. Brouillard et al. proved the identifiability of $\mathcal{I}$ -Markov equivalent graphs under score maximization. Later, Lopez et al. provided a sufficient condition for a causal DAG to be unique given its corresponding f-DAG. Here, we extend the theory of causal discovery of DAGs and f-DAGs showing (1) a derivation of variational Bayes approach to causal discovery, (2) identifiability of $\mathcal{I}$-Markov equivalent causal graphs under ELBO maximization and (3) a sufficient and necessary condition for equivalence between $\mathcal{I}$-Markov equivalence of f-DAGs and $\mathcal{I}$-Markov equivalence of their half-square graphs.

\subsection{Theoretical Foundation for Bayesian Causal Discovery of Factor DAGs}\label{appx_sub:id_foundation}
We first introduce concepts about causal discovery and factor DAG as from DCDI~\citet{brouillard2020differentiable} and DCD-FG~\cite{Lopez2022largescale}. 
\begin{definition}[\citet{Lopez2022largescale}]
    Given a set of nodes, $V$, and factors, $F$, a factor directed acyclic graph (f-DAG), denoted as $(V,F,E)$, is a directed acyclic graph $(V\cup F,E)$ where edges $E\subset\{(i,j):i\in V, j\in F \mbox{ or } i\in F, j\in V\}$.
\end{definition}
An f-DAG is a DAG with two different types of vertices, nodes and factors. All edges connect two vertices of different types. Alternatively, if we represent an f-DAG using an adjacency matrix $\mathbf{A}$, we can use $\U$ and $\V$ to represent node-to-factor and factor-to-node adjacency matricies. Then we have $\mathbf{A} = \U\circ\V$ where $\circ$ denotes the matrix Boolean product. Furthermore, we can condense an f-DAG to a node-only graph as defined below.
\begin{definition}[\citet{Lopez2022largescale}]
    Given an f-DAG, $D=(V,F,E)$, its half-square node graph is defined as $D^2[V]=(V,\{(i,j):\exists f\in F, (i,f),(f,j)\in E\})$, and half-square factor graph is defined as $D^2[F]=(F,\{(f,g):\exists i\in V, (f,i), (i,g) \in D \})$.
\end{definition}
A half-square graph essentially keeps all dependency relations between nodes in the original factor graph. The factors can be interpreted as intermediate nodes on the paths between causally-related observations. We also note that the mapping from the set of f-DAGs to half-square graphs is a surjection.

Denote $par(\cdot;D)$ and $chd(\cdot;D)$ as the set of parent and child nodes in any graph $D$. 
\begin{definition}
    Let $G=(V, E)$ be any graph, $\forall f\in V$, the set of unique parents and children of $f$ are defined as $P_{f}(G):=\{i: i\in par(f;G), chd(i;G)=\{f\}\}$  and $C_{f}(G):=\{j: j\in chd(f;G), par(j;G)=\{f\}\}$.
\end{definition}
With the above definition, we define a subset of f-DAGs:

Given a set of causally related random variables $\X=\{X_1,\ldots,X_n\}$ with a causal graph $G$. A fundamental assumption of a causal DAG underlying $\X$ is the Markov property, which leads to a factorization of the joint distribution. Here, we denote $\boldsymbol{\pi_i}$ as the set of all parents of $i$ in $G$.
\begin{definition}[\citet{brouillard2020differentiable}]
    Let $G=(V,E)$ be a causal DAG with $n$ nodes and $\calI^*=\{I_k:k\in[l]\}$ be a set of interventions. We define $\mathcal{M}_{\calI^*}(G)$ as the set of joint distributions factorized according to the Markov property, i.e. $ \truemarginal:=\{\{\pk:k\in[n^{\calIstar}]\}: \pk(\X)=\prod_{i=1}^{n}\pk(X_i|\X_{\boldsymbol{\pi_i}})\}$.
\end{definition}
By convention, $I_1=\emptyset$ represents a pure observational setting.

Based on the definition above,  \citet{brouillard2020differentiable} defined a type of equivalence relation called $\calI$-Markov equivalence relation to describe DAG equivalence under interventions.
\begin{definition}[$\calI$-Markov Equivalence~\cite{brouillard2020differentiable}] \label{def:imarkov_eq}
    Two DAGs $G_1$ and $G_2$ are $\calI$-Markov equivalence if and only if $\calM_{\calI}(G_1)=\calM_{\calI}(G_2)$. We denote by $\calI$-MEC$(G)$ as the set of all DAGs which are $\calI$-Markov equivalent to $G$.
\end{definition}
In the rest of section B, we use the notation $\simeqint$ to denote $\calI$-Markov equivalence relation.

Since we consider the set of f-DAGs, the causal relations between $i$ and $j$ are passed through latent factors. Denote $\boldsymbol{\pi_i^D}$ as the set of parents of a vertex $i$(node or factor) in the graph $D$. Next, we use a continuous random variable $\Z=\{Z_1,\ldots,Z_m\}$ to represent the factors. Then, we have a class of joint distributions of $\X$ and $\Z$ produced by an f-DAG.
\begin{definition}[Family of Distributions associated with an f-DAG]
    Let $D=(V,F,E)$ be an f-DAG with $n$ nodes and $m$ factors. Then, $\calM_{\calIstar}(D)$ is defined as the set of probabilistic models with the following form: 
    \begin{align}
        \calM_{\calIstar}(D) = \left\{\{\pk(\X,\Z):k\in[n^{\calIstar}]\}: \pk(\X,\Z)=\prod_{i=1}^{n}\pk(X_i|\Zpar)\prod_{j=1}^{m}\pk(Z_j|\Xpar)\right\}
    \end{align}
    where $\pk(X_i|\boldsymbol{X_{\pi_i^D}}) \neq p^{(1)}(X_i|\boldsymbol{X_{\pi_i^D}})$ if and only if $i\in I_k$ and $\pk(Z_j|\Xpar) \neq p^{(1)}(Z_j|\Xpar) $ if and only if $j\in I_k$.
\end{definition}
The above definition assumes knowledge of the intervention targets. When interventions are unknown, we are able to extend f-DAGs in a similar way to the $\mathcal{I}$-DAG introduced by Yang et al.\cite{pmlr-v80-yang18a}. We first mention the concept of $\mathcal{I}$-DAG and then extend it to f-DAGs.
\begin{definition}[\citet{pmlr-v80-yang18a}]
    Let $G=(V,E)$ be a DAG and $\mathcal{I}=\{I_1,\ldots,I_{n^{\calI}}\}$ be a set of interventions with $I_k\subseteq V,\forall k$. An interventional DAG ($\mathcal{I}$-DAG) is defined as an augmented graph
    $$ \Gint=(V\cup \Xi, E\cup E^{\calI}), $$
    where $\Xi := \{\xi_k:k\in[n^{\calI}]\}$ is a set of intervention nodes representing $I_1,\ldots,I_k$ and $E^{\calI}\subseteq\{(\xi_k, i): i\in I_k, k\in[n^{\calI}]\}$ is a set of edges from interventions to targets.
\end{definition}
\begin{definition}[Extended f-DAG]
    Let $D=(V, F,E)$ be an f-DAG and $\mathcal{I}=\{I_1,\ldots,I_{n^{\calI}}\}$ be a set of interventions. Let $\Xi=\{\xi_k, k\in[n^{\calI}]\}$ be $n^{\calI}$ nodes corresponding to the $l$ interventions. An extended f-DAG is defined as an f-DAG $\Dint=(V\cup \Xi, F, E\cup E^{I})$ where $E^{\mathcal{I}}\subseteq\{(\xi_k, f): f\in F\}$, i.e. set of edges from intervention nodes to factors.
\end{definition}
An extended f-DAG is obtained by adding intervention nodes to an f-DAG. Here, we also have low-rank assumption that interventions causally affects downstream nodes via a small number of factors. Put in a matrix form, the adjacency matrix of an extended f-DAG has a low-rank Boolean matrix factorization as
$$ \mathbf{A}^{\boldsymbol{\calI}} = 
\begin{bmatrix}
    \mathbf{U} \\
    \mathbf{W}
\end{bmatrix} \circ
\begin{bmatrix}
    \mathbf{V} \; \mathbf{0}_{\boldsymbol{m} \times \boldsymbol{n^{\calI}}}    
\end{bmatrix},$$
where $\mathbf{W}\in \mathbb{R}^{n^{\calI}\times m}$ is an adjacency matrix representing edges from intervention nodes to factors.

Given the definition of $\truefdagmarginal$ and $\calI$-Markov equivalence, we can further define $\calI$-Markov equivalence relation between f-DAGs.
\begin{definition}[$\mathcal{I}$-Markov Equivalence Class of f-DAGs]\label{def:imarkov_eq_fdag}
    Given a set of interventions, $\calI$, two f-DAGs $D_1$ and $D_2$ are $\mathcal{I}$-Markov equivalent if $\calM_{\calI}(D_1)=\calM_{\calI}(D_2)$.
\end{definition}
The concept of $\calM_{\calI}(D)$ and $\calI$-Markov equivalence for f-DAGs are just the same as those for DAGs except for classifying vertices into nodes and factors.

The following theorem regarding the concept of $\mathcal{I}$-DAG connects statistical independence to graph structures.
\begin{theorem}[\citet{pmlr-v80-yang18a}]\label{thm:yang}
    Two DAGs $G_1$ and $G_2$ belong to the same $\mathcal{I}$-Markov Equivalence Class ($\mathcal{I}$-MEC) if and only if their $\mathcal{I}$-DAGs have the same skeleton and v-structures.
\end{theorem}
Since f-DAGs are one type of DAG, we easily obtain the following corollary.
\begin{corollary}\label{cor:yang}
    Two f-DAGs $D_1$ and $D_2$ belong to the same $\mathcal{I}$-MEC if and only if their extended f-DAGs have the same skeleton and v-structures.
\end{corollary}
\begin{proof}
    Suppose $D_1$ and $D_2$ have $n$ nodes and $m$ factors. Let $G_1$ and $G_2$ be two DAGs obtained by removing the labeling of node or factor in $D_1$ and $D_2$, i.e. we treat all nodes and factors as simply nodes in $G_1$ and $G_2$. We still keep the bijection between vertices and random variables $\X=\{X_i:i\in[n]\}$ and $\Z=\{Z_j:j\in[m]\}$. We have
    \begin{align*}
        & D_1\in \mathcal{I}\mbox{-MEC}(D_2) \\
        \iff & \mathcal{M}_\mathcal{I}(D_1) = \mathcal{M}_\mathcal{I}(D_2) \\
        \iff & \mathcal{M}_\mathcal{I}(G_1) = \mathcal{M}_\mathcal{I}(G_2)  \\
        \iff & G_1\in \mathcal{I}\mbox{-MEC}(G_2) \\ \iff & G_1^\calI \text{ and } G_2^\calI \text{ have the same skeleton and v-structures} \\
        \iff & D_1^\calI \text{ and } D_2^\calI \text{ have the same skeleton and v-structures}
    \end{align*}
    The second line is by definition \ref{def:imarkov_eq_fdag}. The third line implication is by the fact $D^\calI$ and $G^\calI$ have exactly the same structure. The fourth line is by definition \ref{def:imarkov_eq}. The fifth line is by Theorem \ref{thm:yang}. The last line is again by the identical structure between $D^\calI$ and $G^\calI$.
\end{proof}

In reality, we can use a single encoder function to get $Z_j \sim p(f_{enc}(\Uj\odot\X;\bTheta))$ and $X_i\sim p(f_{dec}(\Vi\odot\X;\bPhi))$ to represent the conditional distribution $\pk(X_i|\Zpar)$ and $\pk(Z_j|\Xpar)$. Thus, we define a second set of joint distributions representing our model capacity.
\begin{definition}[Family of Parametric Distributions associated with an f-DAG]
    Let $D=(V,F,E)$ be an f-DAG with $n$ nodes and $m$ factors. Consider two parametric functions $f_{enc}:\mathbb{R}^n\rightarrow\mathbb{R}^m$, parameterized by $\bTheta\in\Omega(\bTheta)$ and $f_{dec}:\mathbb{R}^m\rightarrow\mathbb{R}^n$, parameterized by $\bPhi\in\Omega(\bPhi)$. In addition, let $\U$ and $\V$ be node-to-factor and factor-to-node matrices of an f-DAG $D$. Then, $\calF_{\calIstar}(D)$ is defined as the set of probabilistic models with the following form: 
    \begin{align}
        \calF_{\calIstar}(D) = \left\{\{\fk(\X,\Z):k\in[n^{\calIstar}]\}: \fk(\X,\Z)=\prod_{i=1}^{n}\fk(X_i|\Zpar)\prod_{j=1}^{m}\fk(Z_j|\Xpar)\right\},
    \end{align}
    where $\fk(Z_j|\Xpar)=p(f_{enc}(\Uj\odot\X))$, $\fk(X_i|\Zpar)=p(f_{enc}(\Vi\odot\Z))$, $\fk(X_i|\Zpar) \neq f^{(1)}(X_i|\Zpar)$ if and only if $i\in I_k$ and $\fk(Z_j|\Xpar) \neq f^{(1)}(Z_j|\Xpar) $ if and only if $j\in I_k$.

\end{definition}

\subsection{Derivation of Bayesian Framework for Differentiable Causal Discovery}\label{appx_sub:id_bayesian}
\label{supp:bayes}
We present a Bayesian framework for differentiable causal discovery and show that it reduces to score maximization under a uniform prior over the space of DAGs.

Consider a set of causally related random variables $\X=\{X_i:i\in[n]\}$ and a random intervention set $I^*\subseteq [n]$. First, we assume the observations are generated from a single causal graph $G^*$ via a generative model $p(\X|G^*,I^*)$. We assume each intervention either removes edges towards targets (hard) or keeps the same graph structure (soft). Thus, the generative model becomes $p(\X|G^*,I^*)$ under different interventions. When $I^*$ is known, we can obtain a MAP estimate of $G$:
\begin{align}
    \hat{G} = \arg\max_{G\in\calG}p(G|\X,I^*).
\end{align}
In order to convert this optimization problem to a differentiable one, we consider a variational distribution $q(G)$ and optimize a KL divergence instead:
\begin{align}
    \hat{G} = \arg\max_{G}q^*(G) \mbox{ where } q^*(G) = \arg\min_{q(G)}KL(q(G)||p(G|\X,I^*)).
\end{align}
Because we have control over $q(G)$, finding its argmax will be easy. Directly optimizing $KL(q(G)||p(G|\X))$ suffers from the intractability problem since $p(G|\X,I^*)=\frac{p(\X|G,I^*)p(G|I^*)}{\sum_{G'}p(\X|G',I^*)p(G'|I^*)}$ and the space of DAGs is super-exponential in the number of nodes. Thus, we can derive an alternative objective in the following form:
\begin{align}
    &KL(q(G)||p(G|\X,I^*))\nonumber\\
    &= \mathbb{E}_{p(\X,I^*)}\left[\mathbb{E}_{q(G)}\left[\log\frac{q(G)}{p(G|\X,I^*)}\right]\right]\nonumber\\
    &= \mathbb{E}_{p(\X,I^*)}\left[\mathbb{E}_{q(G)}\left[\log\frac{q(G)p(\X|I^*)}{p(\X|G,I^*)p(G|I^*)}\right]\right]\nonumber\\
    &= \mathbb{E}_{p(\X,I^*)}\left[\log p(\X|I^*) - \mathbb{E}_{q(G)}\left[\log p(\X|G,I^*)\right] + KL(q(G)||p(G|I^*)) \right]\nonumber\\
    \implies &\min_{q(G)}KL(q(G)\;||\;p(G|\X,I^*)) \nonumber\\
    & = \max_{q(G)}\mathbb{E}_{p(\X,I^*)}\left[\mathbb{E}_{q(G)}\left[\log p(\X|G,I^*)\right] - KL(q(G)||p(G|I^*))\right]\nonumber\\
    &= \max_{q(G)}ELBO(G)
\end{align}
In reality, $p(\X,I^*)$ is replaced with an empirical distribution from any dataset. For $I^*$, we can conduct additional experiments by perturbing some nodes $I_k$ in the $k$-th experiment. For the empirical data distribution, we assume the data samples are generated from $p(\X|G^*)$ instead of $p(\X)$. The data samples are not drawn from the marginal over $\X$ because we assume a single causal graph $G^*$ underlying the data generative process. We use parametric $\fk(X_i|\boldsymbol{X_{\pi_i^G}};\bPhi)$ for distributional fitting and $q(G;\bLambda)$ for graph fitting. Here, $\pi_i$ is the set of all parents nodes of node i in a graph $G$. In addition, we need to add an L1 regularization on $G$ to account for the sparsity constraint. Now the optimization problem becomes:
\begin{align}
    \sup_{\bLambda, \bPhi} \sum_{k=1}^{n^{\calI}}\mathbb{E}_{\pk(\X|G^*)}\left[\mathbb{E}_{\qkG}\left[\log \fk(\X|G;\bPhi)\right] - KL(\qkG||p^{(k)}(G))\right] - \lambda \mathbb{E}_{\qkG}[|G|]
\end{align}
The objective function is similar to the one proposed in the VAE paper~\cite{KingmaW13} except that we have a latent space of DAGs instead of a low-dimensional latent embedding. In addition, we assume interventions change neither the prior graph distribution nor our variational posterior. The objective can be extended to that of a $\beta$-VAE:
\begin{align}
    &\sup_{\bPhi, \bLambda} \sum_{k=1}^{n^{\calI}}\mathbb{E}_{\pk(\X|G^*)}\left[\mathbb{E}_{\qG}\left[\log \fk(\X|G;\bPhi)\right] \right] - \beta KL(\qG||p(G))- \lambda \mathbb{E}_{\qG}[|G|]\nonumber\\
    =&\sup_{\bPhi, \bLambda} \mathbb{E}_{\qG}\left[\sum_{k=1}^{n^{\calI}}\mathbb{E}_{\pk(\X|G^*)}\left[\log \fk(\X|G;\bPhi)\right] - \lambda |G| \right] - \beta KL(\qG||p(G))
\end{align}
Notice that the score function is under the expectation of $\qG$. If we set $\beta=0$ and $\qG=\delta(G)$, the Dirac delta function, the optimization problem becomes exactly the same as a score maximization problem as presented in previous score-based methods. The constraint on $\qG$ ensures that $\qG$ does not deviate from the prior arbitrarily. Next, we will prove the identifiability of this Bayesian framework.

\begin{theorem}[\citet{brouillard2020differentiable}]\label{thm:supp:dcdi}
    Let $\X=\{X_1,\ldots,X_n\}$ be a set of causally related random variables with a causal DAG $G^*=(V,E)$ and $\calIstar=\{I_k:k\in[n^{\calIstar}]\}$ be a set of interventions with $I_1=\emptyset$. Assume the following:
    \begin{enumerate}
        \item The set of distributions from our parametric models contains the ground truth interventional distributions: $\{\pk(\X):k\in[n^{\calIstar}]\}\in\calF_{\calIstar}(G^*)$ where $\calF_{\calIstar}(G^*)=\{\{\fk(\X|G^*;\bPhi)\}: \bPhi\in\Omega(\bPhi)\}$.
        \item Denote $\ind_{G^*}$ as the d-separation relation in $G^*$. $\calI$-faithfulness contains the following two conditions. 
        \begin{enumerate}
            \item For any disjoint set $A,B,C\subset V$, $\X_{\boldsymbol{A}} \ind \X_{\boldsymbol{B}} | \X_{\boldsymbol{C}} \implies A \ind_{G^*} B | C$
            \item For any disjoint sets $A,C\subset V$ and $k\in[n^{\calIstar}]$, $\pk(\X_{\boldsymbol{A}}|\X_{\boldsymbol{C}})=p^{(1)}(\X_{\boldsymbol{A}}|\X_{\boldsymbol{C}}) \implies A \ind_{{G^*}^{\calI^*}} \xi_k | C$
        \end{enumerate}
        \item $\forall G, I, \bPhi, \fk(\X|G,I;\bPhi)>0$.
        \item $\forall k\in[n^{\calIstar}]$, $\left|\mathbb{E}_{\pk(\X)}\left[\log \pk(\X)\right]\right| < +\infty$.
    \end{enumerate}
    Define the score function as 
    $$\scoreIstar = \sup_{\bPhi}\sum_{k=1}^{n^{\calIstar}}\mathbb{E}_{\pk(\X)}\left[\log \fk(\X|G;\bPhi)\right]-\lambda |G|$$
    Then, with a small enough $\lambda>0$, we have $\scoreIGstar > \scoreIstar$.
\end{theorem}
The previous theorem claims optimality of the score function when the causal DAG is treated as a deterministic object. Next, we give a probabilistic view of this optimality. First, we define the Bayesian score function as follows.
\begin{definition}[Bayesian Score Function]
     Let $\X=\{X_1,\ldots,X_n\}$ be a set of causally related random variables with a causal DAG $G^*$ and $\calIstar=\{I_k:k\in[n^{\calI}]\}$ be a set of interventions with $I_1=\emptyset$. Let $p(G)$ be a prior over DAGs and $\qG$ be a variational distribution. The Bayesian score function, $\calL(\qG)$ is defined as
     $$ \calL(\qG)=\mathbb{E}_{\qG}\left[\scoreIstar\right] - \beta KL(\qG||p(G))$$
     where $\scoreIstar$ is the score function defined in Theorem \ref{thm:supp:dcdi}.
\end{definition}

\begin{theorem}[Identifiability via ELBO maximization]\label{thm:supp:elbo_identify}
    Let $\X=\{X_1,\ldots,X_n\}$ be a set of causally related random variables with a causal DAG $G^*$ and $\calIstar=\{I_k:k\in[n^{\calI}]\}$ be a set of interventions with $I_1=\emptyset$. Let $\calG$ be a subset of all causal DAGs and $q^*(G)$ be an optimal graph distribution from the optimization problem:
    $$ \sup_{\qG:supp(q)\subseteq\calG}\calL(\qG), $$
    where
    \begin{align*}
        & \calL(\qG) = \mathbb{E}_{\qG}\left[\scoreIstar\right] - \beta KL(\qG||p(G)),\\
        & \scoreIstar = \sup_{\bPhi}\sum_{k=1}^{n^{\calI}}\mathbb{E}_{\pk(\X)}\left[\log \fk(\X|G;\bPhi)\right]-\lambda |G|.
    \end{align*}
    If $G^*\in\calG$, then, under the same assumptions as those in Theorem \ref{thm:supp:dcdi}, for small enough $\beta > 0$ and small enough $\lambda > 0$, $\hat{G}=\arg\max_{G}q^*(G)$ is $\calIstar$-Markov equivalent to $G^*$.
\end{theorem}
\begin{proof}
    We prove this theorem by contradiction. Suppose $\exists \hat{G}=\arg\max_{G}q^*(G)$ such that $\hat{G} \not\simeqintstar G^*$. 
    
    Consider another PMF $q'(G)$ which has the same support and same mass as $q^*(G)$ except for $q'(G^*) - q^*(G^*) = \epsilon > 0$ and $q'(\hat{G}) - q^*(\hat{G}) = -\epsilon < 0$. Because $q^*(\hat{G}) > 0$, such $q'$ and $\epsilon$ exist. By the definition of $q^*$, $\calL(q^*) \geq \calL(q')$. Then, we have
    \begin{align}
        &\calL(q')-\calL(q^*) \nonumber\\
        &= \left[\mathbb{E}_{q'(G)}\left[\scoreIstar \right] - \beta KL(q'(G)||p(G))\right] - \left[\mathbb{E}_{q^*(G)}\left[\scoreIstar \right] - \beta KL(q^*(G)||p(G))\right]\nonumber\\
        &= \sum_{G\in\calG}(q'(G) - q^*(G))\scoreIstar + \beta \left[KL(q^*(G)||p(G)) - KL(q'(G)||p(G))\right] \nonumber\\
        &= \epsilon \left(\scoreIGstar - S_{\calIstar}(\hat{G})\right) + \beta \left[KL(q^*(G)||p(G)) - KL(q'(G)||p(G))\right].
    \end{align}
    By Theorem \ref{thm:supp:dcdi}, $\exists \lambda > 0$ such that $\scoreIGstar > \scoreIstar, \forall G\notsimeqintstar G^*$. Therefore, $\scoreIGstar - S_{\calIstar}(\hat{G}) = \Delta > 0$. If $\sum_{k=1}^{n^{\calIstar}}\left[KL(q^*(G)||p(G)) - KL(q'(G)||p(G))\right] \geq 0$, we already have $\calL(q') > \calL(q^*)$. Otherwise, we can pick
    $$ 0 < \beta < \frac{\epsilon\Delta}{KL(q'(G)||p(G) - KL(q^*(G)||p(G)))}$$
    and $\calL(q') > \calL(q^*)$. Both cases contradict the fact that $\calL(q^*) \geq \calL(q')$. Therefore, we conclude that $G^*$ must be a argmax of $q$.
\end{proof}
Notice that we add a constraint on the support of $\qG$ to account for cases when we have prior knowledge about the DAG and only need to search over a subset. As discussed below, this applies when the true causal DAG is a half-square graph of an f-DAG. If we set $\calG$ to the set of all DAGs, the constraint will be removed.

ABCDEFG aims at optimizing $KL(q(D^2[V])||p(G|\X))$ with respect to a distribution on f-DAGs instead of DAGs. As long as the adjacency matrix of the true causal DAG can be factorized as a Boolean product of a node-to-factor and factor-to-node matrices, optimization over f-DAGs guarantees identifiability of the true causal DAG, as a half-square graph of an optimal f-DAG.

\subsection{Identifiability of the True f-DAG}\label{appx:sub_iden_fdag}
So far, our theory has covered the major interest of causal discovery. However, if we also assume there is a ground truth for the f-DAG, it is not obvious that identifying the causal DAG is equivalent to identifying an f-DAG. Identifying causal connection between nodes and latent factors may have important implications in applications such as computational biology, where genes often function together through some pathways. Thus, we introduce additional theoretical results about f-DAG identifiability in this section.


The first question we would like to ask is: Is $\calI$-Markov equivalence between f-DAGs equivalent to $\calI$-Markov equivalence between their half-square graphs? To answer this question, we first need to introduce some notations and new concepts.
\begin{enumerate}
    \item We denote $par(\cdot;D)$ and $chd(\cdot;D)$ as the set of parents and children of a vertex in an f-DAG, $D$
    \item We denote $\simeq$ as the Markov equivalence and $\simeq_\mathcal{I}$ as the $\mathcal{I}$-Markov equivalence relation.
    \item For any factor $f$ in an f-DAG $G$, we define the set $P_f(G)=\{i: chd(i)=\{f\}\}$ as the set of parents with  $f$ as the unique child in $G$.
    \item For any factor $f$ in an f-DAG $G$, we define the set $C_f(G)=\{i: par(i)=\{f\}\}$ as the set of children with only $f$ as the unique parent in $G$.
\end{enumerate}

In fact, not every f-DAG can be identified up to an $\calI$-Markov equivalence class. We consider a subset of f-DAGs defined as follows.
\begin{definition}\label{appx:def_dm}
    Let $D=(V,F,E)$ be an f-DAG where $|F|=m$. Let $\calI$ be a set of interventions. $\calD_m$ is defined as the set of f-DAGs with $m$ factors and the following properties:
    \begin{enumerate}
        \item $\forall f\in F$, $P_f(D)\neq\emptyset$ and $C_f(D)\neq\emptyset$.
        \item $\forall f_1, f_2 \in F, f_1\neq f_2, f_1\rightarrow f_2 \in D^2[F]$, $|P_{f_1}(D)|>1$ (inclusively) or $|P_{f_2}(D)|>1$.
        \item $\forall f\in F$, if $|par(f;D)| > 1$, there is at most one factor $g\in par(f;D^2[F])$ such that $|par(g;D)| = 1$.
    \end{enumerate}
\end{definition}
Intuitively, the three additional conditions for f-DAGs mean
\begin{enumerate}
    \item Any factor should have a unique parent and unique child that distinguish it from other factors.
    \item There cannot be adjacent ``chain" or ``tree" structures in the f-DAG.
    \item There should be enough v-structures in the f-DAG.
\end{enumerate}
We give three counterexamples (Fig. \ref{appx:fig_violate_1}-\ref{appx:fig_violate_3}) when each of the three conditions is violated. In these cases, the f-DAGs are no longer Markov equivalent but their half-square graphs are Markov equivalent. Thus, we note the three conditions are all necessary for proving the identifiability of f-DAGs.
\begin{figure}[h!]
    \centering
    \includegraphics[width=\linewidth]{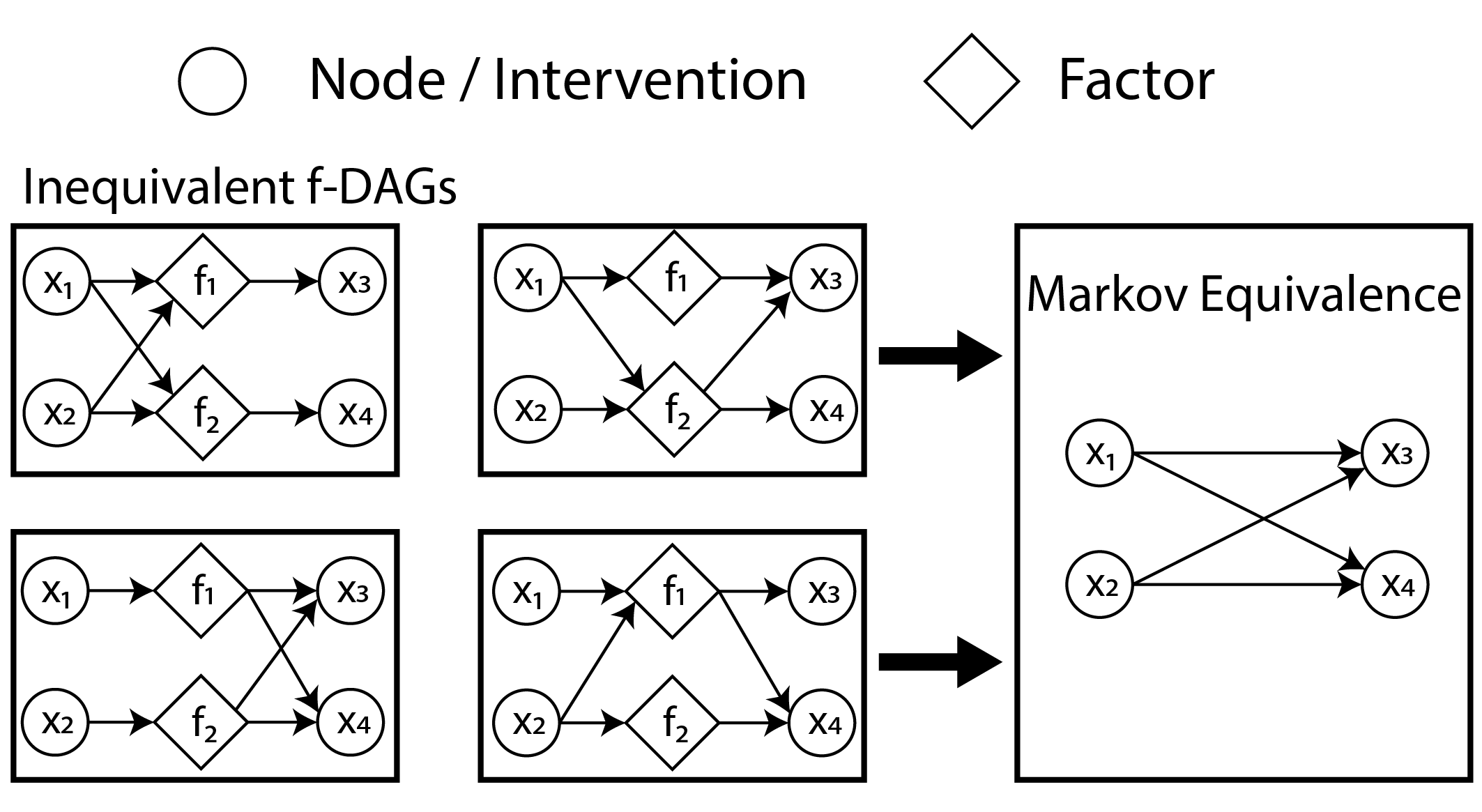}
    \caption{\textbf{Markov equivalence of DAGs does not imply Markov equivalence of f-DAGs when condition 1 in Def. \ref{appx:def_dm} is violated.}}
    \label{appx:fig_violate_1}
\end{figure}
\begin{figure}[h!]
    \centering
    \includegraphics[width=\linewidth]{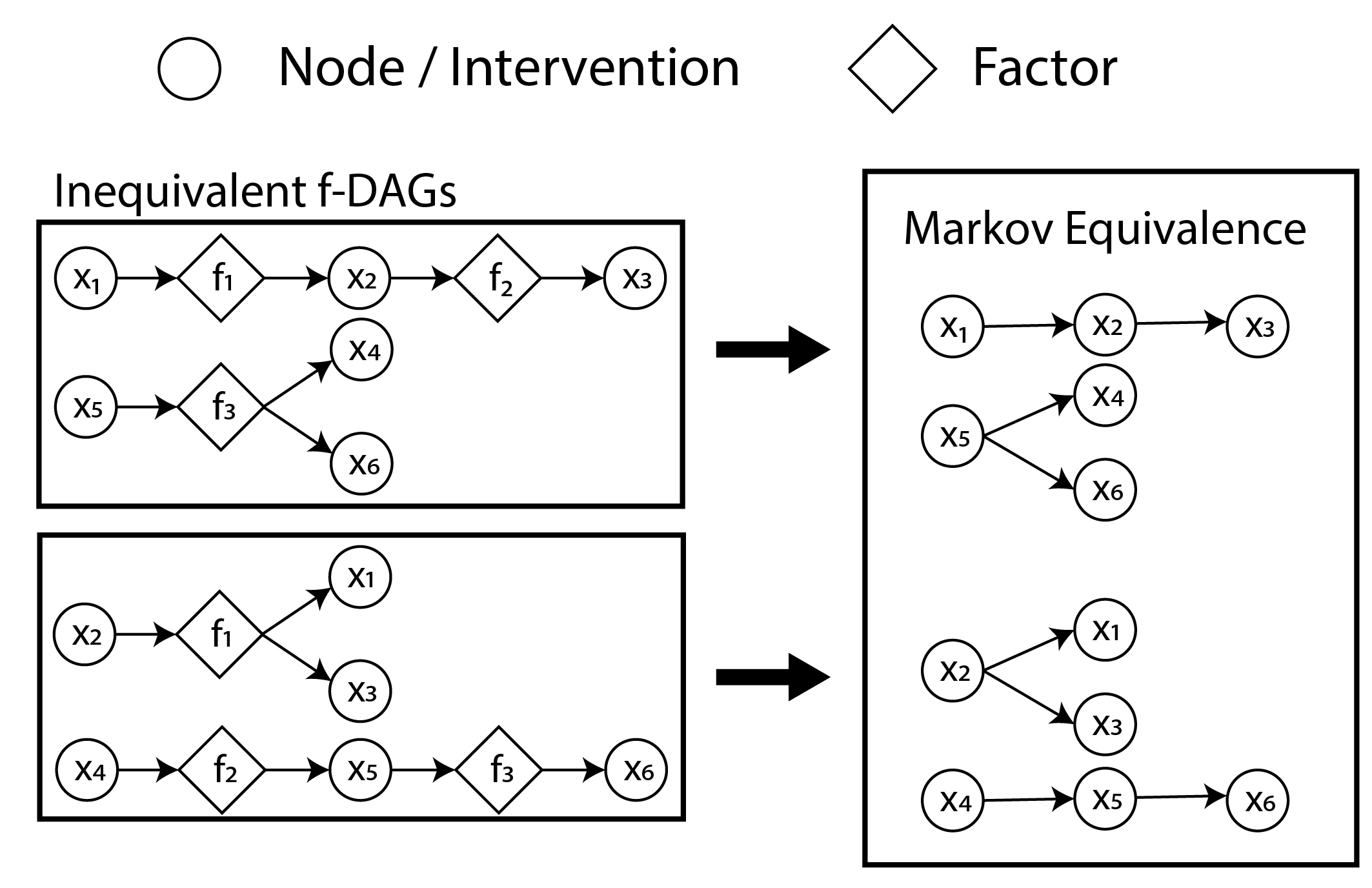}
    \caption{\textbf{Markov equivalence of DAGs does not imply Markov equivalence of f-DAGs when condition 2 in Def. \ref{appx:def_dm} is violated.}}
    \label{appx:fig_violate_2}
\end{figure}
\begin{figure}[h!]
    \centering
    \includegraphics[width=\linewidth]{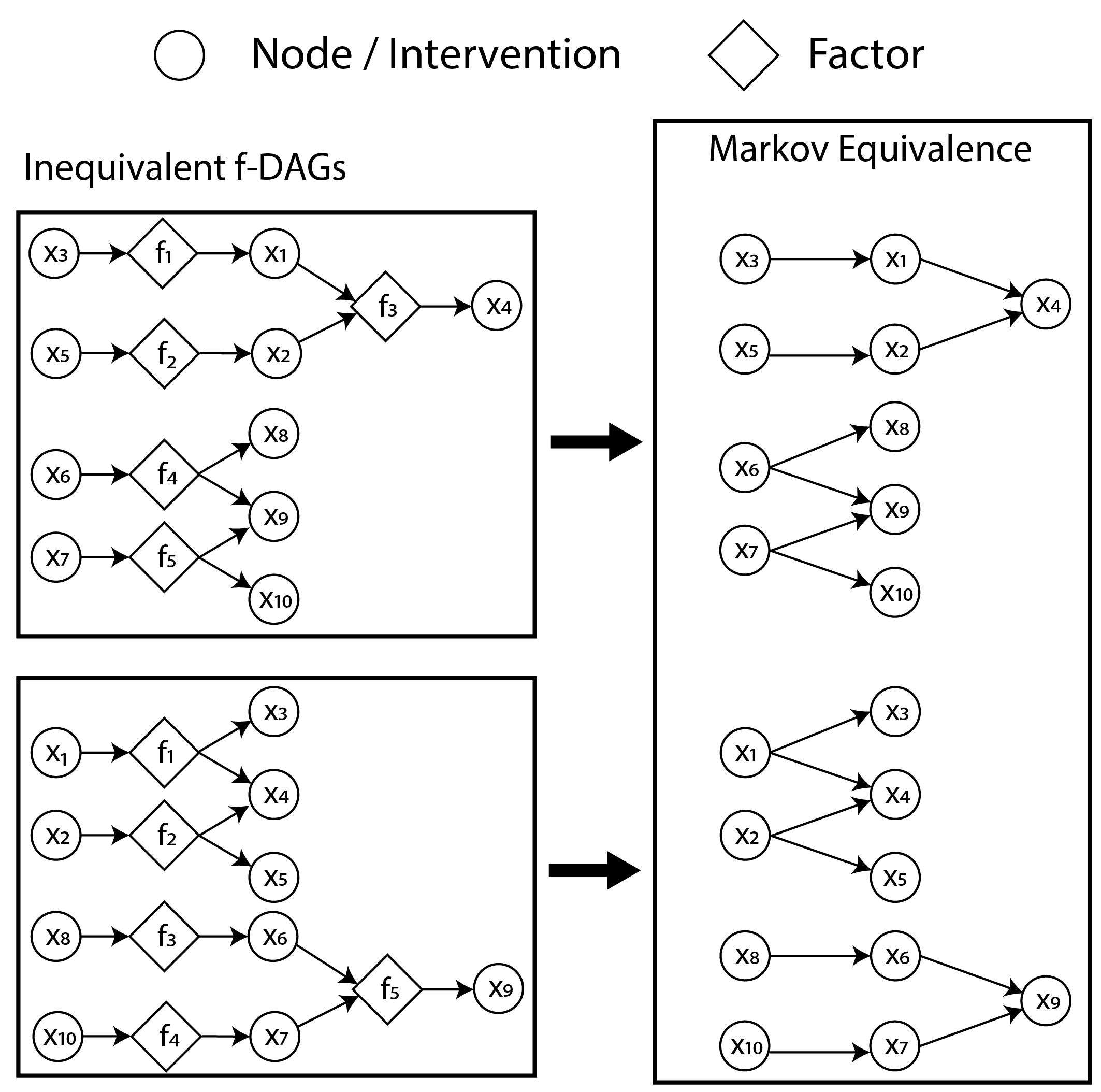}
    \caption{\textbf{Markov equivalence of DAGs does not imply Markov equivalence of f-DAGs when condition 3 in Def. \ref{appx:def_dm} is violated.}}
    \label{appx:fig_violate_3}
\end{figure}

\begin{definition}
    Let $G=(V,E)$ be a DAG which can be represented as a half-square graph of an f-DAG with $m$ factors. $\calG_m$ is defined as the set of all DAGs having an identifiable f-DAG representation:
    $$ \calG_m:=\{G:\exists D\in\calD_m, G=D^2[V]\}$$
\end{definition}

Now we present the following lemma.
\begin{lemma}\label{supp:lemma:1}
    Let $D_1=(V,F,E_1)$ and $D_2=(V,F,E_2)$ be two f-DAGs on the same set of nodes, $V$, and factors, $F$, and $\mathcal{I}$ be a set of interventions. Let $\Xi=\{\xi_k:k\in[n^{\calI}]\}$ be the intervention nodes. In addition, suppose $D_1, D_2\in \calD_m$ defined as in Def. \ref{appx:def_dm}. Then, under a permutation of factors, we have $D_1\simeq_\mathcal{I}D_2\iff D_1^2[V]\simeq_\mathcal{I}D_2^2[V]$.
\end{lemma}

\begin{proof} (
    By Theorem \ref{thm:yang}, we convert the proof of $\calI$-Markov equivalence to proof of equal graph structure.
    
    Let $D_1^\mathcal{I}$ and $D_2^\mathcal{I}$ be the corresponding extended $f$-DAGs. 

    First, we prove the forward direction by contradiction. Suppose $D_1\simeq_\mathcal{I}D_2$ but $\Doneintsq\notsimeqint\Dtwointsq$. Then, either there is an edge $(i,j)$ in $\Doneintsq$, $(i,j)$ not in $\Dtwointsq$ or a v-structure $i\rightarrow k \leftarrow j$ mismatch between $\Doneintsq$ and $\Dtwointsq$. In the former case, $\exists f\in F$ such that $i\rightarrow f \rightarrow j$ in $D_1$, but no such factor in $D_2$. This leads to a mismatch in skeleton between $D_1$ and $D_2$ and results in contrdiction. 
    
    In the latter case, without loss of generality, we assume $i\rightarrow k \leftarrow j$ in $\Doneintsq$ but $i\rightarrow k \leftarrow j$ not in $\Dtwointsq$. In addition, we can also assume the skeletons of $\Doneintsq$ and $\Dtwointsq$ match. Otherwise, we will be in the former case again. Let $f_1$ and $f_2$ be the factors \ such that $i\rightarrow f_1 \rightarrow k$ and $j\rightarrow f_2 \rightarrow k$ in $\Dint_1$.
    
    If $f_1 = f_2$, we have a v-structure $i \rightarrow f_1 \leftarrow j$ in $\Dint_1$. Because $D_1\simeq_{\calI} D_2$ and $i\rightarrow k \leftarrow j$ not in $\Dtwointsq$, $i \rightarrow f_1 \leftarrow j$ and $k \rightarrow f_1$ must exist in $\Dint_2$. This implies two additional v-structures $k\rightarrow f_1 \leftarrow i$ and $k\rightarrow f_1 \leftarrow j$ in $\Dint_2$ but not in $\Dint_1$, contradicting the fact $D_1 \simeq_{\calI} D_2$.
    
    If $f_1 \neq f_2$, we have a v-structure $f_1 \rightarrow k \leftarrow f_2$ in $\Dint_1$ and $\Dint_2$. Because $\Doneintsq$ and $\Dtwointsq$ share the same skeleton but not the v-structure $i \rightarrow k \leftarrow j$, either $f_1 \rightarrow i$ or $f_2 \rightarrow j$ exists in $\Dint_2$. Without loss of generality, we assume $f_1 \rightarrow i$ in $\Dint_2$, now that we have $k \leftarrow f_1 \rightarrow i$ in $\Dint_2$, there must be another factor $f'$ such that $k \rightarrow f' \rightarrow i$ in $\Dint_2$. If $f' = f_2$, there will be a cycle between $k$ and $f_2$. Thus, $f'\neq f_2$. Now we have a v-structure $f_1 \rightarrow i \leftarrow f'$ in $\Dint_2$ but not in $\Dint_1$, leading to a contraction. 
    
    Therefore, we conclude that $D_1\simeq_\mathcal{I}D_2\implies D_1^2[V]\simeq_\mathcal{I}D_2^2[V]$.
    
    Next, we prove the reverse direction. We first claim and prove propositions about the three types of factors. Before presenting the propositions, we repeat important notations here:
    \begin{enumerate}
        \item $par(i;D)$: the set of all parents of node $i$ in an f-DAG $D$.
        \item $chd(i;D)$: the set of all children of node $i$ in an f-DAG $D$.
        \item $P_{i}(D):=\{h: chd(h;D)=\{i\}\}$: the set of parents \textbf{unique} to node $i$ in an f-DAG $D$.
        \item $C_{i}(D):=\{j: par(j;D)=\{i\}\}$: the set of children \textbf{unique} to node $i$ in an f-DAG $D$.
        \item $A\rightarrow f \rightarrow B$: all nodes in set $A$ are connected to all nodes in set $B$ via a factor $f$. We slightly abuse the notation here for conciseness.
    \end{enumerate}
    \begin{proposition}\label{supp:prop:1}
        Given $\Doneintsq \simeq \Dtwointsq$ under the assumptions of lemma \ref{supp:lemma:1}. Then, $\forall j\in F$ such that $par(j;\Dint_1) = \{i\}, chd(j;\Dint_1) = \{k\}$, $\exists j'\in F$ such that $par(j';\Dint_2)=\{i\}, chd(j';\Dint_2)=\{k\}$ or $par(j';\Dint_2)=\{k\}, chd(j';\Dint_2)=\{i\}$.
    \end{proposition}
    \begin{proof}
        Because $\Doneintsq$ and $\Dtwointsq$ share the same skeleton, $\exists j'\in F$ such that $i\rightarrow j' \rightarrow k $ in $ \Dtwointsq$ or $i \leftarrow j' \leftarrow k $ in $ \Dtwointsq$. By assumption $P_{j}(\Dint_1)=par(j;\Dint_1)=\{i\}$ and $C_{j}(\Dint_1)=chd(j;\Dint_1)=\{k\}$. Now let's consider two cases.

        \textbf{Case I: $i \rightarrow j' \rightarrow k $.} Because $k$ is not a collider in $\Dint_1$, $par(j';\Dint_2)=\{i\}$. Now we only need to prove $chd(j';\Dint_2)=\{k\}$. Suppose $\exists k', k'\neq k$ and $k'\in chd(j';\Dint_2)$. To match the skeleton in $\Doneintsq$, we must have $k' \rightarrow f \rightarrow i$ in $\Dint_1$ (Fig. \ref{fig:supp_b15}(a)). Now consider $f$ and $j$. By property 2 in Def. \ref{appx:def_dm}, $|P_f(\Dint_1)|>1$ or $|P_j(\Dint_1)|>1$. Since $P_j(\Dint_1)=\{i\}$, we must have $|P_f(\Dint_1)|>1$. This implies $i$ is a collider in $\Doneintsq$ and $k'\rightarrow i$ in $\Doneintsq$ and hence, in $\Dtwointsq$. This leads to a contradiction.

        \textbf{Case II: $k \rightarrow j' \rightarrow i$.} Because $i$ is not a collider in $\Dint_1$, $par(j';\Dint_2)=\{k\}$. We only need to prove $chd(j';\Dint_2)=\{i\}$. Suppose $\exists i', i'\neq i$ and $i'\in chd(j';\Dint_2)$. Since $C_j(\Dint_1)=\{k\}$ and $P_j(\Dint_1)=\{i\}$, we must have $k \rightarrow f \rightarrow i'$ in $\Dint_1$ (Fig. \ref{fig:supp_b15}(b)). Now consider $j$ and $f$. By property 2 in Def. \ref{appx:def_dm}, $|P_f(\Dint_1)|>1$ or $|P_j(\Dint_1)|>1$. Since $par(j;\Dint_1)=\{i\}$, we must have $|P_f(\Dint_1)|>1$. Let $k'\neq k$ such that $k'\in P_f(\Dint_1)$. Now $i'$ is a collider in $\Doneintsq$ and hence, $\Dtwointsq$. We cannot have $k' \rightarrow j' \rightarrow i'$ because this would cause the v-structure $k' \rightarrow i \leftarrow k$ in $\Dtwointsq$ and contradicts $i \rightarrow k$ in $\Doneintsq$. Thus, $\exists f'\neq j'$, $k'\rightarrow f' \rightarrow i'$ in $\Dint_2$. This implies $i'\notin C_{f'}(\Dint_2)$. We can pick $l'\in C_{f'}(\Dint_2)$. It is clear that $k$ is not connected to $l'$ in $\Dtwointsq$. Otherwise, we would have $k\rightarrow f' \implies k\notin P_{j'}(\Dint_2) \implies i$ is a collider in $\Dtwointsq \implies k \rightarrow i$ in $\Dint_1$, which is a contradiction. Now consider $l'$ in $\Dint_1$. We claim that $k'\rightarrow l'$ in $\Doneintsq$. Otherwise, $\exists g\in F$ such that $l'\rightarrow g \rightarrow k'$ in $\Dint_1$ and $par(g;\Dint_1)=\{l'\}$, which contradicts property 3 in Def. \ref{appx:def_dm}. Now we can conclude that $k'\rightarrow f \rightarrow l'$. Since $k\in P_f(\Dint_1)$, this implies $k\rightarrow l'$ in $\Doneintsq$ and leads to a contradiction.
        
        Thus, we conclude that $\exists j'$ such that $par(j';\Gint_2)=\{i\}, chd(j';\Gint_2)=\{k\}$ or $par(j';\Gint_2)=\{k\}, chd(j';
        .3\Gint_2)=\{i\}$.
    \end{proof}
    \begin{figure}[h!]
        \centering
        \includegraphics[width=\linewidth]{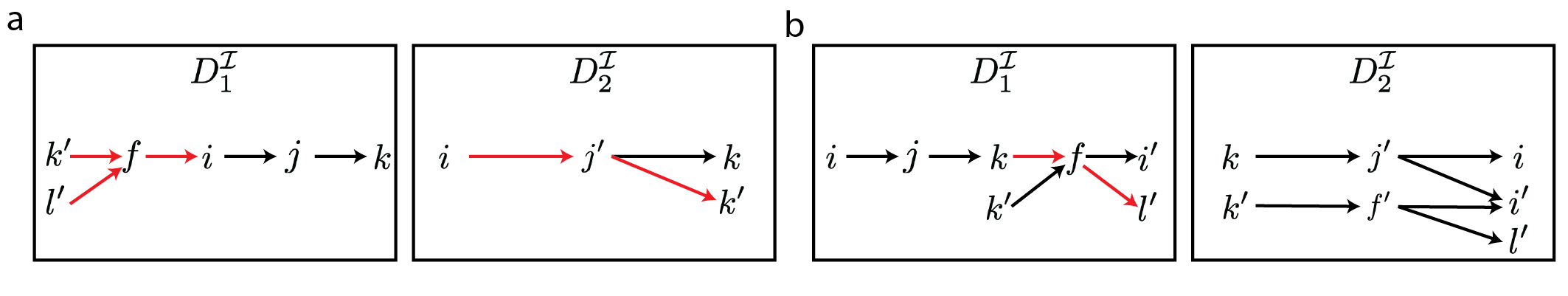}
        \caption{\textbf{Illustration of Proposition \ref{supp:prop:1}} \textbf{(a)} Case I:  $i \rightarrow j' \rightarrow k$ in $\Dint_2$. \textbf{(b)} Case II: $k \rightarrow j' \rightarrow i$ in $\Dint_2$. Contradictory edges are colored red.}
        \label{fig:supp_b15}
    \end{figure}

    \begin{proposition}\label{supp:prop:2}
        Given $\Doneintsq \simeq \Dtwointsq$ under the assumptions of lemma \ref{supp:lemma:1}. Then, $\forall j\in F$ such that $par(j;\Dint_1)=P_j(\Dint_1) = \{i\}, |chd(j;\Dint_1)| > 1$, $\exists j' \in F$, $par(j';\Dint_2)=\{i\}, chd(j';\Dint_2)=chd(j;\Dint_1)$.
    \end{proposition}
    \begin{proof}
        Because $\Doneintsq$ and $\Dtwointsq$ share the same skeleton and $|chd(j;\Dint_1)|>1$, $\forall k\in chd(j;\Dint_1)$, either $(i,k)$ or $(k,i)$ exists in $\Dtwointsq$. We first claim that the edges between $i$ and $chd(j;\Dint_1)$ in $\Dtwointsq$ should all be from $i$ to $chd(j;\Dint_1)$. To prove this, we consider any $k\in chd(j;\Dint_1)$. If $k\notin C_j(\Dint_1)$, $\exists f\in F, f\neq j, h\in V, h\neq i$ such that $h\rightarrow f\rightarrow k$ in $\Dint_1$. Hence, there is a v-structure $h \rightarrow k \leftarrow i$ in $\Doneintsq$ and it has to exist in $\Dtwointsq$. This implies $i \rightarrow k$ in $\Dtwointsq$ (Fig. \ref{fig:supp_b16}(a)). 
        
        We have shown $i\rightarrow chd(j;\Dint_1)\backslash C_j(\Dint_1)$ in $\Dtwointsq$. Now we prove $i\rightarrow C_j(\Dint_1)$ in $\Dtwointsq$ by contradiction. If this is not the case, there are two cases:
        
        \textbf{Case I.} $\exists k_1, k_2\in C_j(\Dint_1), k_1\neq k_2, f_1, f_2\in F$ such that $k_1\rightarrow f_1 \rightarrow i, k_2 \rightarrow f_2 \rightarrow i$ in $\Dint_2$. There will be a v-structure $k_1 \rightarrow i \leftarrow k_2, k_1, k_2\in A$, but such a v-structure does not exist in $\Doneintsq$ (Fig. \ref{fig:supp_b16}(b)). Therefore, we have a contradiction.
        
        \textbf{Case II.} There is only one $h\in C_j(\Dint_1), f'\in F$ such that $h\rightarrow f' \rightarrow i$ in $\Dint_2$. Since $|chd(j;\Dint_2)|>1$, $\exists k\neq h, k\in chd(j;\Dint_2)$ and $j'\in F$ such that $i\rightarrow j' \rightarrow k$ in $\Dint_2$. By property 2 of Def. \ref{appx:def_dm}, $|P_{f'}(\Dint_1)|>1$ or $|P_{j'}(\Dint_1)|>1$. If $|P_{f'}(\Dint_1)|>1$, $i$ will be a collider and $h \rightarrow i$ is in $\Dtwointsq$, but $i \rightarrow h$ in $\Doneintsq$ and we have a mismatch in the v structure. Therefore, the only possibility is $|P_{j'}(\Dint_1)|>1$ and let $i'\neq i$ be one of them (Fig. \ref{fig:supp_b16}(c)). Because $par(j;\Dint_1)=\{i\}$, $i'\not\rightarrow j$ in $\Dint_1$ and consequently, $\exists f\in F,f\neq j$ such that $i' \rightarrow f \rightarrow k$ in $\Dint_1$. Since $j \rightarrow k$ and $f \rightarrow k$ in $\Dint_1$, $k\notin C_f(\Dint_1)$. Thus, we can pick $l\in C_f(\Dint_1), l\neq k$. We claim that $i'\rightarrow l$ in $\Dtwointsq$. Otherwise, we will have $l\rightarrow g \rightarrow i'$ and $par(g;\Dint_2)=\{l\}$, just like case II in the proof of proposition \ref{supp:prop:1}.  Thus, we have $i \rightarrow l$ in $\Dtwointsq$, which contradicts $i\not\rightarrow l$ in $\Doneintsq$.

        So we conclude that $\forall k\in chd(j;\Dint_1),\exists j'\in F$ such that $i\rightarrow j' \rightarrow k$ in $\Dint_2$.
        
        Now that all edges between $i$ and $chd(j;\Dint_2)$ start from $i$, the remaining piece is to show such $j'$ is unique. If $\exists J \subseteq F, |J|>1$ such that $\forall k\in chd(j;\Dint_1)$, $\exists j'\in J$, $i \rightarrow j' \rightarrow k$ in $\Dint_2$, $i\notin P_{j'}(\Dint_2)\; \forall j'\in J$. Now consider any $k\in C_j(\Dint_1)$. Suppose $i \rightarrow j' \rightarrow k$ in $\Dint_2$. Because $i\notin P_{j'}(\Dint_2)$, $\exists h\in P_{j'}(\Dint_2)$ such that $h \rightarrow k$ in $\Dint_2$. Now, we have the v-structure $h \rightarrow k \leftarrow i$ in $\Dtwointsq$ and $\Doneintsq$. However, $h \rightarrow k$ cannot exist in $\Doneintsq$ because $h \notin par(j;\Dint_1)$ and $k \in P_j(\Dint_1)$~(Fig. \ref{fig:supp_b16}(d)). We have a contradiction.

        Thus, we know that $\exists$ unique $j'$ such that $i \rightarrow j' \rightarrow chd(j;\Dint_1)$ in $\Dint_2$ and consequently $chd(j;\Dint_1) \subseteq chd(j';\Dint_2)$. The last step is to prove $chd(j';\Dint_2)=chd(j;\Dint_1)$, which only requires proving $chd(j';\Dint_2) \subseteq chd(j;\Dint_1)$. Suppose $\exists k'\in chd(j';\Dint_2)$ such that $j'\notin chd(j;\Dint_1)$, $(i, k')$ is in $\Dtwointsq$ but not $\Doneintsq$. To match the skeleton, $(k', i)$ must exist in $\Doneintsq$. Then, $\exists f\in F$ such that $k' \rightarrow f \rightarrow i$ in $\Dint_1$ (Fig. \ref{fig:supp_b16}(e)).
        
        We claim that $k'\in C_{j'}(\Dint_2)$. Otherwise, $k'$ is a collider in $\Dtwointsq$ and we must have a v-structure with $i\rightarrow k'$ in $\Dtwointsq$ but not in $\Doneintsq$. Now we can assume $k'\in C_{j'}(\Dint_2)$. By property 2 in Def. \ref{appx:def_dm}, $|P_{j}(\Dint_1)| > 1$ or $|P_{f}(\Dint_1)| > 1$. Because $|par(j\Dint_1)|=1$ by assumption, $|P_{f}(\Dint_1)| > 1$ and this implies $i$ is a collider with $k'$ as its parent in $\Dint_1$ (Fig.\ref{fig:supp_b16}(e)). However, this cannot happen in $\Dint_2$. We have a contradiction.  

        Finally, we conclude that for any factor in $\Dint_1$ under the conditions in proposition \ref{supp:prop:2}, there exists a factor in $\Dint_2$ with the same parent and children.
    \end{proof}
    \begin{figure}[h!]
        \centering
        \includegraphics[width=\linewidth]{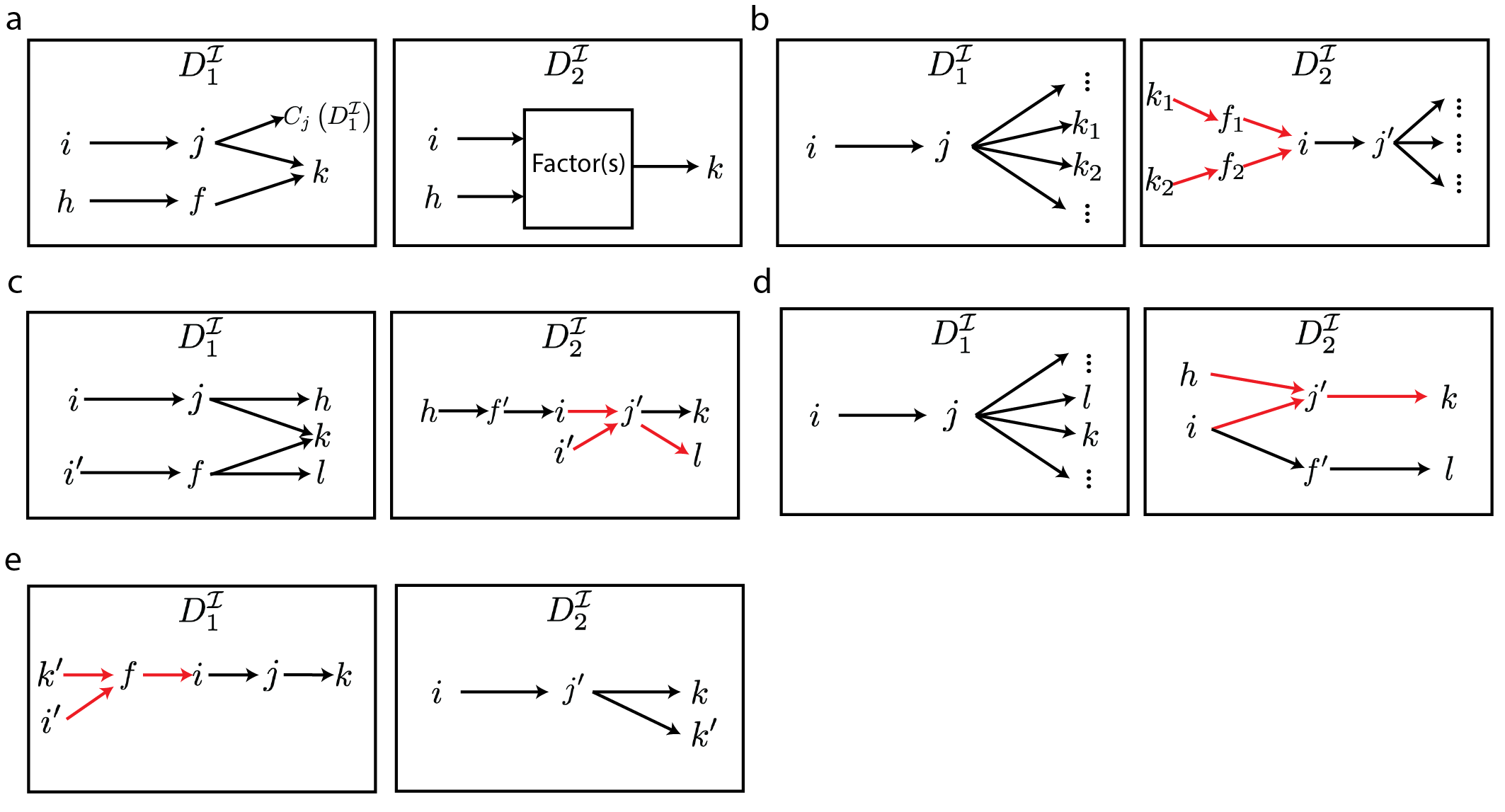}
        \caption{\textbf{Illustration of Cases in Proposition \ref{supp:prop:2}} \textbf{(a)} $k\in chd(j), k\notin C_j(\Dint_1)$. \textbf{(b)} Contradictory Case I: $\exists k_1, k_2\in C_j(\Dint_1)$ s.t. $k_1 \rightarrow f_1 \rightarrow i, k_2 \rightarrow f_2 \rightarrow i$ in $\Dint_2$ \textbf{(c)} Contradictory case II: $\exists$ unique $h$ s.t. $h \rightarrow f' \rightarrow i$ in $\Dint_2$. \textbf{(d)} Contradictory Case: there exists multiple factors connecting $i$ to $chd(j;\Dint_1)$. \textbf{(e)} Contradictory Case: $chd(j;\Dint_) \subset chd(j';\Dint_2)$.}
        \label{fig:supp_b16}
    \end{figure}

    \begin{proposition}\label{supp:prop:3}
        If $\Doneintsq \simeqint \Dtwointsq$ under the assumptions of lemma \ref{supp:lemma:1}, then $\forall j\in F$ such that $|par(j;\Dint_1)| > 1$, $\exists j'\in F$ such that $par(j;\Dint_1)=par(j';\Dint_2)$ and $chd(j;\Dint_1)=chd(j';\Dint_2)$.
    \end{proposition}
    \begin{proof}
        The proof is trivial when $|F|=1$. Now we assume $|F|>1$.  We first prove $\exists j', P_j(\Dint_1)=P_{j'
        }(\Dint_2)$ and $C_j(\Dint_1)=C_{j'}(\Dint_2)$.

        Since $|par(j;\Dint_1)| > 1$, any node in $C_j(\Dint_1)$ is a collider in $\Doneintsq$. Thus, $\forall i\in P_j(\Dint_1), k\in C_j(\Dint_1)$, $(i, k)\in \Dtwointsq$. Thus, $P_j(\Dint_1)$ is still connected to $C_j(\Dint_1)$ in $\Dtwointsq$ via a set of factors $F_C=\{f_1,\ldots,f_M\}$ in $\Dint_2$. Now we consider two cases. 
        
        \textbf{Case I.} First, we consider $M > 1$.  $\forall f_m\in F_C$, we define $P_m=par(f_m;\Dint_2)\cap P_j(\Dint_1)$ and $C_m=chd(f_m;\Dint_2)\cap C_j(\Dint_1)$. By our assumption $P_m\neq \emptyset, C_m\neq \emptyset, \forall m$.
        
        Pick any two factors $f_1$ and $f_2$. We consider the following two conditions
        \begin{enumerate}
            \item $\exists i\in P_j(\Dint_1)$ such that $i\rightarrow f_1$ and $i\rightarrow f_2$ in $\Dint_2$.
            \item $\exists k\in C_j(\Dint_1)$ such that $f_1 \rightarrow k$ and $f_2 \rightarrow k$ in $\Dint_2$.
        \end{enumerate}
        Each condition can be either true or false. The combination of these two conditions give us four different subcases.
        
        \textbf{Subcase I-1: Condition (1) and (2) are True.} 
        
        This means $i$ is a common parent of $f_1,f_2$ in $\Dint_2$. Now consider $P_{f_1}(\Dint_2)$ and $P_{f_2}(\Dint_2)$. $P_{f_1}(\Dint_2)$ and $\{i\}$ are disjoint and they form v-structures at any node in $C_{f_1}(\Dint_2)$ in $\Dint_2$. Similarly, $P_{f_2}(\Dint_2)$ and $\{i\}$ form v-structures at any node in $C_{f_2}(\Dint_2)$. We also have $P_{f_1}(\Dint_2)$ and $P_{f_2}(\Dint_2)$ form v-structures at $k$ in $\Dint_2$. By the definition of $k$, $P_{f_1}(\Dint_2)\cup P_{f_2}(\Dint_2)$ is connected to $j$ in $\Dint_1$. Because both $P_{f_1}(\Dint_2)$ and $P_{f_2}(\Dint_2)$ are connected to $j$ in $\Dint_1$, we have $ P_{f_1}(\Dint_2)$, $P_{f_2}(\Dint_2)$ form v-structures at $C_{f_1}(\Dint_2)$ in $\Doneintsq$, but such a v-structure cannot exist in $\Dtwointsq$ by the definition of $C_{f_1}(\Dint_2)$. We reach contradiction. Fig. \ref{fig:supp_prop_B17_case_1}(a) shows this scenario.

        \textbf{Subcase I-2: Condition (1) is False but (2) True.} 
        
        We claim that $|par(f_1;\Dint_2)| > 1$. Otherwise, by proposition \ref{prop:2}, $\exists j'$ in $\Dint_1$ such that $par(j')=par(f_1)=P_1, chd(j')=chd(f_1)$. Because $C_1 \subset C_j(\Dint_1)$, the only possibility is $j'=j$. But we know $|par(j';\Dint_1)| = |P_1| = 1$ and $|par(j;\Dint_1)| > 1$. This leads to a contradiction. Hence, $|par(f_1;\Dint_2)| > 1$.  Similarly, $|par(f_2;\Dint_2)| > 1$. Consequently, we must have v-structures at $C_{f_1}(\Dint_2)$ and $C_{f_2}(\Dint_2)$ in $\Dint_2$. This implies the edge connections between $P_1$ and $C_{f_1}(\Dint_2)$ are all from $P_1$ to $C_{f_1}(\Dint_2)$. Similarly, edge connections between $P_2$ and $C_{f_2}(\Dint_2)$ are all from $P_2$ to $C_{f_2}(\Dint_2)$. By the definition of $P_1$ and $P_2$, we have $j\rightarrow C_{f_1}(\Dint_2)$ and $j\rightarrow C_{f_2}(\Dint_2)$. This means $P_1$ and $P_2$ form v-structures at $C_{f_1}(\Dint_2)$ in $\Doneintsq$. But such v-structure(s) cannot exist in $\Dtwointsq$. Again we have a contradiction. Fig. \ref{fig:supp_prop_B17_case_1}(b) shows the scenario.

        \textbf{Subcase I-3: Condition (1) is True but (2) False.}

        This implies $C_1$ and $C_2$ are disjoint sets. Now consider $P_{f_1}(\Dint_2)$ and $P_{f_2}(\Dint_2)$. $P_{f_1}(\Dint_2)$ and $i$ form v-structures at $C_1$ in $\Dtwointsq$. Hence, such a v-structure must exist in $\Doneintsq$. By the definition of $C_1$, $P_{f_1}(\Dint_2) \rightarrow j$ in $\Dint_1$. Similarly, $P_{f_2}(\Dint_2)\rightarrow j$ in $\Dint_1$. Again, we have $P_{f_1}(\Dint_2)$ and $P_{f_2}(\Dint_2)$ for v-structures at $C_1$, which contradicts the fact that $C_1$ and $C_2$ are disjoint. Fig. \ref{fig:supp_prop_B17_case_1}(c) shows the scenario.

        \textbf{Subcase I-4: Condition (1) and (2) are false.}

        This implies $C_1$ and $C_2$ are disjoint and $P_1$ and $P_2$ are disjoint. By proposition \ref{supp:prop:1} and \ref{supp:prop:2}, $|par(f_1;\Dint_2)| > 1$ and $|par(f_2;\Dint_2)| > 1$ under the same argument as in subcase I-2. Hence, we have v-structures from $par(f_1;\Dint_2)$ to $C_1$ and $par(f_2;\Dint_2)$ to $C_2$. This means the edges must point from $P_1$ to $C_1$ and from $P_2$ to $C_2$ in $\Doneintsq$. By the definition of $P_1,P_2,C_1,C_2$, they are all connected via $j$ in $\Dint_1$. Hence, $P_1$ is also connected to $C_2$, but in $\Dtwointsq$ this is not true. We reach contradiction again. Fig. \ref{fig:supp_prop_B17_case_1}(d) shows the scenario.

        \begin{figure}[h!]
            \centering
            \includegraphics[width=1\linewidth]{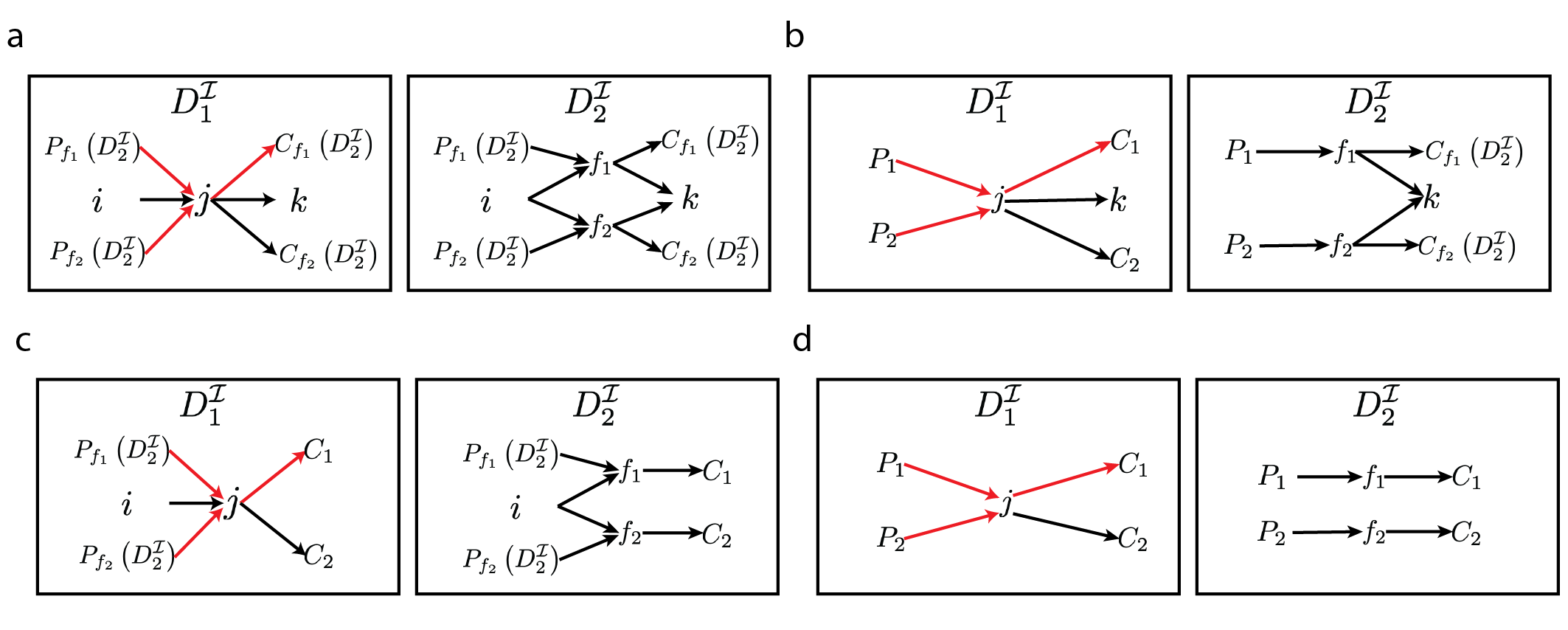}
            \caption{\textbf{Illustration of Cases in Case I of Proposition \ref{supp:prop:3}} \textbf{(a)} Subcase I-1. \textbf{(b)} Subcase I-2. \textbf{(c)} Subcase I-3. \textbf{(d)} Subcase I-4.}
            \label{fig:supp_prop_B17_case_1}
        \end{figure}
        Thus, we know that $M>1$ is impossible.
        
        \textbf{Case II.} We already excluded the possibility of $M>1$ in case I. Now we consider $M = 1$. There are two parts: prove equality of unique parents and equality of unique children.
        
        Let $F_C = \{f\}$.
        We first prove $P_j(\Dint_1)\subseteq P_{f}(\Dint_2)$. Suppose $\exists i\in P_j(\Dint_1)$, $i\notin P_{f}(\Dint_2)$. At the beginning of this proof, we concluded that $P_j(\Dint_1)$ must connect to $C_j(\Dint_1)$ in $\Dtwointsq$ due to the presence of v-structures. Since $i\in P_j(\Dint_1), i\notin P_{f}(\Dint_2)$, the only possibility is $i$ being connected to another factor $f'\neq f$ in $\Dint_2$. This further implies $i\notin P_{f'}(\Dint_2)$. 
        
        Now consider the v-structure formed by $i$, any node in $P_{f'}(\Dint_2)$ and any node in $C_{f'}(\Dint_2)$. Due to such a v-structure, connection between $i$ and $C_{f'}(\Dint_2)$ should be from $i$ to $C_{f'}(\Dint_2)$ in $\Doneintsq$. Because $i\in P_j(\Dint_1)$, $P_j(\Dint_1)$ should all be connected to $C_{f'}(\Dint_2)$ via $j$. In short, we have $j\rightarrow C_{f'}(\Dint_2)$ in $\Dint_1$.

        On the other hand, $i$, $P_f(\Dint_2)$ and $C_f(\Dint_2)\cup C_j(\Dint_1)$ form v-structures in $\Dtwointsq$. Hence, $P_f(\Dint_2)$ is connected to $C_j(\Dint_1)$ via $j$ in $\Dint_1$.
        
        Based on the reasoning above, we have (1) $j\rightarrow C_{f'}(\Dint_2)$  and (2) $P_f(\Dint_2)\rightarrow j $ in $\Dint_1$. This implies $P_f(\Dint_2) \rightarrow C_{f'}$ in $\Doneintsq$. Due to the v-structure formed by $i$, $P_{f'}(\Dint_2)$ and $C_{f'}(\Dint_2)$ in $\Dtwointsq$, we also have $P_{f'}(\Dint_2) \rightarrow C_{f'}(\Dint_2)$ in $\Doneintsq$. The previous two facts further imply we have a v-structure $P_f(\Dint_2) \rightarrow C_{f'}(\Dint_2) \leftarrow P_{f'}(\Dint_2)$ in $\Doneintsq$, but such v-structure is not present in $\Dtwointsq$. Thus, we reach a contradiction. Fig. \ref{fig:supp_prop_B17_case_2}(a) shows the scenario.

        To conclude, so far we have proven $P_j(\Dint_1) \subseteq P_{f}(\Dint_2)$.

        Next, we prove $P_{f}(\Dint_2) \subseteq P_j(\Dint_1)$. Suppose $\exists i\in P_{f}(\Dint_2), i\notin P_j(\Dint_1)$. First, by our previous argument, $P_j(\Dint_1) \rightarrow j \rightarrow C_j(\Dint_1)$ in $\Dint_1$ and $P_j(\Dint_1) \rightarrow f \rightarrow C_j(\Dint_1)$ in $\Dint_2$. By our assumption, we also have $i \rightarrow j \rightarrow C_j(\Dint_1)$in $\Dint_1$ due to the v-structures at $C_j(\Dint_1)$. However, since $i \notin P_j(\Dint_1)$, $\exists j' \neq j, i \rightarrow j'$ in $\Dint_1$. Apparently, $i\notin P_{j'}(\Dint_1)$, so we can consider the v-structure formed by $i, P_{j'}(\Dint_1), C_{j'}(\Dint_1)$. This implies $i\rightarrow C_{j'}(\Dint_1)$ in $\Dtwointsq$. Because $i\in P_f(\Dint_2)$, we must have $f\rightarrow C_{j'}(\Dint_1)$ in $\Dint_2$. This means $P_j(\Dint_1), i, C_{j'}(\Dint_1)$ form v-structures at $C_{j'}(\Dint_1)$ in $\Dtwointsq$ and hence, $P_j(\Dint_1)\rightarrow C_{j'}(\Dint_1)$ in $\Doneintsq$. However, there is no connection between $P_j(\Dint_1)$ and $C_{j'}(\Dint_1)$. We reach a contradiction. Now we can conclude that $P_j(\Dint_1)=P_f(\Dint_2)$. Fig. \ref{fig:supp_prop_B17_case_2}(b) shows the scenario.

        Next, we can proceed to prove $C_j(\Dint_1) = C_f(\Dint_2)$. First, we prove $C_j(\Dint_1) \subseteq C_f(\Dint_2)$ by contradiction. Suppose $\exists k\in C_j(\Dint_1), k\notin C_f(\Dint_2)$. By the assumption of case II, $P_j(\Dint_1)$ is connected to $C_j(\Dint_1)$ via a unique factor, $f$, in $\Dint_2$. We have $P_j(\Dint_1) \rightarrow f \rightarrow k$ in $\Dint_2$. Due to the v-structures at $C_j(\Dint_1)$ in $\Doneintsq$, all nodes in $par(j;\Dint_1)$ must connect to $C_j(\Dint_1)$ in $\Dint_2$. Therefore, we have a stronger conclusion:
        $$ \forall i\in par(j;\Dint_1), i\rightarrow f \rightarrow C_j(\Dint_1)\cup C_f(\Dint_2) \mbox{ in } \Dint_2$$
        
        Because $k \notin C_f(\Dint_2)$, there must be another factor $f'\neq f$ such that $f'\rightarrow k$ in $\Dint_2$. Now pick $l\in P_{f'}(\Dint_2)$ and $m\in C_f(\Dint_2)$.  Since $|par(f;\Dint_2)|>1, |par(j;\Dint_1)|>1$, we can pick $i'\in par(f;\Dint_2), i'\neq i$. We have the v-structures $i \rightarrow k \leftarrow l$, $i'\rightarrow k \leftarrow l$ and $i \rightarrow m \leftarrow i'$. in $\Dtwointsq$. Because $f\neq f'$, $l$ and $m$ are not connected in $\Dtwointsq$. On the other hand, because $k \in C_j(\Dint_1)$ and the v-structure $i \rightarrow k \leftarrow l$ must exist in $\Doneintsq$, we must have $ l \rightarrow j \rightarrow m$ in $\Dint_1$ and hence, $l$ and $m$ are connected. This leads to a contradiction. \textbf{Fig. \ref{fig:supp_prop_B17_case_2}c} shows the scenario.

        Next, we prove $C_f(\Dint_2) \subseteq C_j(\Dint_1)$ using a similar argument. Suppose $\exists k\in C_f(\Dint_2), k\notin C_j(\Dint_1)$. Consider $i\in P_f(\Dint_2), i' \in par(f;\Dint_2)$, we have v-structures $i \rightarrow k \leftarrow i'$ and $i \rightarrow m \leftarrow i',\forall m\in C_j(\Dint_1)$ in $\Dtwointsq$. This implies $\{i,i'\}\rightarrow C_j(\Dint_1)\cup \{k\}$ in $\Doneintsq$. By the definition of $C_f(\Dint_2)$ and our previous conclusion that $P_j(\Dint_1)=P_f(\Dint_2)$, we must have $i \rightarrow j \rightarrow C_j(\Dint_1)\cup \{k\}$ in $\Dint_1$. Because $k \notin C_j(\Dint_1)$, $\exists j'\neq j$ such that $j'\rightarrow k$ in $\Dint_1$. Now pick $l\in P_{j'}(\Dint_1)$. We have the v-structure $l \rightarrow k \leftarrow i$ in $\Doneintsq$. Hence, such a v-structure must exist in $\Dtwointsq$. Because $k \in C_f(\Dint_2)$, we must have $l \rightarrow f \rightarrow k$ in $\Dint_2$. This implies the v-structure $l$ is connected to any $m\in C_j(\Dint_2)$ in $\Dtwointsq$, but such connection does not exist in $\Doneintsq$ since $l\in P_{j'}(\Dint_1), m \in C_j(\Dint_1), j'\neq j$. \textbf{Fig. \ref{fig:supp_prop_B17_case_2}d} shows the scenario.

        From the proof above, we have $\forall j$, $\exists f$, $P_j(\Dint_1) = P_f(\Dint_2)$ and $C_j(\Dint_1)=C_f(\Dint_2)$.

        \begin{figure}[h!]
            \centering
            \includegraphics[width=1\linewidth]{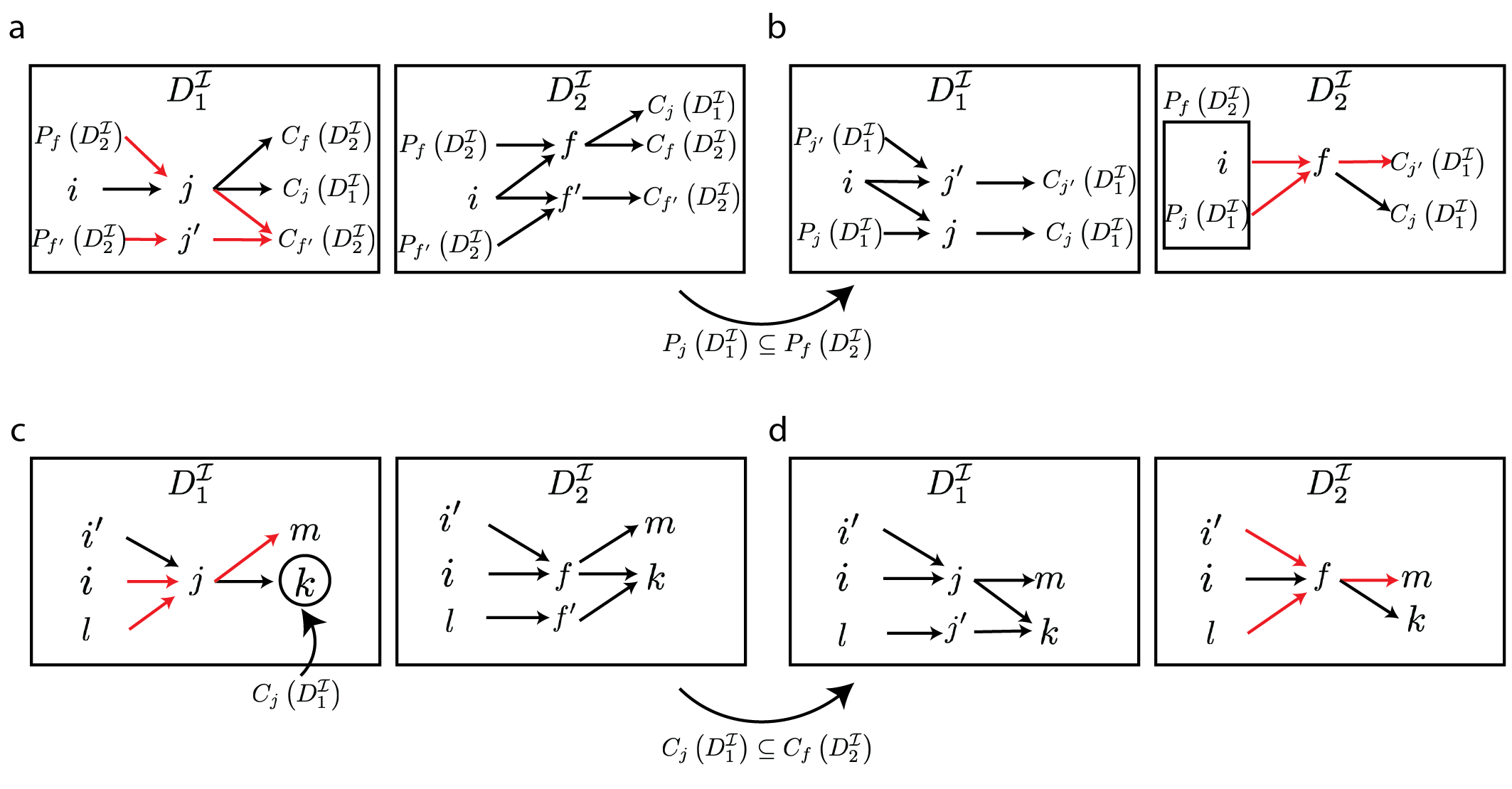}
            \caption{\textbf{Illustration of Cases in Case II of Proposition \ref{supp:prop:3}.} \textbf{(a)} Contradictory case when $P_j(\Dint_1)\not\subseteq P_f(\Dint_2)$. \textbf{(b)} Contradictory case when $P_j(\Dint_1)\subseteq P_f(\Dint_2)$ but $P_f(\Dint_2)\not\subseteq P_j(\Dint_1)$. \textbf{(c)} Contradictory case when $C_j(\Dint_1)\not\subseteq C_f(\Dint_2)$. \textbf{(d)} Contradictory case when $C_j(\Dint_1)\subseteq C_f(\Dint_2)$ but $C_f(\Dint_2)\not\subseteq C_j(\Dint_1)$.}
            \label{fig:supp_prop_B17_case_2}
        \end{figure}

        Because $|par(j;\Dint_1)| > 1$ and $|chd(j;\Dint_1)| > 1$, a v-structure forms $\forall i_1, i_2\in par(j;\Dint_1), k\in chd(j;\Dint_1)$. Hence, $\forall i\in par(j;\Dint_1), k\in chd(j;\Dint_1)$, the direction of edge $(i,k)$ must be from $i$ to $k$. Because we have $P_j(\Dint_1) = P_f(\Dint_2)$ and $C_j(\Dint_1)=C_f(\Dint_2)$ and any $i \in par(j;\Dint_1)$ must form v-structures at any $k\in C_j(\Dint_1)$ in $\Dtwointsq$, we have $i \rightarrow k$ in both $\Doneintsq$ and $\Dtwointsq$. The only possibility for $(i, k)$ exist in $\Dtwointsq$ is $i \rightarrow f \rightarrow k$. This implies $par(j;\Dint_1)\subseteq par(f;\Dint_2)$. Similarly, $chd(j;\Dint_1)\subseteq chd(f;\Dint_2)$. $\forall i' \in par(f;\Dint_2)$, $i'\in par(j;\Dint_1)$ because otherwise $(i', k),\forall k\in C_j(\Dint_1)$ will not be in $\Doneintsq$. Therefore, we have $par(f;\Dint_2)\subseteq par(j;\Dint_1)$ and similarly $chd(f;\Dint_2)\subseteq chd(j;\Dint_1)$.

        Finally, we conclude that $\forall j\in F, \exists f\in F$, $par(f;\Dint_2) = par(j;\Dint_1)$ and similarly $par(f;\Dint_2) = par(j;\Dint_1)$.
    \end{proof}

    Propositions \ref{supp:prop:1}-\ref{supp:prop:3} provide insights about extra conditions in order for two f-DAGs to be $\mathcal{I}$-Markov equivalent. The proof of lemma \ref{supp:lemma:1} will be straightforward given these propositions.

    Given the conditions in lemma \ref{supp:lemma:1}, we have two f-DAGs with the same set of nodes and factors. Now let's consider any factor $j\in F$. For any extended f-DAG, $D$, and interventions, $\calI$, We define the following three types of factors:
    \begin{enumerate}
        \item $j: |par(j;\Dint)| = |chd(j;\Dint)| = 1$
        \item $j: |par(j;\Dint)|=1, |chd(j;\Dint)| > 1$
        \item $j: |par(j;\Dint)| > 1$
    \end{enumerate}
    We can partition $F$ into three subsets $F_1, F_2, F_3$ for $\Dint_1$ and $F_1', F_2', F_3'$ for $\Dint_2$. Because $\Doneintsq\simeqint\Dtwointsq$, by proposition \ref{supp:prop:1}, $\forall j\in F_1$, $\exists j'\in F_1'$ with the same parents and children and $\forall j'\in F_1'$, $\exists j\in F_1$ with the same parents and children. Hence, there is a bijection between $F_1$ and $F_1'$. Similarly, we have bijections from $F_2$ to $F_2'$ by proposition \ref{supp:prop:2} and $F_3$ to $F_3'$ by proposition \ref{supp:prop:3}. Finally, we have a bijection $\phi:F\rightarrow F$ from all factors in $\Dint_1$ to factors in $\Dint_2$. $\forall f\in F_2\cup F_3$, $f$ and $\phi(f)$ share the same parent(s) and children. $\forall f\in F_1$, $f$ and $\phi(f)$, either $par(f;\Dint_1)=par(\phi(f);\Dint_2), chd(f;\Dint_1)=chd(\phi(f);\Dint_2)$ or $par(f;\Dint_1)=chd(\phi(f);\Dint_2), chd(f;\Dint_1)=par(\phi(f);\Dint_2)$. Since $|par(f;\Dint_1)|=|chd(f;\Dint_1)|=1$, flipping the parent and child maintains the same skeleton and does not introduce v-structure.
    
    Finally, we conclude that $\Dint_1$ and $\Dint_2$ share the same skeleton and v-structures and $D_1^2[V] \simeqint D_2^2[V] \implies D_1 \simeqint D_2$
\end{proof}

Next, we discuss the identifiability of an f-DAG. The following theorem summarizes our conclusion.
\begin{theorem}[Identifiability of the f-DAG via ELBO maximization]\label{thm:iden_fdag}
    Let $\{X_i:i\in[n]\}$ be a set of causally related random variables whose joint distribution follows a causal  DAG $G^*$ and  $\calIstar=\{I_k^*:k\in[n^{\calI}]\}$ be a set of interventions where $I_1=\emptyset$. Suppose $G^*\in \calG_m$ and there is a true f-DAG representation, $D^*\in \calD_m$, such that $(D^*)^2[V] = G^*$. Let $\qG$ be a variational distribution on $\calG_m$. Define the Bayesian score function as
    $$ S_{\calIstar}(q):=\mathbb{E}_{\qG}\left[\scoreIstar\right] - \beta KL(\qG||p(G))$$
    where $\scoreIstar = \sup_{\bPhi}\sum_{k=1}^{n^{\calI}}\mathbb{E}_{\pk(\X)}\left[\log \fk(\X|G;\bPhi)\right]-\lambda |G|$ is the score function. Suppose $q^*(G;\bLambda)=\arg\max_{\qG}\bayesianscorestar$ and $\hat{G}=\arg\max_{G}q^*(G;\bLambda)$. Let $\hat{G}$. Then, under the same assumptions of Theorem \ref{thm:supp:dcdi}, namely sufficient capacity, $\calI$-faithfulness, positivity and finite entropy, and for sufficiently small $\lambda > 0$ and $\beta > 0$, $\hat{D}=\arg\max_{D}q^*(D^2[V];\bLambda)$ is $\calIstar$-Markov equivalent to $D^*$ under a permutation of factors.
\end{theorem}

\begin{proof} The proof is a direct result of previous results.

First, we have $G^*\in\calG_m$. Therefore, by Theorem \ref{thm:supp:elbo_identify}, $q^*(G)=\arg\max_{\qG:supp(q)\subseteq\calG_m}\calL(q)$ has an argmax graph $\hat{G}=(\hat{D})^2[V]$ such that $\hat{G}\simeqintstar G^*$.  Next, since $\hat{D},D^* \in \calD_m$, by lemma \ref{supp:lemma:1}, $\hat{D}\simeqintstar D^*$ under a permutation of factors.
\end{proof}


The proposed model architecture is presented in \Cref{fig:abcdefg}

\begin{figure}
    \centering
    \includegraphics[width=1\linewidth]{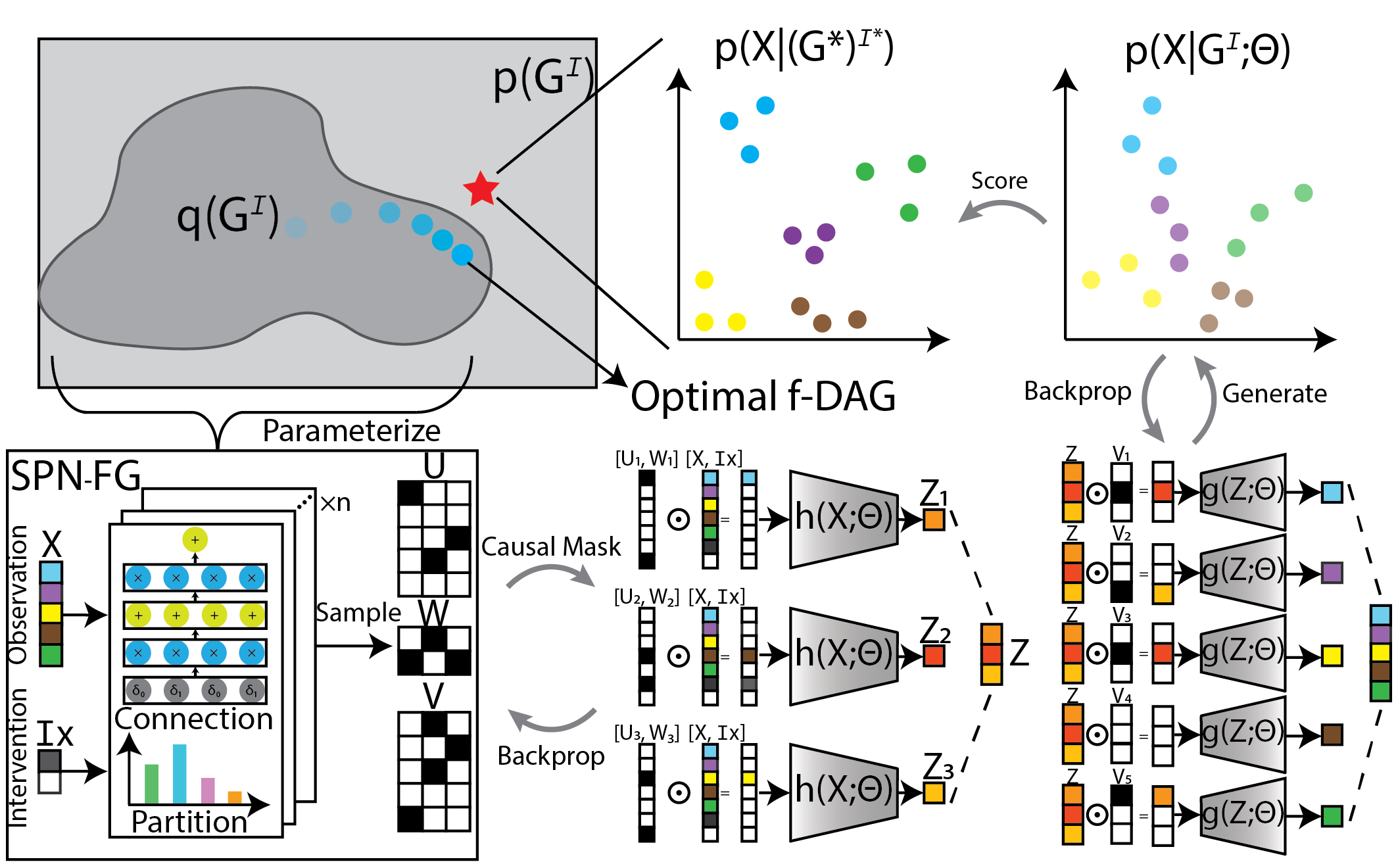}
    \caption{\textbf{Overview of ABCDEFG.} \textbf{Top Left: } Bayesian framework. A prior $p$ with a support of all DAGs and a variational distribution with a support of f-DAGs. The red star represents the ground-truth DAG and light blue dots with increasing transparency show an optimization process w.r.t. the variational distribution. \textbf{Top Right: } Real vs. generated data distribution. \textbf{Bottom: } ABCDEFG model architecture. Binary matrices $\U,\V,\W$ are sampled from a parametric f-DAG model such as SPN-FG. Next, observations are masked by sampled causal relations~(under Hadamard product, $\odot$) and fed to a VAE model fitting data distribution. Arrows show direction of data flow and back propagation.}
    \label{fig:abcdefg}
\end{figure}

\section{Supplementary Results}\label{supp:results}
\subsection{Results on Toy and Extended Datasets}
We benchmarked existing methods on simulated data using both SPN-FG and previous f-DAG simulation method from \citet{Lopez2022largescale}. In a preliminary study, we tested all methods on simple toy datasets simulated with 16 nodes and 2 factors (\Cref{table:sim_targeted}). We changed the sparsity penalty in ENCO but it produced mainly zero adjacency matrix except for one dataset with 0.13 F1 score. Hence, we report zero F1 scores here as a placeholder. Then we extend our experiment to 200 and 500 nodes with nonlinear intervention, to evaluate the performance on larger graph (\Cref{table:Scored based methods F1 and SHD targeted with 200 and 500 nodes}). Note that ENCO and DCDI were too slow and/or required too much memory on larger graphs, so we omitted them from this comparison. We also evaluate our methods on denser graphs containing 100, 200, and 500 nodes (\Cref{table:Scored based methods F1 and SHD targeted  on dense graphs}), using targeted and hard interventions. For graphs of 100 nodes, the edge number increased by 100 edges per graph for the factor graph dataset, and 1,000 per graph for the spn dataset. In addition to F1 and SHD, we also report the structural intervention distance (SID) \citet{10.1162/NECO_a_00708} for score-based and Bayesian methods (\Cref{table:SID Bayesian based comparison} and \Cref{table:SID score-based comparison}).

\begin{table}[h!]
\caption{Performance on Simulated Datasets with 16 Nodes. Best performance is in bold text and second best is underlined.}
\label{table:sim_targeted}
\vskip 0.15in
\begin{center}
\begin{tiny}
\begin{sc}
\begin{tabular}{llcccccccccccc}
\toprule
Metric      & Method        & \multicolumn{3}{c}{Linear (FG)}       & \multicolumn{3}{c}{Linear (SPN-FG)} & \multicolumn{3}{c}{Nonlinear (FG)} & \multicolumn{3}{c}{Nonlinear (SPN-FG)}\\
    &   &   D1 & D2 & D3 & D1 & D2 & D3 & D1 & D2 & D3 & D1 & D2 & D3 \\
\midrule
SHD$\downarrow$& DCDI       & 12& \underline{4}& 26& 14& 12& 25& 14& \textbf{7}& \textbf{4}& \textbf{2}& 28& 14\\
            & DCDFG         & 48& 33& 31& 43& 43& 56& 48& 18& \underline{5}& 46& 36& 17\\
            & ENCO          & 27& 24& 28& 28& 54& 29& 27& 18& 29& 37& 41& 29\\
            & SDCD          & \underline{11}& 16& \underline{3}& 12& 15& 5& \underline{4}& \textbf{7}& 6& 8& \underline{16}& \textbf{5}\\
            & ABCDEFG       & \textbf{0}& \textbf{0}& \textbf{0}& 12& \textbf{0}& \textbf{0}& \textbf{2}& \underline{12}& 13& \underline{3}& \textbf{12}& \underline{9}\\
            & ABCDEFG       & \textbf{0}& 10& 12& \textbf{5}& 21& \underline{1}& 26& 28& 26& 22& 17& 25\\
            & (SPN)         & & & & & & & & & & & & \\
\midrule
F1$\uparrow$& DCDI          & \underline{0.842}& \underline{0.923}& 0.678& 0.793& \underline{0.876}& 0.679& 0.781& \textbf{0.759}& \textbf{0.935}& \textbf{0.964}& 0.682& 0.774\\
            & DCDFG         & 0.529& 0.190& 0.644& 0.566& 0.650& 0.509& 0.529& N/A& \underline{0.915}& 0.477& 0.667& 0.691\\
            & ENCO          & 0.000& 0.000& 0.000& 0.000& 0.000& 0.000& 0.000& 0.000& 0.000& 0.000& 0.128& 0.000\\
            & SDCD          & 0.825& 0.750& \underline{0.949}& \underline{0.818}& 0.842& 0.918& \underline{0.931}& \textbf{0.759}& 0.889& 0.833& 0.795& \textbf{0.915}\\
            & ABCDEFG       & \textbf{1.000}& \textbf{1.000}& \textbf{1.000}& 0.806& \textbf{1.000}& \textbf{1.000}& \textbf{0.964}& \underline{0.500}& 0.800& \underline{0.949}& \textbf{0.842}& \underline{0.857}\\
            & ABCDEFG       & \textbf{1.000}& 0.828& 0.824& \textbf{0.912}& 0.753& \underline{0.983}& 0.675& 0.333& 0.690& 0.718& \underline{0.805}& 0.683\\
            & (SPN)         & & & & & & & & & & & &\\
\bottomrule
\end{tabular}
\end{sc}
\end{tiny}
\end{center}
\vskip -0.1in
\end{table}

\begin{table}[t!]
\caption{F1 score and SHD of Scored methods on Nonlinear Targeted Simulated Datasets with 200 and 500 nodes.}
\label{table:Scored based methods F1 and SHD targeted with 200 and 500 nodes} 
\vskip 0.15in
\begin{center}
\begin{small}
\begin{sc}
\begin{tabular}{llcccc}
\toprule
Metric & Method & Hard & Soft & SPN & SPN\\
&& Intvn & Intvn & Hard & Soft\\
\midrule
F1 & DCDFG   &   $0.10\pm0.03$ &   $0.13\pm0.03$   &  $0.06\pm0.04$     &   $0.40\pm0.26$   \\
(200 nodes)&SDCD    &  \underline{$0.50\pm0.08$} &   \underline{$0.43\pm0.06$}   &  $0.16\pm0.01$  &   $0.14\pm0.01$    \\
&ABCDEFG &   \boldmath$0.52\pm0.13$ &   \boldmath$0.49\pm0.07$   &  $\underline{0.61\pm0.05}$     &   \underline{$0.62\pm0.04$} \\
&ABCDEFG &{$0.48\pm0.13$} &   $0.42\pm0.05$   &  \boldmath$0.62\pm0.04$     &   \boldmath$0.62\pm0.03$ \\
&(SPN)\\
\midrule
SHD &DCDFG   &   $7678\pm1846$ &   $4954\pm1147$   &  $13901\pm272$     &   $10892\pm1959$   \\
(200 nodes)&SDCD    &   \boldmath$517\pm38$ &   $592\pm16$   &  $13634\pm368$  &   $13854\pm516$    \\
&ABCDEFG &   \underline{$657\pm401$} &   \underline{$583\pm285$}   &  \underline{$8978\pm966$}     &   \boldmath$8739\pm743$ \\
&ABCDEFG  &$770\pm525$ &   \boldmath$583\pm169$   &  \boldmath$8950\pm697$     &   \underline{$8854\pm565$} \\
&(SPN)\\
\midrule
F1 & DCDFG   &   $0.11\pm0.10$ &   $0.06\pm0.00$   &  $0.07\pm0.04$     &   $0.17\pm0.10$   \\
(500 nodes)&SDCD    &   $0.34\pm0.01$ &   $0.32\pm0.00$   &  $0.10\pm0.01$  &   $0.09\pm0.01$    \\
&ABCDEFG &   \boldmath$0.56\pm0.00$ &   \boldmath$0.55\pm0.03$   &  \underline{$0.49\pm0.05$}     &   \underline{$0.52\pm0.06$} \\
&ABCDEFG & \underline{$0.48\pm0.01$} &   \underline{$0.50\pm0.04$}   &  \boldmath$0.54\pm0.04$     &   \boldmath$0.56\pm0.05$ \\
&(SPN)\\
\midrule
SHD &DCDFG   &   $22507\pm17922$ &   $25849\pm2023$   &  $105723\pm2134$     &   $99553\pm6921$   \\
(500 nodes)&SDCD    &   $1777\pm180$ &   $1849\pm134$   &  $105352\pm518$  &   $105834\pm508$    \\
&ABCDEFG &   \boldmath$1262\pm25$ &   \boldmath$1228\pm94$   &  \underline{$72836\pm6061$}     &   \underline{$69401\pm7637$} \\
&ABCDEFG  & \underline{$1562\pm50$} &   \underline{$1340\pm132$}   &  \boldmath$67007\pm4940$     &   \boldmath$64650\pm5395$ \\
&(SPN)\\
\bottomrule
\end{tabular}
\end{sc}
\end{small}
\end{center}
\vskip -0.1in
\end{table}

\begin{table}[t!]
\caption{F1 score and SHD of Scored methods on Nonlinear Targeted Simulated Datasets on dense graphs with hard interventions.}
\label{table:Scored based methods F1 and SHD targeted  on dense graphs} 
\vskip 0.15in
\begin{center}
\begin{small}
\begin{sc}
\begin{tabular}{llcc}
\toprule
Metric & Method & non linear & non linear SPN \\
\midrule
F1 & DCDI   &   $0.47\pm0.04$ &   $0.34\pm0.05$\\
(100 nodes)&DCDFG   &   $0.25\pm0.03$ &   $0.11\pm0.06$\\
&ENCO   &   $0.04\pm0.00$ &   $0.12\pm0.06$\\
&SDCD    &  \boldmath$0.66\pm0.06$ &   $0.35\pm0.06$\\
&ABCDEFG &   \underline{$0.53\pm0.07$} &   \boldmath$0.69\pm0.04$   \\
&ABCDEFG &{$0.43\pm0.05$} &   \underline{$0.65\pm0.01$}     \\
&(SPN)\\
\midrule
SHD   & DCDI   &   $475\pm81$ &   $3882\pm135$\\
(100 nodes)&DCDFG   &   $1464\pm910$ &   $3866\pm98$\\
&ENCO   &   $2185\pm223$ &   $3982\pm228$\\
&SDCD    &   \boldmath$245\pm30$ &   $3283\pm148$   \\
&ABCDEFG &   \underline{$567\pm92$} &   \boldmath$1909\pm182$  \\
&ABCDEFG  &$770\pm525$ &   \underline{$2132\pm59$}    \\
&(SPN)\\
\midrule
F1 & DCDFG   &   $0.18\pm0.06$ &   $0.11\pm0.02$  \\
(200 nodes)&SDCD    &   $0.42\pm0.05$ &  $0.18\pm0.02$\\
&ABCDEFG &   \boldmath$0.56\pm0.08$ &   \underline{$0.59\pm0.03$}  \\
&ABCDEFG & \underline{$0.49\pm0.06$} &   \boldmath$0.60\pm0.01$  \\
&(SPN)\\
\midrule
SHD &DCDFG   &   $6933\pm2978$ &   $17035\pm257$ \\
(200 nodes)&SDCD    &   \boldmath$885\pm168$ &   $16575\pm142$\\
&ABCDEFG &   \underline{$932\pm408$} &  \underline{$9817\pm600$} \\
&ABCDEFG  & $1130\pm445$ &   \boldmath$9649\pm104$  \\
&(SPN)\\
\midrule
F1 & DCDFG   &   $0.09\pm0.07$ &   $0.03\pm0.03$  \\
(500 nodes)&SDCD    &   \underline{$0.26\pm0.01$} &  $0.10\pm0.00$\\
&ABCDEFG &   \boldmath$0.33\pm0.05$ &   \underline{$0.45\pm0.04$}  \\
&ABCDEFG & $0.25\pm0.06$ &   \boldmath$0.47\pm0.03$  \\
&(SPN)\\
\midrule
SHD &DCDFG   &   \underline{$4298\pm131$} &   $118625\pm1433$ \\
(500 nodes)&SDCD    &   \boldmath$4294\pm191$ &   $114183\pm421$\\
&ABCDEFG &   $6598\pm1804$ &  \underline{$80315\pm5300$} \\
&ABCDEFG  & $8662\pm3101$ &   \boldmath$77089\pm3660$  \\
&(SPN)\\
\bottomrule
\end{tabular}
\end{sc}
\end{small}
\end{center}
\vskip -0.1in
\end{table}

\begin{table}[t!]
\caption{SID of Bayesian methods on Non linear Simulated Datasets with 16 Nodes.}
\label{table:SID Bayesian based comparison}
\vskip 0.15in
\begin{center}
\begin{small}
\begin{sc}
\begin{tabular}{llcccc}
\toprule
Metric&Method & NON LINEAR& NON LINEAR SPN\\
\midrule
SID & DECI   &   $65.76\pm34.86$ &   $109.32\pm18.75$\\
&VI-DP-DAG   &   $83.52\pm35.83$ &   $92.19\pm7.59$ \\
&ProDAG    &   $62.01\pm25.20$ &   $90.1\pm21.78$  \\
&ABCDEFG   &   \boldmath$41.93\pm27.27$ &   \boldmath$64.33\pm15.75$  \\
&ABCDEFG & $50.72\pm27.49$ & $70.85\pm24.81$ \\
&(SPN)\\
\bottomrule
\end{tabular}
\end{sc}
\end{small}
\end{center}
\vskip -0.1in
\end{table}

\begin{table}[t!]
\caption{SID of score-based methods on Non linear Simulated Datasets with 100 Nodes.}
\label{table:SID score-based comparison}
\vskip 0.15in
\begin{center}
\begin{small}
\begin{sc}
\begin{tabular}{llcccc}
\toprule
Metric&Method & HARD & SOFT & SPN & SPN \\
&& INTVN & INTVN & HARD & SOFT\\
\midrule
SID & DCDFG   &   $1839\pm308$ &   $1595\pm1418$   &  $5860\pm1662$     &   $6976\pm346$  \\
&ENCO   &   $3668\pm864$ &   $3722\pm945$   &  $8805\pm221$     &   $8899\pm200$   \\
&SDCD    &   $2189\pm616$ &   $2224\pm1009$   &  $6843\pm504$     &   $6858\pm381$     \\
&ABCDEFG   &   $1005\pm327$ &   $771\pm438$   &  \boldmath$4615\pm681$  &   \boldmath$4710\pm694$    \\
&ABCDEFG & \boldmath$809\pm510$ &   \boldmath$671\pm414$   &  $4724\pm595$     &   $4783\pm594$ \\
&(SPN)\\
\bottomrule
\end{tabular}
\end{sc}
\end{small}
\end{center}
\vskip -0.1in
\end{table}
\subsection{Availability of Benchmark Results}
\begin{figure}[h!]
    \centering
    \includegraphics[width=1\linewidth]{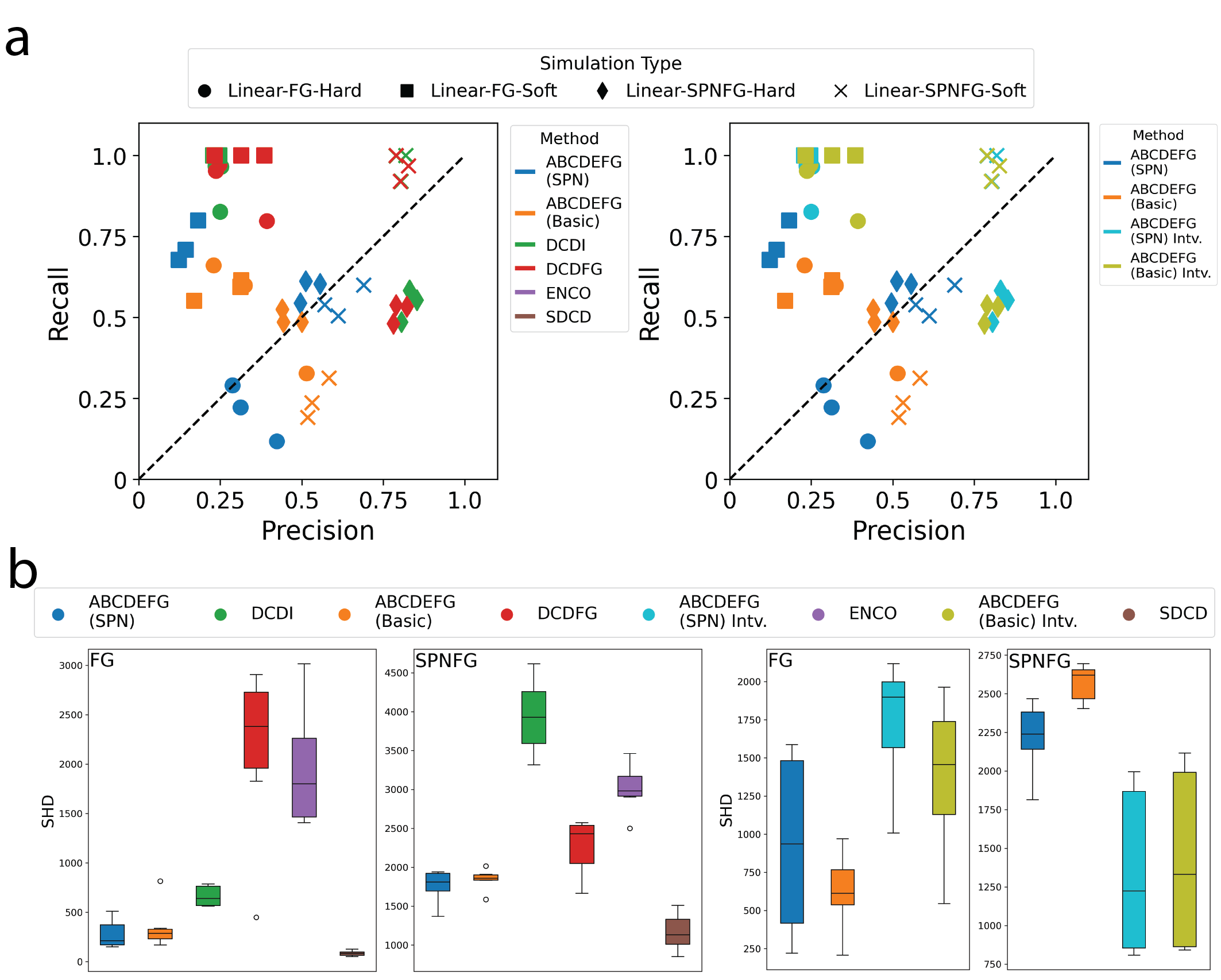}
    \caption{\textbf{Benchmarking of score-based methods on linear datasets.} \textbf{(a)} Precision and recall for different score based methods, dataset types are shown in different shapes.\textbf{(b)} SHD comparison between different score based methods on targeted datasets(left two), and untargeted datasets (right two).}
    \label{supp:fig:linear_bechmark}
\end{figure}

\begin{figure}[h!]
    \centering
    \includegraphics[width=1\linewidth]{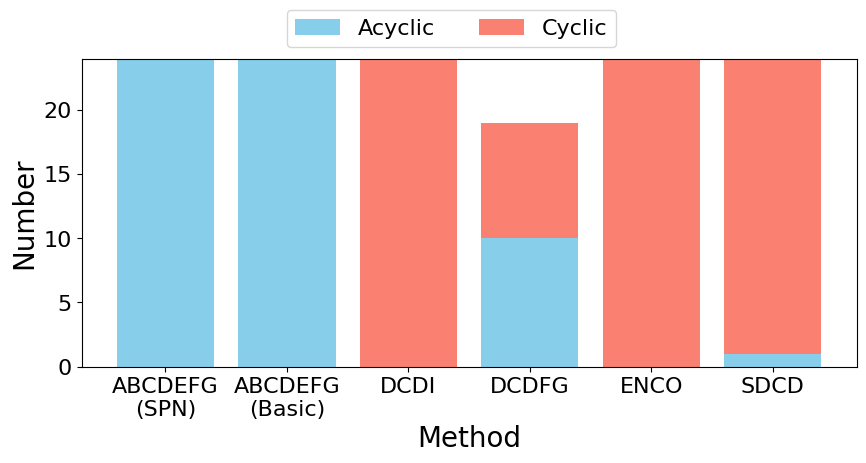}
    \caption{\textbf{Comparison of number of acyclic and cyclic graphs.}}
    \label{supp:fig:num_acyc_bar}
\end{figure}

\begin{figure}[h!]
    \centering
    \includegraphics[width=1\linewidth]{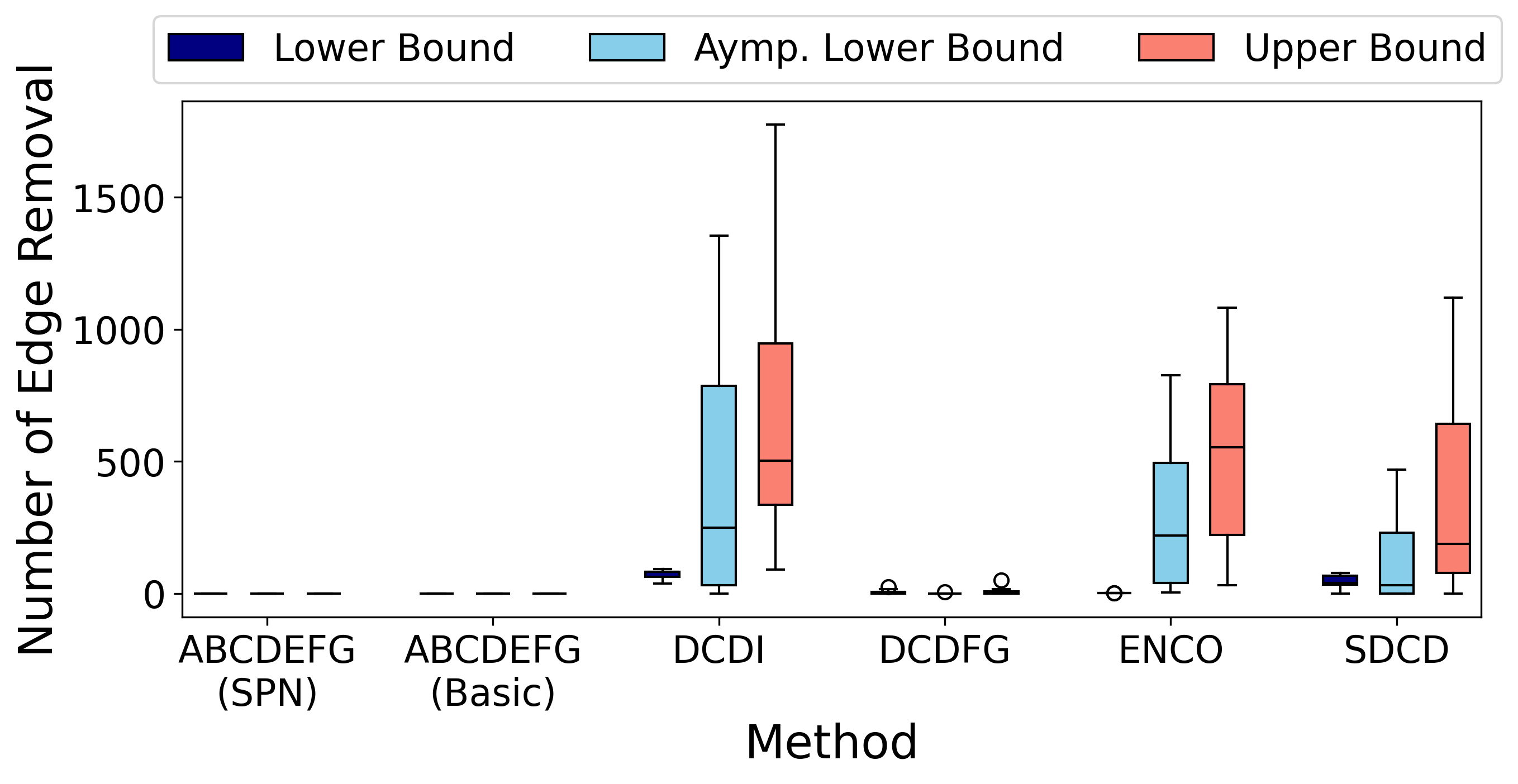}
    \caption{\textbf{Comparison of to be removed number of edges for acyclic graphs.} Upper and lower bound of number of to be removed edges are colored in blue and red, respectively.}
    \label{supp:fig:dag_edge}
\end{figure}

We conducted benchmark studies on a variety of data simulation settings at a larger scale, with 100 nodes and 10 factors. We classify the simulations by (1) SEM - linear vs. nonlinear, (2) factor graph model - SPN-FG vs. regular f-DAG and (3) type of intervention (hard vs. soft). We included all results as csv files in our supplementary material. Each csv file records a metric (precision, recall, f1, SHD) for all methods run on one type of simulation. The tables summarized in Table \ref{table:Scored based methods F1 and SHD targeted} and Table \ref{table:F1 and SHD for ABCDEFG on untargeted} show the mean $\pm$ standard deviation for each dataset type, based on the corresponding experimental results. Moreover, the benchmarking results for score-based methods on linear datasets are presented in Fig. \ref{supp:fig:linear_bechmark}, as discussed in the main text. In addition, as proof that our model can construct acyclic graphs by design, we calculated the number of cycles when compared with score-based methods (Fig. \ref{supp:fig:num_acyc_bar}), as well as the number of edges that would need to be removed to obtain an acyclic graph (Fig. \ref{supp:fig:dag_edge}). Both results suggest that the graphs predicted by our model are naturally acyclic.

\subsection{Experiment Settings}\label{supp:hparam_setting}
In this section, we report the hyperparameters used in our simulation study. Because ABCDEFG has many hyperparameters, we did not comprehensively tune each of them. Instead, we fixed hyperparameters across the same SEM model type. Here, we report some key hyperparameter values. For the other hyperparameters, our python program contains default values and we used the same value in all experiments. Table \ref{supp:table:default_hparam} summarizes the most important hyperparameters. In addition, we unexhaustively tuned the L1 regularization coefficient by trying two different values per simulation type. We also have a separate L1 regularization coefficient for the intervention-to-node bipartite graph in simulation with unknown intervention targets. 


Table \ref{supp:table:hparam} lists the set of best parameters we chose for each simulation type.  For conciseness, we name a simulation type by a sequence of four attributes: (1) targeted (T) vs. untargeted (U), (2) canonical f-DAG (FG) vs. SPN-FG (SPNFG), (3) linear (L) vs nonlinear (N) SEM, and (4) hard (H) vs. soft (S) intervention, separated by ``-". 

\begin{table}[h!]
\caption{Default Hyper-Parameter Setting of ABCDEFG in a Simulation Study.}
\label{supp:table:default_hparam}
\vskip 0.15in
\begin{center}
\begin{small}
\begin{sc}
\begin{tabular}{lccc}
\toprule
Parameter Name & Default Value \\
\midrule
Batch Size & 128\\
Hidden Dimension & 1000\\
Number of Epochs & 1000 \\
Number of Hidden Layers & 1\\
Width Bound of SPN (max\_copies) & 8\\
Learning Rate (VAE) & $5\times 10^{-4}$\\
Learning Rate (f-DAG Model) & $5\times 10^{-3}$\\ 
KL Div. Coeff. ($\beta$) & $1\times 10^{-8}$\\
Gaussian Noise Level & 0.05\\
VAE Weight L2 Reg. & $1\times 10^{-3}$\\
Latent Factor Prior & $\mathcal{N}(\boldsymbol{0}, 10^{-3}\cdot \mathbf{I})$\\
\bottomrule
\end{tabular}
\end{sc}
\end{small}
\end{center}
\vskip -0.1in
\end{table}

\begin{table}[t!]
\caption{Hyper-Parameter Setting of ABCDEFG in a Simulation Study.}
\label{supp:table:hparam}
\vskip 0.15in
\begin{center}
\begin{small}
\begin{sc}
\begin{tabular}{lccccc}
\toprule
Simulation Type & L1 Reg. & L1 Reg. (Intv.) & Activation Function &  SPN Parallelism\\
\midrule
T-FG-L-H        & 0.1, 0.1   & N/A       & Identity& Node\\
T-FG-L-S        & 0.01, 0.01 & N/A       & Identity& Factor\\
T-FG-N-H        & 1.0, 1.0   & N/A       & Tanh    & Factor\\
T-FG-N-S        & 0.01, 0.001& N/A       & Tanh    & Node\\
T-SPNFG-L-H     & 0.01, 0.01 & N/A       & Identity& Node\\
T-SPNFG-L-S     & 1e-4, 1e-4 & N/A       & Identity& Node\\
T-SPNFG-N-H     & 0.01, 0.01 & N/A       & Tanh    & Node\\
T-SPNFG-N-S     & 1e-4, 1e-4 & N/A       & Tanh    & Factor\\
U-FG-L-H        & 0.01, 0.01 & 10.0, 10.0& Identity& Node\\
U-FG-L-S        & 1e-4, 1e-4 & 10.0, 10.0& Identity& Node\\
U-FG-N-H        & 0.01, 0.01 & 10.0, 10.0& Tanh    & Node\\
U-FG-N-S        & 1e-4, 1e-4 & 10.0, 10.0& Tanh    & Node\\
U-SPNFG-L-H     & 1e-6, 1e-6 & 0.1, 0.1  & Identity& Factor\\
U-SPNFG-L-S     & 1e-7, 1e-7 & 1.0, 1.0  & Identity& Node\\
U-SPNFG-N-H     & 1e-6, 1e-6 & 0.1, 0.1  & Tanh    & Factor\\
U-SPNFG-N-S     & 1e-8, 1e-7 & 1.0, 1.0  & Tanh    & Node\\
\bottomrule
\end{tabular}
\end{sc}
\end{small}
\end{center}
\vskip -0.1in
\end{table}

\subsection{Time and Memory Consumption}

\begin{table}[t!]
\caption{Time usage on Simulated Datasets with 16 Nodes.}
\label{supp:table:Time_Bayesian}
\vskip 0.15in
\begin{center}
\begin{small}
\begin{sc}
\begin{tabular}{lcccc}
\toprule
Method & LINEAR & LINEAR  & NONLINEAR& NONLINEAR   \\
&FG&SPNFG&FG&SPNFG\\

\midrule
BaCaDi    &   $1704.56\pm33.70$ &   $1405.64\pm277.53$   &  $1265.71\pm35.65$     &   $1435.08\pm20.63$  \\
DECI   &   $987.81\pm5.72$ &   $985.27\pm1.64$   &  $994.50\pm0.91$     &   $991.40\pm0.30$   \\
VI-DP-DAG    &   $764.77\pm255.60$ &   $245.63\pm132.17$   &  $501.51\pm328.49$     &   $302.64\pm180.52$     \\
ProDAG    &   $79.37\pm0.76$ &   $79.99\pm2.34$   &  N/A     &   N/A   \\
ABCDEFG &   $82.60\pm24.36$ &   $65.95\pm24.61$   &  $70.51\pm34.89$     &   $106.71\pm68.05$ \\
ABCDEFG (SPN) &$138.64\pm37.63$ &   $67.63\pm32.89$   &  $136.17\pm52.82$     &   $177.71\pm154.83$ \\
\bottomrule
\end{tabular}
\end{sc}
\end{small}
\end{center}
\vskip -0.1in
\end{table}

\begin{figure}[h!]
    \centering
    \includegraphics[width=1\linewidth]{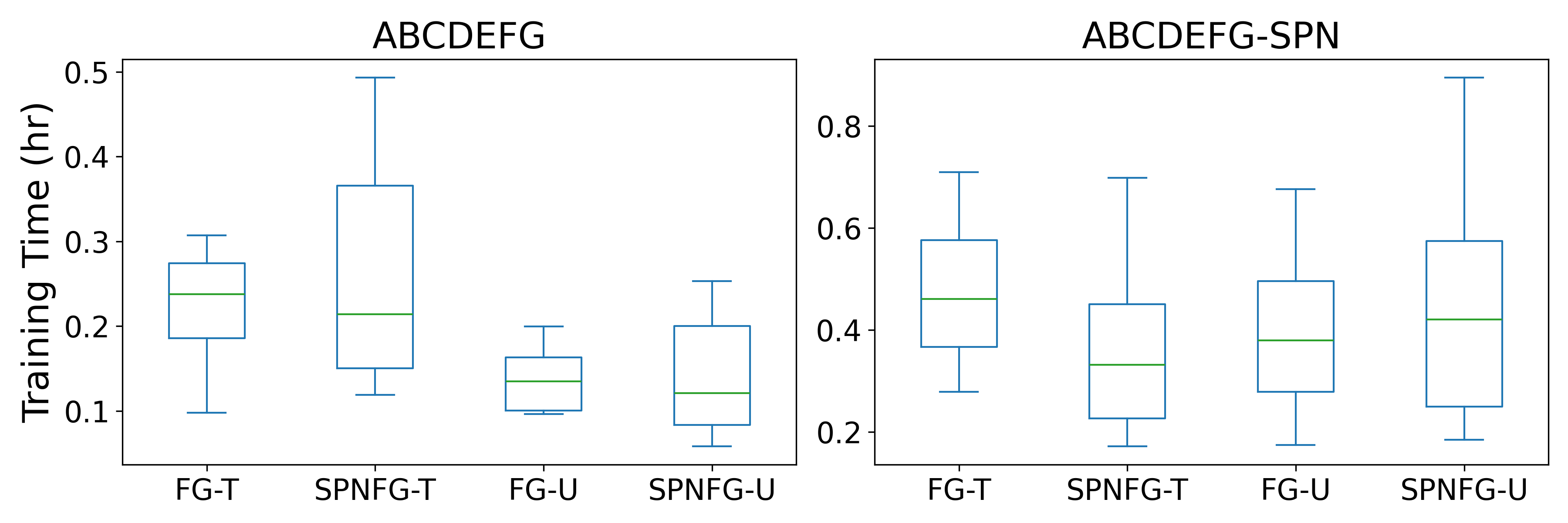}
    \caption{\textbf{Training Time of ABCDEFG.} Each box represents one type of simulation. We group simulation regarding the ground truth graph type and known vs. unknown intervention targets. We use the suffix ``-T" for known intervention targets and ``-U" for unknown ones.}
    \label{supp:fig:run_time}
\end{figure}

All simulated datasets with known intervention targets contain 25k samples and those with unknown intervention targets contain 30k samples. With a batch size of 128, we were able to train our model on a server with 2 2x 2.9 GHz Intel Xeon Gold 6226R, 16 GB of RAM and an NVIDIA A40 GPU with 48GB of memory. The training time of ABCDEFG is shown in Fig. \ref{supp:fig:run_time}. Since the datasets are of similar sizes, the training time is stable across different simulations. Training ABCDEFG with SPN-FG consumes more time due to a larger number of parameters and extra time for forward and backward through the network layers. The benchmarking of Bayesian methods was conducted on datasets with 16 nodes. The training times for the different methods are shown in Table~\ref{supp:table:Time_Bayesian}. All methods, except BaCaDi, were run on an NVIDIA A40 GPU with 16GB of RAM. (No GPU implementation was available for BaCaDi.)

\section{Preprocessing single cell perturbation data}

The data used for single cell perturbation is downloaded from \citet{AMIN20241831} and we followed the preprocessing steps described by \citet{Lopez2022largescale}. For each untargeted perturbation, we removed the description words like 'high','low','early',eta, and only retain the name of each biomolecule as the perturbation. We used scanpy to select the top 1000 highly variable genes as input of our model, and used 10 factors. We performed gene ontology analysis using the online tool at \href{https://geneontology.org/}{the Gene Ontology Website}.

\bibliography{sample}

\end{document}